\documentclass[twoside]{article}

\usepackage{natbib}
\usepackage[accepted]{aistats2019}

\usepackage[utf8]{inputenc} 
\usepackage[T1]{fontenc}    
\usepackage{hyperref}       
\usepackage{url}            
\usepackage{booktabs}       
\usepackage{amsfonts}       
\usepackage{nicefrac}       
\usepackage{microtype}      
\usepackage{amsmath}
\usepackage{amsthm}
\usepackage[noend,linesnumbered]{algorithm2e}
\usepackage{algpseudocode}
\usepackage{amssymb}
\usepackage{bbm}
\usepackage{enumitem}
\usepackage{graphicx}
\usepackage{caption}
\usepackage{subcaption}
\usepackage{multirow}
\usepackage{floatflt}
\usepackage{color}

\theoremstyle{definition}
\newtheorem{definition}{Definition}

\newtheorem{theorem}{Theorem}
\newtheorem{lemma}{Lemma}
\newtheorem{claim}{Claim}

\newenvironment{proofof}[1]{\smallskip\noindent{\bf Proof of #1}}%
        {\hspace*{\fill}$\Box$\par}

\DeclareMathOperator*{\argmax}{arg\,max}

\DeclareMathOperator*{\maxp}{top^p}

\DeclareMathOperator*{\maxz}{top^0}
\DeclareMathOperator*{\maxten}{top^{10}}

\newcommand{\specialcell}[2][c]{%
  \begin{tabular}[#1]{@{}c@{}}#2\end{tabular}}

\let\emptyset\varnothing

\title{Distributed Maximization of Submodular plus Diversity Functions\\for Multi-label Feature Selection on Huge Datasets}

\begin{document}

\twocolumn[
\runningtitle{Multi-Label Feature Selection Using Submodular Plus Diversity Maximization}
\aistatstitle{Distributed Maximization of Submodular plus Diversity Functions\\for Multi-label Feature Selection on Huge Datasets}

\aistatsauthor{ Mehrdad Ghadiri \And Mark Schmidt }

\aistatsaddress{ University of British Columbia } ]

\begin{abstract}
There are many problems in machine learning and data mining which are equivalent to selecting a non-redundant, high ``quality'' set of objects. Recommender systems, feature selection, and data summarization are among many applications of this. In this paper, we consider this problem as an optimization problem that seeks to maximize the sum of a sum-sum diversity function and a non-negative monotone submodular function. The diversity function addresses the redundancy, and the submodular function controls the predictive quality. We consider the problem in big data settings (in other words, distributed and streaming settings) where the data cannot be stored on a single machine or the process time is too high for a single machine. We show that a greedy algorithm achieves a constant factor approximation of the optimal solution in these settings. Moreover, we formulate the multi-label feature selection problem as such an optimization problem. This formulation combined with our algorithm leads to the first distributed multi-label feature selection method. We compare the performance of this method with centralized multi-label feature selection methods in the literature, and we show that its performance is comparable or in some cases is even better than current centralized multi-label feature selection methods.
\end{abstract}

\section{Introduction}
Many problems from different areas of machine learning and data mining can be modeled as an optimization problem that tries to maximize the sum of a sum-sum diversity function (which is the sum of the distances between all of the pairs in a given subset) and a non-negative monotone submodular function. Examples include query diversification problem in the area of databases~\citep{DemidovaFZN10,LiuSC09}, search result diversification~\citep{AgrawalGHI09,DrosouP10}, and recommender systems~\citep{YuLA09}. The size of the datasets in these applications is growing rapidly, and there is a need for scalable methods to tackle these problems on huge datasets. Inspired by these applications, we propose an algorithm for approximately solving this optimization problem with a theoretical guarantee in distributed and steaming settings.~\citet{BorodinJLY17} presented a 0.5-approximation for this optimization problem in the centralized setting in which data can be stored and processed on a single machine. In this paper, we consider this problem for big data settings where the data cannot be stored on a single machine, or the process time is too high for a single machine. We show that our algorithm achieves a $1/31$-approximation. Note that solving this problem in a distributed or streaming setting is strictly harder than solving it in the centralized setting because, in the aforementioned settings, the algorithm does not use all of the data. As a result, our algorithm is $\frac{\sqrt{d/k}}{2}$ times faster in the distributed setting and it needs $\sqrt{d/k}$ times less memory in the streaming setting compared to the centralized setting, where $d$ is the size of the ground set (for example, the number of features in the feature selection problem), and $k$ is the number of machines (in the distributed setting) or is the number of partitions of the data (in the streaming setting). Therefore, our algorithm gives a worse approximate solution compared to the centralized method of~\citet{BorodinJLY17} but it is much faster and needs less memory. This trade-off might be interesting and useful in some applications.

One of the problems that can be modeled as such an optimization problem and is in need of scalable methods in modern applications is multi-label feature selection. The diversity part controls the redundancy of the selected features and the submodular part is to
promote features that are relevant to the labels. A multi-label dataset is made up of a number of samples, features, and labels. Each sample is a set of values for the features and labels. Usually, labels have binary values. For example, if a patient has diabetes or not. Multi-label datasets can be found in different areas, including but not limited to semantic image annotation, protein and gene function studies, and text categorization~\citep{kashefmultilabel}. Applications, number, and size of such datasets are growing very rapidly, and it is necessary to develop efficient and scalable methods to deal with them.

Feature selection is a fundamental problem in machine learning. Its goal is to decrease the dimensionality of a dataset in order to improve the learning accuracy, decrease the learning and prediction time, and prevent overfitting. There are three different categories of feature selection methods depending on their interaction with the learning methods. Filter methods select the features based on the intrinsic properties of the data and are totally independent of the learning method. Wrapper methods select the features according to the accuracy of a specific learning method, like SVMs. Finally, embedded methods select the features as a part of their learning procedure~\citep{GuyonE03}. Decision trees and use of $\ell_0$ and $\ell_1$ regularization for feature selection fall into the latter. When the number of features is large, filter methods are a reasonable choice since they are fast, resistant to over-fitting, and independent of the learning model. Therefore, we can quickly select a number of features with filter methods and then try different learning methods to see which one fits the data better (possibly with wrapper or embedded feature selection methods). However, with millions of features, centralized filter methods are not applicable anymore. To deal with such huge datasets, we need scalable methods. Although there were efforts to develop scalable and distributed filter methods for single-label datasets~\citep{ZadehGMZ17,Bolon-CanedoSA15}, to the best of our knowledge, there are no previous distributed multi-label feature selection method.

In this paper, we propose an information theoretic filter feature selection method for multi-label datasets that is usable in distributed, streaming, and centralized settings. In the centralized setting, all of the data is stored and can be processed on a single machine. In the distributed setting, the data is stored on multiple machines, and there is no shared memory between machines. In the streaming setting, although the computation is done on a single machine, this machine does not have enough memory to store all of the data at once. The data in our method is distributed vertically which means that the features are distributed between machines instead of samples (horizontal distribution). Feature selection is considered harder when the data is distributed vertically because we lose much information about the relations of the features~\citep{Bolon-CanedoSA15B}. However, when the number of instances is small, and the number of features is large (for example, biological or medical datasets) vertical distribution is the only reasonable choice. Our work can be seen as an extension of~\citet{BorodinJLY17} to distributed and streaming settings or an extension of~\citet{ZadehGMZ17} to multi-label data. However, our results cannot be derived from these previous works in a straightforward manner. The main contributions of the paper are listed in the following.

\subsection*{Our Contributions}
\begin{itemize}[leftmargin = 0.35cm]
\setlength\itemsep{0.1cm}
\item We present a greedy algorithm for maximizing the sum of a sum-sum diversity function and a non-negative monotone submodular function in the distributed and streaming settings. We prove that it achieves a constant factor approximation of the optimal solution.
\item We formulate the multi-label feature selection problem as such a combinatorial optimization problem. Using this formulation we present information theoretic filter feature selection methods for distributed, steaming, and centralized settings. The distributed method is the first distributed multi-label feature selection method proposed in the literature.
\item We perform an empirical study of the proposed distributed method and compare its results to different centralized multi-label feature selection methods. We show that the results of the distributed method are comparable to the current centralized methods in the literature. We also compare the runtime and the value of the objective function that our centralized and distributed methods achieve. Note that the centralized methods have access to the all of the data and can do computation on it. We do not expect that our distributed or streaming method to beat the centralized methods because it is not possible. However, we argue that our results are comparable to the results of centralized methods and our method is much faster (in case of the distributed setting) and needs much less memory (in case of the streaming setting). We compared our results with the centralized methods (this comparison is unfair to the distributed setting) in the literature because to the best of our knowledge there is no distributed multi-label feature selection method prior to this work.
\end{itemize}
\vspace{-0.2cm}

Our techniques can be used prior to multi-label classification, multi-label regression, and in some multi-task learning setups. The structure of the paper is as follows. In the next section, we review the related work and preliminaries. In Section 3, we formulate the multi-label feature selection problem as the mentioned optimization problem and present the algorithm for maximizing it in the distributed and streaming settings. In Section 4, we show the theoretical approximation guarantee of the proposed algorithm. In Section 5, we evaluate the performance of the proposed distributed algorithm in practice.

\section{Related Work}
In this section, we review the previous works on different aspects of the problem including diversity maximization, submodular maximization, composable core-sets, and feature selection.
\subsection*{Diversity Maximization and Submodular Maximization}
Usually, the diversity maximization problem is defined on a metric space of a set of points $U$ with the goal of finding a subset of them which maximizes a diversity function subject to a constraint. For example, a cardinality constraint or a matroid constraint. If $S$ is a subset of the points, the sum-sum diversity of $S$ is $D(S) = 0.5\sum_{x\in S}\sum_{y\in S} d(x,y)$ where $d(.,.)$ is a metric distance. In the centralized setting, a simple greedy or local search algorithm can achieve a half approximation of the optimal solution subject to $|S|=k$~\citep{maxdispersion,AbbassiMT13}. TA better approximation factor is not achievable under the planted clique conjecture~\citep{BhaskaraGMS16,BorodinJLY17}.

Submodular functions are important concepts in machine learning and data mining with many applications. See~\citet{krause2008beyond} for their applications. A submodular function is a set function with a diminishing marginal gain. A function $f: 2^U \rightarrow \mathbb{R}$ is submodular if $f(A\cup \{x\}) - f(A) \geq f(B\cup \{x\}) - f(B)$ for any $A \subseteq B \subset U$, and $x\in U\setminus B$. It is monotone if $f(A) \leq f(B)$ and it is non-negative if $f(A) \geq 0$ for any $A \subseteq B \subseteq U$. Maximizing a monotone submodular function subject to a cardinality constraint is NP-hard but using a simple greedy algorithm we can achieve $(1-\frac{1}{e})$ of the optimal solution. A better approximation factor is not achievable using a polynomial time algorithm unless P=NP~\citep{0001G14}.

Let $U$ be a set and $f(.)$ be a submodular function defined on $U$ and $d(.,.)$ be a metric distance defined between pairs of elements of $U$.~\citet{BorodinJLY17} showed that in the centralized setting, using a simple greedy algorithm, we can achieve half of the optimal value for maximizing $f(S)+\lambda \sum_{\{u,v\}: u,v \in S} d(u,v)$ subject to $S \subseteq U$ and $|S| = k$. This result is extended to semi-metric distances in~\citet{zadeh2015max}. Similar problems are considered in~\citet{DasguptaKR13} where the diversity part can be other diversity functions. Namely, they considered the sum-sum diversity, the minimum spanning tree, and the minimum of distances between all pairs. They showed that the greedy algorithm achieves a constant factor approximation in all of these cases.

\subsection*{Composable Core-sets}
In computational geometry, a core-set is a small subset of points that approximately preserve a measure of the original set~\citep{agarwal2005geometric}. Composable core-sets extend this property to the combination of sets. Therefore, they can be used in a divide and conquer manner to find an approximate solution. Let $U$ be a set, $f:2^U\rightarrow \mathbb{R}$ be a set function on $U$, $(T^1, \ldots, T^m)$ be a random partitioning of elements of $U$, and $k$ be a positive integer. Let $\texttt{OPT}(T) = \argmax_{S\subseteq T, |S|=k} f(S)$ where $T\subseteq U$. Let $\texttt{ALG}$ be an algorithm which takes $T \subseteq U$ as an input and outputs $S \subseteq T$. For $\alpha>0$, we call $\texttt{ALG}$ an $\alpha$-approximate composable core-set with size $k$ for $f$ if the size of its output is $k$ and $f(\texttt{OPT}(\texttt{ALG}(T^1)\cup \cdots \cup \texttt{ALG}(T^m))) \geq \alpha f(\texttt{OPT}(T^1 \cup \cdots \cup T^m))$~\citep{IndykMMM14}. We call $\texttt{ALG}$ an $\alpha$-approximate \emph{randomized} composable core-set with size $k$ for $f$ if the size of its output is $k$ and $\mathbb{E}[f(\texttt{OPT}(\texttt{ALG}(T^1)\cup \cdots \cup \texttt{ALG}(T^m)))] \geq \alpha f(\texttt{OPT}(T^1 \cup \cdots \cup T^m))$~\citep{MirrokniZ15}. Composable core-sets and randomized composable core-sets can be used in distributed settings (like the MapReduce framework) and streaming settings (see Figure~\ref{Fig:systems}).

Composable core-sets first were used to approximately solve several diversity maximization problems in distributed and streaming settings~\citep{IndykMMM14}. It resulted in an approximation algorithm for the sum-sum diversity maximization with an approximation factor of less than $0.01$. This approximation factor is improved to $\frac{1}{12}$ in~\citet{AghamolaeiFZ15}. Randomized composable core-sets were first introduced to tackle submodular maximization problem in distributed and streaming settings which resulted in a $0.27$-approximation algorithm for monotone submodular functions~\citep{MirrokniZ15}. Then they were used to improve the approximation factor of the sum-sum diversity maximization from $\frac{1}{12}$ to $0.25$~\citep{ZadehGMZ17}. The randomized composable core-sets used in the latter case find the approximate solution with high probability instead of expectation.

There are a number of other works on distributed submodular maximization~\citep{MirzasoleimanKS16,BarbosaENW15}. Moreover, submodular and weak submodular functions are used for distributed \emph{single-label} feature selection~\citep{KhannaEDNG17}. We should note that the discussed objective function in our work is neither submodular nor weak submodular. This is because of the diversity term of the function. An advantage of using this diversity function is that it is evaluated by a pairwise distance function. As a result, it is easy to evaluate our objective function on datasets with few samples. On the contrary, evaluating the pure submodular functions, that were used for feature selection in the literature, are quite hard and need a large amount of data and computing power. 

\subsection*{Feature Selection and Multi-label Feature Selection}
Filter feature selection methods select features independent of the learning algorithm. Hence, they are usually faster and immune to overfitting~\citep{GuyonE03}. Mutual information based methods are a well-known family of filter methods. The best-known method of this kind for single-label feature selection is minimum redundancy and maximum relevance (mRMR) which tries to find a subset of features $S$ that maximizes the following objective function using a greedy algorithm

\vspace{-0.4cm}
\[
\frac{1}{|S|}\sum_{x_i \in S} I(x_i,c) - \frac{1}{|S|^2}\sum_{x_i,x_j\in S} I(x_i,x_j),
\]
\vspace{-0.4cm}

where $I(.,.)$ is the mutual information function, and $c$ is the label vector~\citep{PengLD05}. The proposed method in this paper can be seen as a variation of mRMR which is capable of being used for multi-label feature selection in distributed, streaming, and centralized settings.

Although there have been great advancements in centralized feature selection, there are few works on distributed feature selection, and most of them distribute the data horizontally. \citet{ZadehGMZ17} was the first work on the single-label vertically distributed feature selection that considered the redundancy of the features. Their method selects features using randomized composable core-sets in order to maximize a diversity function defined on the features. Although there are some similarities between the formulations presented in~\citet{ZadehGMZ17} and this work, we should note that the single-label formulation cannot be applied directly to multi-label datasets. Moreover, maximization of the functions and the analysis of the algorithms to prove the theoretical guarantee are completely different.

Most of the multi-label feature selection methods transform the data to a single-label form. Binary relevance (BR) and label powerset (LP) are two common ways to do so. BR methods consider each label separately and use a  single-label feature selection method to select features for each label, and then they aggregate the selected features. A disadvantage of BR methods is that they cannot consider the relations of the labels. LP methods consider the multi-label dataset as one single-label multi-class dataset where each class of its single label are a possible combination of labels in the dataset (treating the labels as a binary string). Then they apply a single-label feature selection method. Although LP methods consider the relations of the labels, they have significant drawbacks. For example, some classes may end up with very few samples or none at all. Moreover, the method is biased toward the combination of the labels which exist in the training set~\citep{kashefmultilabel}. Our proposed method does not transform the data to single-label data and is designed in a way to not suffer from the mentioned disadvantages.


\section{Problem Formulation}
Let $U$ be a set of $d$ features and $L$ be a set of $t$ labels. 
We also have a set $A$ of $n$ instances each of which is a vector of  observations for elements of  $U \cup L$. The goal of {\em  multi-label feature selection} is to find a small {\em non-redundant} subset of $U$ which can {\em predict} labels in $L$ accurately. In order to quantify redundancy  it is natural to use a metric distance $d$ over the feature set to measure dissimilarity. In our application (feature selection) we are particularly interested in the following metric distance. For any $u_i, u_j \in U$, we define
\begin{align*}
d(u_i,u_j) & = 1-\frac{I(u_i,u_j)}{H(u_i,u_j)} \\ & = 1 - \frac{\sum_{x \in u_i, y \in u_j} p(x,y) \log{\frac{p(x,y)}{p(x)p(y)}}}{-\sum_{x \in u_i, y \in u_j}p(x,y)\log p(x,y)},
\end{align*}
where $H(.,.)$ is the joint entropy and $I(.,.)$ is the mutual information. This distance function is called {\em normalized} (values lie between $0$ and $1$) {\em variation of information} and it is a metric~\citep{NguyenEB10}. In~\citet{ZadehGMZ17}, this distance function plus a modular function is used for single-label feature selection. 

In order to quantify the predictive quality of the selected features, we define a non-negative monotone submodular function $g:2^U\rightarrow \mathbb{R}$ which measures the relevance of the selected features to the labels. For any positive integer $p$, we define
\[
g(S) = \sum_{\ell \in L} \maxp_{x\in S} \{MI(x, \ell)\},
\]
where $\maxp_{x\in S} \{MI(x, \ell)\}$ is the sum of the $p$ largest numbers in $\{MI(x, \ell)|x\in S\}$. Here $MI(x, \ell) = \frac{I(x, \ell)}{\sqrt{H(x)H(\ell)}}$ is the normalized mutual information where $H(.)$ is the entropy function and the value $MI(.,.)$ lies in $[0,1]$. Note that if we only have one label (i.e., $|L|=1$), and $p=d$ (the number of all features of the dataset) then $g$ will be exactly the modular function used in~\citet{ZadehGMZ17}. Therefore, our formulation is a generalization of theirs. Using the $\maxp$ function, this formulation tries to select at least $p$ relevant features for each label. In order to understand the importance of $\maxp$ function, we discuss two extreme cases: $p=1$ and $p=d$. If $p=1$ then a feature that is somewhat relevant to all the features can dominate the $g(S)$ and prevent other features, that are highly relevant to one or few features, to get selected. If $p=d$ then a label that has a lot of relevant features can dominate $g(S)$ and prevent other labels to get relevant features, while a few features would be enough for predicting this label with a high accuracy. 
In the following lemma, we show that $g$ has the nice properties we need in our model. Its proof is included in Appendix A.

\begin{lemma} 
$g$ is a non-negative, monotone, submodular function.
\end{lemma}

Hence if we define $f(S)=g(S)+\sum_{\{u,v\}\in S} d(u,v)$, then our  feature selection model reduces to solving the following combinatorial optimization problem.

\vspace{-0.4cm}
\begin{equation}
\label{eqn:defn}
\max_{\substack{S\subseteq U \\ |S| = k}} f(S) = \max_{\substack{S\subseteq U \\ |S| = k}} \{g(S) + \sum_{\{u,v\}\in S} d(u,v)\},
\end{equation}
\vspace{-0.4cm}

\noindent
where $d(.,.)$ is a metric distance and $g(.)$ is a non-negative monotone submodular function. In the actual feature selection method we are free to scale the relative contributions of the diversity or submodular parts, since  both metric and submodular functions are closed under multiplication by a positive constant. Hence, we use a weighted version of the objective function in our application.

\vspace{-0.2cm}
\RestyleAlgo{algoruled}
\begin{algorithm}
\footnotesize
\textbf{Input:} Set of features $U$, set of labels $L$, number of features we want to select $k$.\\[0.6ex]
 \textbf{Output:} Set $S \subset U$ with $|S| = k$.\\[0.6ex]
$S \leftarrow \{\argmax_{u \in U} g(\{u\})\}$\; 
\ForAll{$2 \leq i \leq k$}{
	$u^* \leftarrow \argmax\limits_{u \in U \setminus S}\ \ g(S\cup \{u\}) - g(S) + \sum\limits_{x\in S} d(x, u)$\; 
	\Comment{\ }\parbox{6cm}{{This $\argmax$ has a consistent tiebreaking rule (see Definition 1).}} \\ 
Add $u^*$ to $S$\;
}
Return $S$\;
\caption{Greedy}
\label{alg:AlgGMM}
\end{algorithm}
\vspace{-0.2cm}

The problem  (\ref{eqn:defn}) is NP-hard but~\citet{BorodinJLY17} show that Algorithm~\ref{alg:AlgAGMM} is a half approximation in the centralized setting. Note that this is a greedy algorithm under the objective where $g(S)$ is scaled  by $\frac{1}{2}$.
  On the other hand,  Algorithm~\ref{alg:AlgGMM} is a standard greedy algorithm for (\ref{eqn:defn}) and in the next section we show it is a constant factor randomized composable core-set for any functions $f$ which are the sum of a sum-sum diversity function and a non-negative, monotone, submodular function. Combining these  we conclude that Algorithm~\ref{alg:DistAlgGMM} is a constant factor approximation algorithm for maximizing $f$. Moreover, Algorithm~\ref{alg:DistAlgGMM} can be used both in distributed and streaming settings, as illustrated in Figure~\ref{Fig:systems}.
In our experiments, to select $k$ features, we use the following function.
\vspace{0.0cm}
\begin{align}
\label{eqn:objform}
h(S) = (1-\lambda)\frac{k(k-1)}{2p|L|} g(S) + \lambda \sum_{x_i, x_j \in S} d(x_i,x_j).
\end{align}
\vspace{-0.4cm}

\noindent
As discussed, the first term of $h(S)$ controls redundancy of the selected features and the second term is to promote features that are relevant to the labels. The term $\frac{k(k-1)}{2p|L|}$ is a normalization coefficient to make the range of both terms the same. Also, $\lambda$ is a hyper-parameter which controls the effect of two criteria on the final function.

\vspace{-0.1cm}
{\RestyleAlgo{algoruled}
\begin{algorithm}[h]
\footnotesize
\textbf{Input:} Set of features $U$, set of labels $L$, number of features we want to select $k$.\\[0.6ex]
 \textbf{Output:} Set $S \subset U$ with $|S| = k$.\\[0.6ex]
$S \leftarrow \{\argmax_{u \in U} g(\{u\})\}$\; 
\ForAll{$2 \leq i \leq k$}{
	$u^* \leftarrow \argmax\limits_{u \in U \setminus S}\ \ \frac{1}{2}(g(S\cup \{u\}) - g(S)) + \sum\limits_{x\in S} d(x, u)$\;
Add $u^*$ to $S$\;
}
Return $S$\;
\caption{AltGreedy}
\label{alg:AlgAGMM}
\end{algorithm}
}
\vspace{-0.4cm}

\section{Theoretical Results}
Let $f(S) = D(S) + g(S)$ be a set function defined on $2^U$ where $g(S)$ is a non-negative, monotone, submodular function and $D(S)$ is a sum-sum diversity function, i.e. $D(S) = \sum_{\{u,v\} \in S} d(u,v)$ where $d(.,.)$ is a metric distance. In this section, we show that Algorithm~\ref{alg:AlgGMM} is a constant factor randomized composable core-set with size $k$ for $f$. We also show that running Algorithm~\ref{alg:DistAlgGMM} which is equivalent to running Algorithm~\ref{alg:AlgGMM} in each slave machine and then running Algorithm~\ref{alg:AlgAGMM} in the master machine on the union of outputs of slave machines is a constant factor randomized approximation algorithm for maximizing $f$ subject to a cardinality constraint. 

\vspace{-0.2cm}
\RestyleAlgo{algoruled}
\begin{algorithm}
\footnotesize
\textbf{Input:} Set of features $U$, set of labels $L$, number of features we want to select $k$, number of machines $m$.\\[0.6ex]
 \textbf{Output:} Set $S \subset U$ with $|S| = k$.\\[0.6ex]
Randomly partition $U$ into $(T_{i})_{i=1}^m$\;
\ForAll{$1 \leq i \leq m$}{
	$S_{i} \leftarrow$ output of $Greedy$($T_{i}$, $L$, $k$)\;
}
$S \leftarrow $ output of  $AltGreedy$($\cup_{i=1}^m S_{i}$, $L$, $k$)\;
Return $S$\;
\caption{Distributed Greedy}
\label{alg:DistAlgGMM}
\end{algorithm}
\vspace{-0.2cm}

We use the following key concept of a $\beta$-nice algorithm from~\citet{MirrokniZ15} throughout our analysis.

\begin{definition}
Let  $f$ be a set function on $2^U$. Let $\texttt{ALG}$ be an algorithm that given any $T \subseteq U$ outputs $\texttt{ALG}(T) \subseteq T$.
 Let $t \in T \setminus \texttt{ALG}(T)$. For $\beta \in \mathbb{R}^+$, we call $\texttt{ALG}$ a $\beta$-nice algorithm if it has the following properties.
 \vspace{-0.7cm}
\begin{itemize}
\setlength\itemsep{0.05cm}
\item $\texttt{ALG}(T) = \texttt{ALG}(T \setminus \{t \})$.
\item $f(\texttt{ALG}(T) \cup \{t \}) - f(\texttt{ALG}(T)) \leq \beta \frac{f(\texttt{ALG}(T))}{k}$.
\end{itemize}
\vspace{-0.4cm}
\end{definition}

The intuition behind the first condition is simply that by removing an element of $T$ which is not used in the algorithm's output, we do not change the output.
This is effectively a condition on how we perform tiebreaking. The second condition helps to bound $f(\texttt{ALG}(T) \cup O)$ where $O$ is a global optima.
Our theoretical analysis heavily relies on the following theorem which is proved in Appendix B.

\begin{theorem}
Let $k \geq 10$.
Algorithm~\ref{alg:AlgGMM} is a $5$-nice algorithm for $f(.) = D(.) + g(.)$.
 Also, if $\texttt{ALG}$ is Algorithm~\ref{alg:AlgGMM}, $T\subseteq U$, and $t\in T\setminus \texttt{ALG}(T)$,  then $\frac{4.5}{k-1} f(\texttt{ALG}(T)) \geq \sum_{x \in \texttt{ALG}(T)} d(t, x)$.
\label{bnice}
\end{theorem}

Our main result is that Algorithm~\ref{alg:DistAlgGMM} is a constant factor approximation algorithm.

\begin{figure}[t]
\centering
\begin{subfigure}{0.9\columnwidth}
\includegraphics[width=\columnwidth]{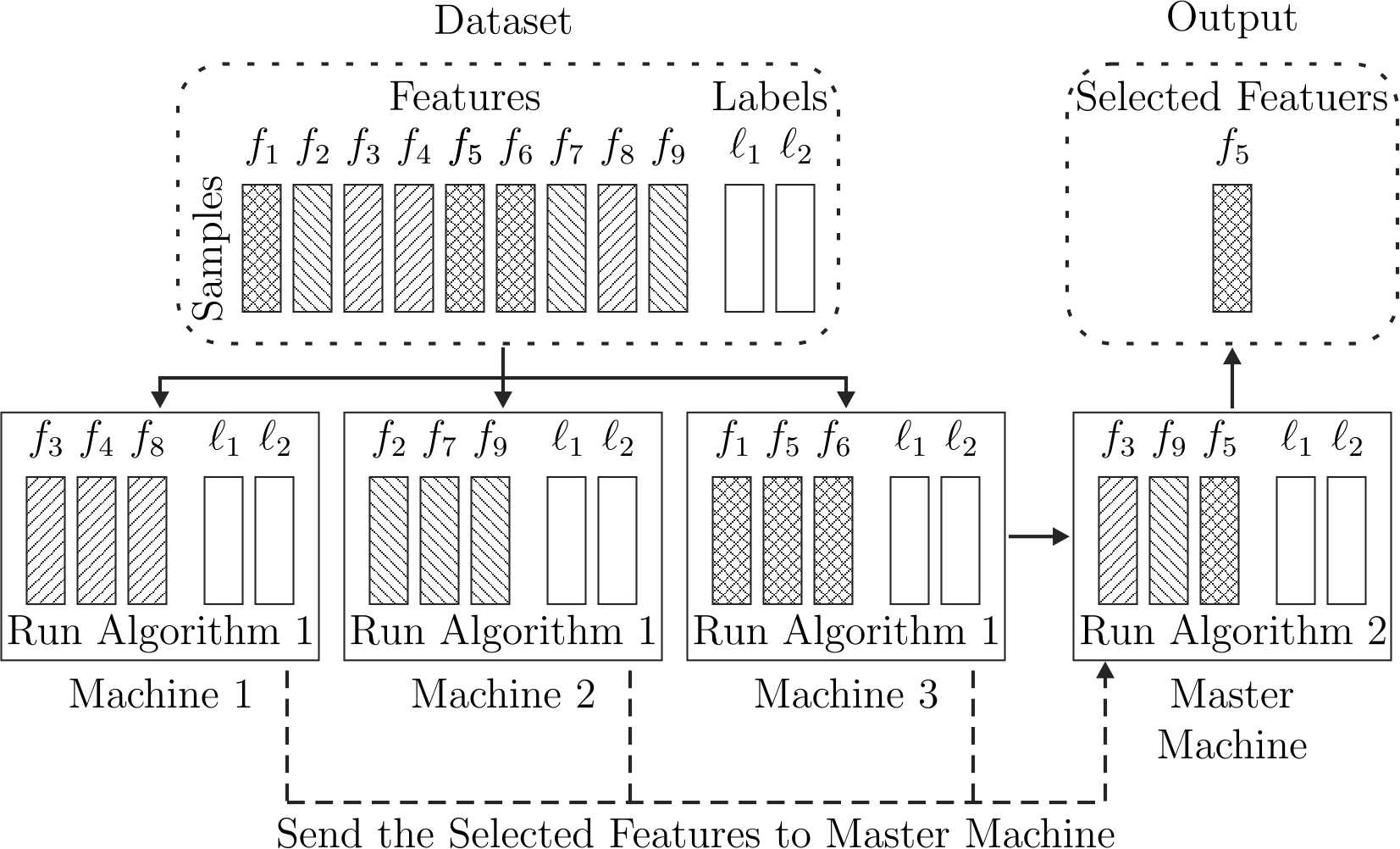}
\subcaption{Distributed setting}
\end{subfigure}
\hfill
\begin{subfigure}{0.9\columnwidth}
\includegraphics[width=\columnwidth]{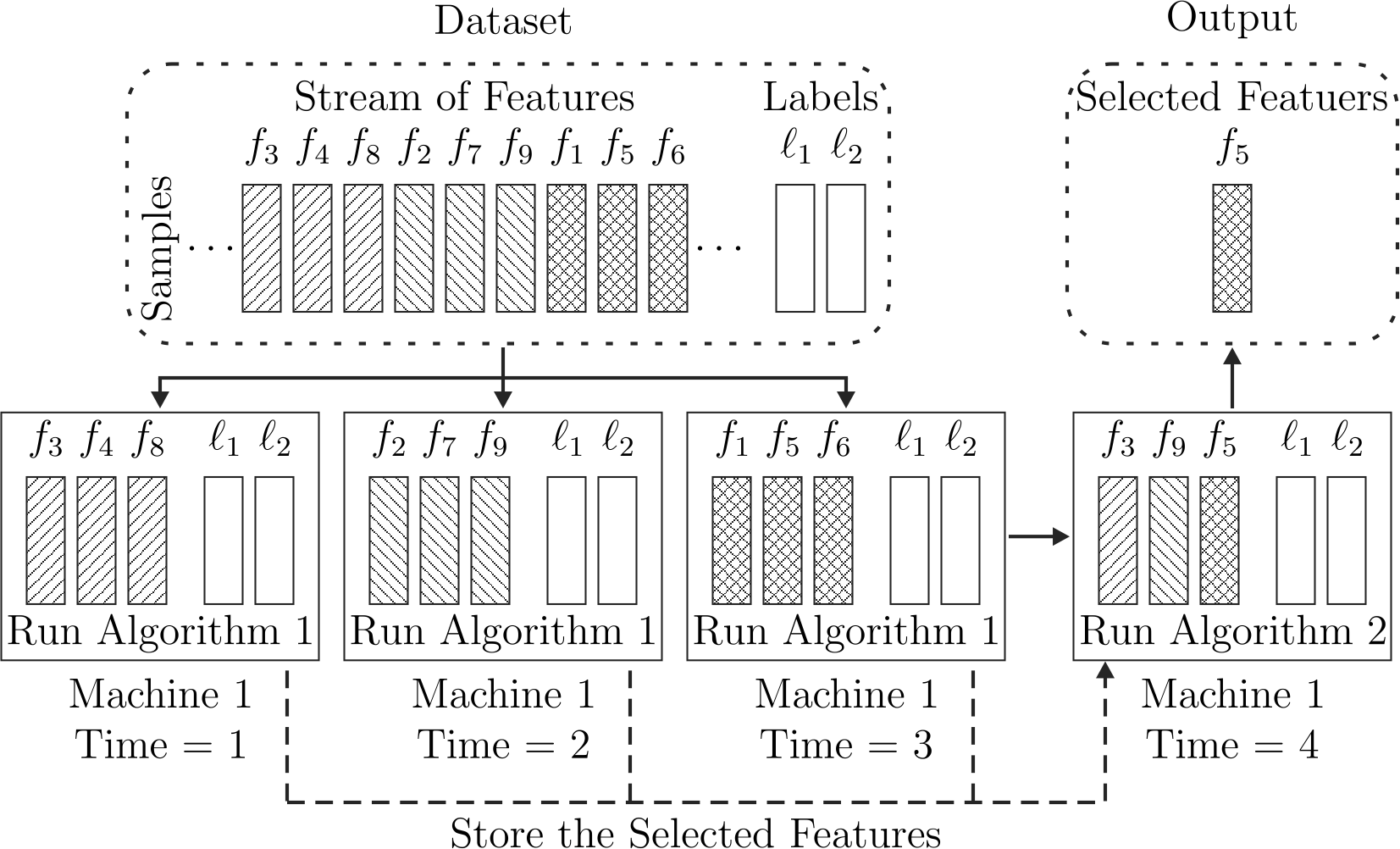}
\subcaption{Streaming setting}
\end{subfigure}
\caption{Algorithm~\ref{alg:DistAlgGMM} operating in big data settings.}
\label{Fig:systems}
\vspace{-0.3cm}
\end{figure}

\begin{theorem}
\label{thm:main}
Let $k \geq 10$.
Algorithm~\ref{alg:DistAlgGMM} gives a $\frac{1}{31}$-approximate solution in expectation for maximizing $f(S)$ subject to $|S| = k$.
\end{theorem}

We note that for $k < 10$, the constant degrades so we focus on the large $k$ regime.
The proof of this theorem follows from two key lemmas which bound the diversity and submodular portions
of an optimal solution. We use $O$ to denote a global optimum. To state the lemmas, we need the following notations. Let $\texttt{OPT}(T) = \argmax_{R\subseteq T} f(R)$ subject to $|R| = k$. Let $U$ be the set of all elements (for example, the set of all features for the feature selection problem) and $(T^1, \ldots, T^m)$ be a random partitioning of the elements of $U$.

\begin{lemma}
Let $\texttt{ALG}$ be Algorithm~\ref{alg:AlgGMM} and $S^i=\texttt{ALG}(T^i)$. Then $D(O) \leq 8.5f(\texttt{OPT}(\cup_{i=1}^m S^i))$.
\label{boundD}
\end{lemma}

\begin{lemma}
\label{boundg}
Let $\texttt{ALG}$ be Algorithm~\ref{alg:AlgGMM} and $S^i=\texttt{ALG}(T^i)$. Then $g(O) \leq  6f(\texttt{OPT}(\cup_{i=1}^m S^i)) + \mathbb{E}[f(\texttt{OPT}(\cup_{i=1}^m S^i))]$.
\end{lemma}

We use Theorem~\ref{bnice} and techniques from a number of papers \citep{ZadehGMZ17,IndykMMM14,MirrokniZ15,AghamolaeiFZ15} to prove these two key lemmas in Appendix B. Even in cases where some parts of proofs are similar to previous work we include a complete proof for the sake of completeness. We should note that our analysis is not a straightforward combination of the ideas in the mentioned papers. Using Lemma~\ref{boundD} and~\ref{boundg}, we can easily prove Theorem~\ref{thm:main}.

\begin{proofof}{Theorem~\ref{thm:main}.}
Lemma~\ref{boundD} and \ref{boundg} immediately yield
$
f(O) \leq 15.5\mathbb{E}[f(\texttt{OPT}(\cup_{i=1}^m S^i))]
$.
Based on~\citet{BorodinJLY17}, we know that Algorithm~\ref{alg:AlgAGMM} is a half approximation algorithm for maximizing $f$. Therefore, if $\texttt{ALG'}$ is Algorithm~\ref{alg:AlgAGMM} then $f(\texttt{OPT}(\cup_{i=1}^m S^i)) \leq 2f(\texttt{ALG'}(\cup_{i=1}^m S^i))$. Hence
$
f(O) \leq 31\mathbb{E}[f(\texttt{ALG'}(\cup_{i=1}^m S^i))]
$
which is exactly the statement of the theorem.
\end{proofof}

\begin{table}[t]
\centering
\begin{scriptsize}
    \caption{Specifications of the datasets.}
    \vspace{-0.2cm}
    \label{table:dataset}
    \begin{tabular}{cccc}
        \toprule Dataset Name & \# Features & \# Instances & \# Labels \\
        \midrule Corel5k & 499 & 5000 & 374 \\
        \midrule Eurlex-ev & 5000 & 19,348 & 3993 \\
        \midrule Synthesized & 800 & 256 & 8 \\
        \bottomrule
        \end{tabular}
\end{scriptsize}
\vspace{-0.3cm}
\end{table}

\section{Empirical Results}
In this section, we investigate the performance of our method in practice. In the first experiment, we compare our distributed method with centralized multi-label feature selection methods in the literature on a classification task. We show that our method's performance is comparable to, or in some cases is even better than previous centralized methods. Next, we compare our distributed and centralized methods on two large datasets. We show that the distributed algorithm achieves almost the same objective function value and it is much faster. This implies that the distributed algorithm achieves a better approximation in practice compared to the theoretical guarantee.

\begin{scriptsize}
\begin{table*}[h]
    \centering
    \caption{Comparison of the distributed and the centralized algorithms. ``h'' and ``m'' means hour and minute.}   
    \vspace{-0.3cm}
    \label{table:centVsDist}
    \resizebox{\textwidth}{!}{
    \begin{tabular}{cccccccccccc} 
        \toprule
        {\specialcell{Dataset\\Name}} & {Reference} & {\# Features} & {\# Instances} & {\# Labels} & {\specialcell{\# Selected\\Features}} & {\# Machines} & {\specialcell{Distributed\\Algorithm\\Objective\\Value}} & {\specialcell{Centralized\\Algorithm\\Objective\\Value}} & {\specialcell{Distributed\\Algorithm\\Runtime}} & {\specialcell{Centralized\\Algorithm\\Runtime}} & {Speed-up}\\ 
        \midrule
        \multirow{4}{*}{RCV1V2} & \multirow{4}{*}{\citep{LewisYRL04}} & \multirow{4}{*}{47,236} & \multirow{4}{*}{6000} & \multirow{4}{*}{101} & 10 & 69 & 22.7 & 22.6 & 2.8m & 1h 33m & 33.2 \\ & & & & & 50 & 31 & 618.7 & 616.4 & 10.8m & 2h 30.0m & 15.1 \\ & & & & & 100 & 22 & 2468.2 & 2490.7 & 20.3m & 3h 39m & 10.8 \\ & & & & & 200 & 16 & 9338.7 & 10,016.0 & 47.0m & 6h 16.8m & 8.0 \\
        \midrule
        \multirow{4}{*}{TMC2007} & \multirow{4}{*}{\citep{srivastava2005discovering}} & \multirow{4}{*}{49,060} & \multirow{4}{*}{28,596} & \multirow{4}{*}{22} & 10 & 71 & 22.8 & 22.6 & 4.6m & 2h 32.5m & 33.4 \\ & & & & & 50 & 32 & 620.0 & 615.6 & 24.2m & 6h 24.7m & 15.9 \\ & & & & & 100 & 23 & 2510.0 & 2487.7 & 59.5m & 11h 6.2m & 11.2 \\ & & & & & 200 & 16 & 10,104.3 & 10,001.4 & 2h 41.3m & 20h 49.8m & 7.7 \\
        \bottomrule\\
    \end{tabular}}
\vspace{-0.7cm}
\end{table*}
\end{scriptsize}

\subsection*{Comparison to Centralized Methods}
As mentioned in Section 2, most of the multi-label feature selection methods convert the multi-label dataset to one or multiple single-label datasets and then use single-label feature selection methods and then aggregate the results. Binary relevance (BR) and label powerset (LP) are the two best known of these conversions. Here, we combine these two conversion methods with two single-label feature selection methods which results in four different centralized feature selection methods. We considered ReliefF (RF)~\citep{Kononenko94,Robnik-SikonjaK03} and information gain (IG)~\citep{zhao2010advancing} for single-label methods. These methods compute a score for each feature and for aggregating their results in Binary Relevance conversion, it is enough to calculate the sum of the scores of each feature and use these scores for selecting features. These methods are used before in the literature for multi-label feature selection \citep{ChenYZCY07,dendamrongvit2011irrelevant,spolaor2011using,SpolaorCML12,SpolaorCML13}.

\vspace{-0.4cm}
\begin{figure}[hb]
\centering
\begin{subfigure}[t]{0.48\columnwidth}
\centering
\includegraphics[width=1.05\linewidth]{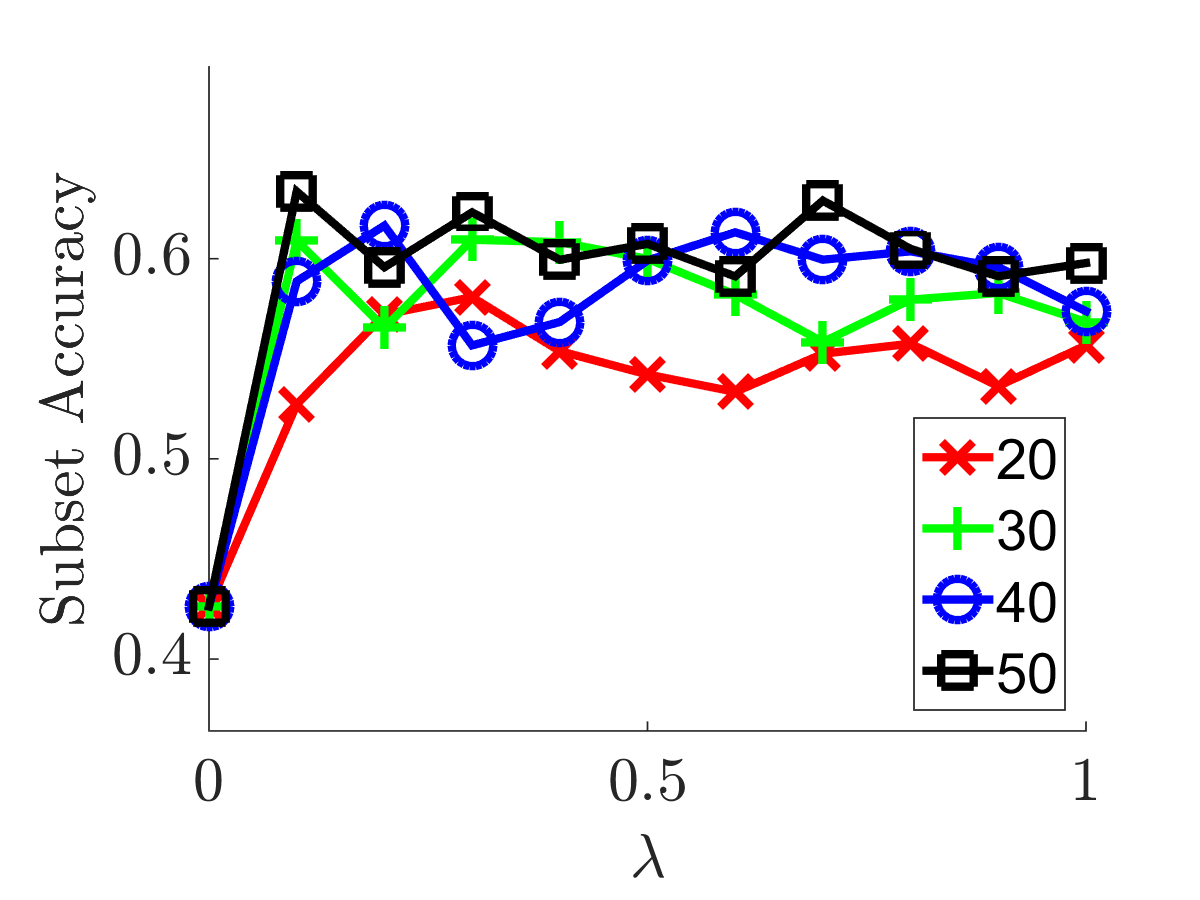}
\end{subfigure}
\begin{subfigure}[t]{0.48\columnwidth}
\centering
\includegraphics[width=1.05\linewidth]{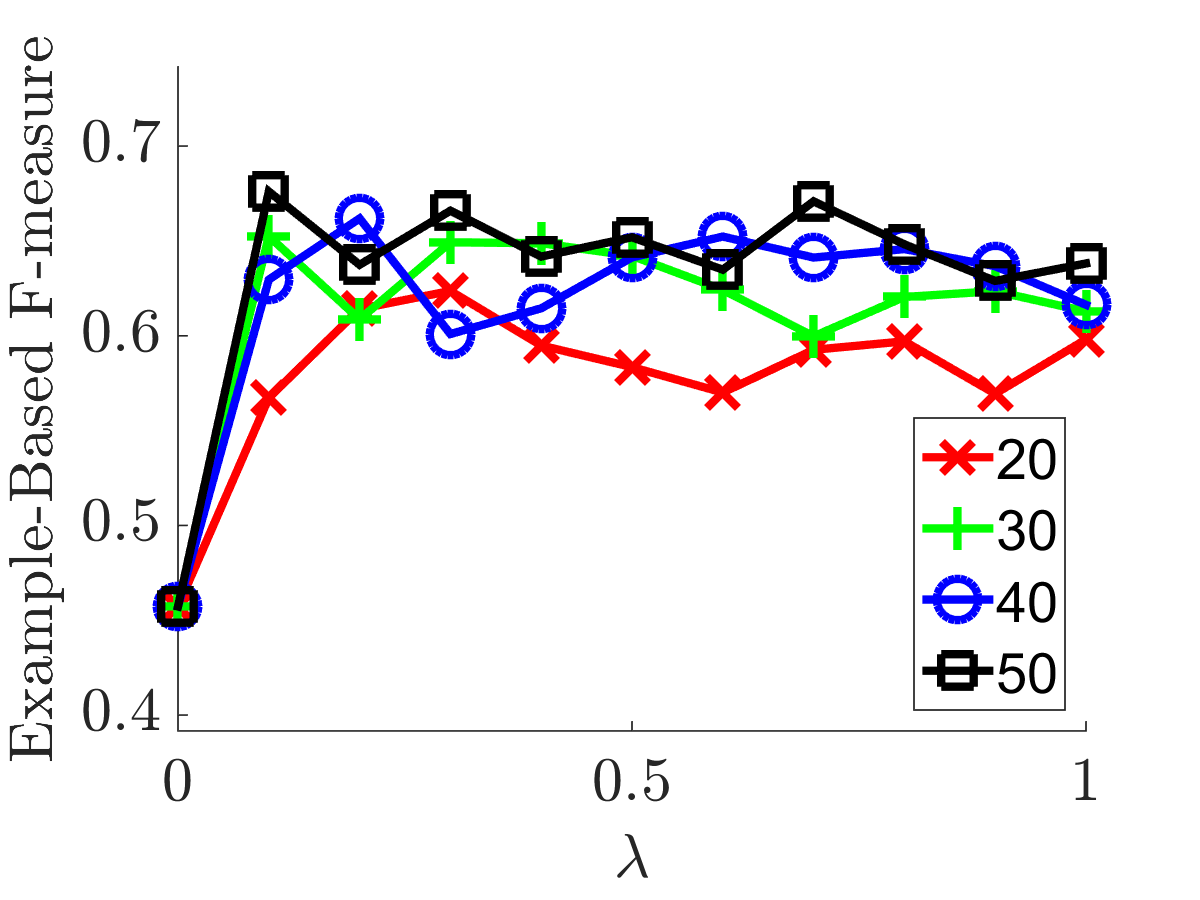}
\end{subfigure}
\vspace{-0.3cm}
\caption{Effect of $\lambda$ on the performance of the method.}
\label{Fig:lambda}
\end{figure}
\vspace{-0.3cm}

\begin{figure*}[t]
\centering
\begin{subfigure}[t]{0.27\textwidth}
\centering
\includegraphics[width=1.05\linewidth]{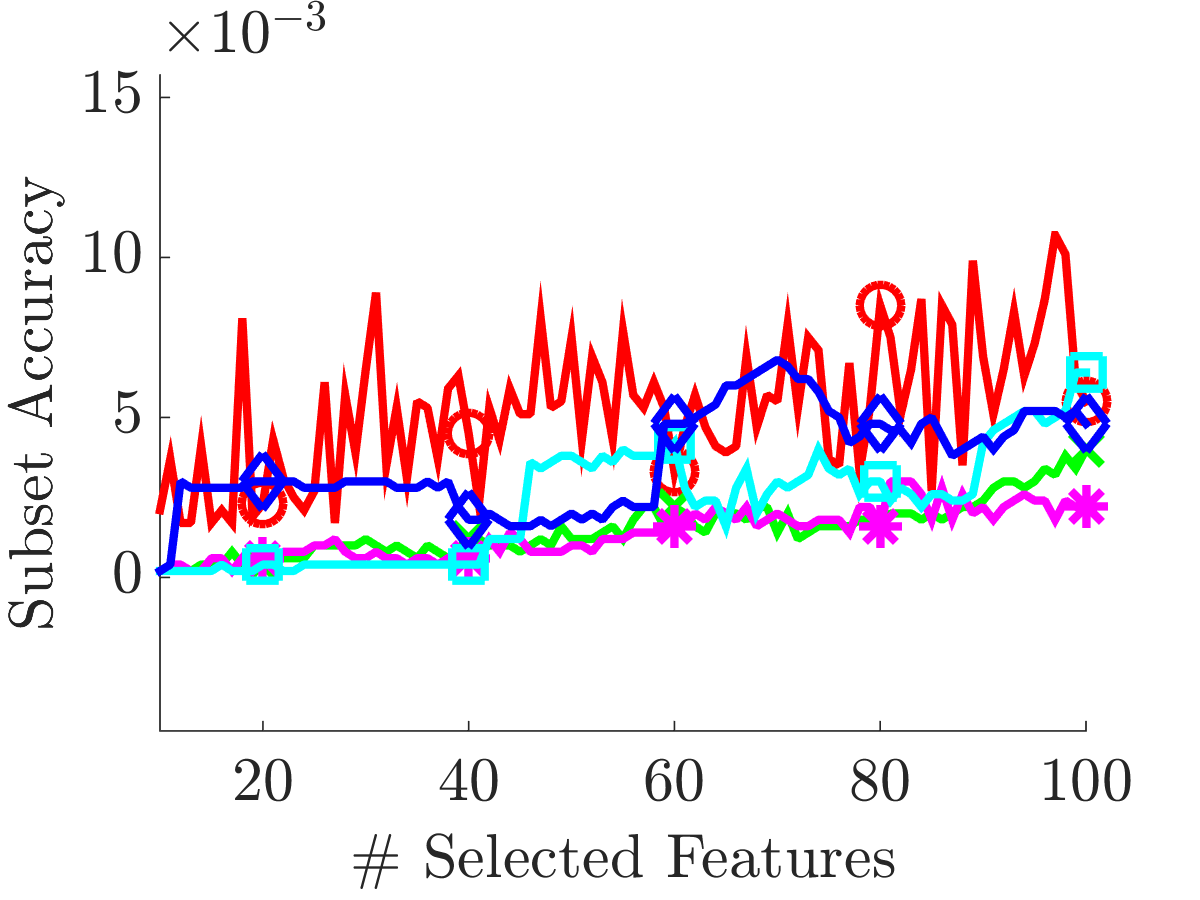}
\end{subfigure}
\begin{subfigure}[t]{0.27\textwidth}
\centering
\includegraphics[width=1.05\linewidth]{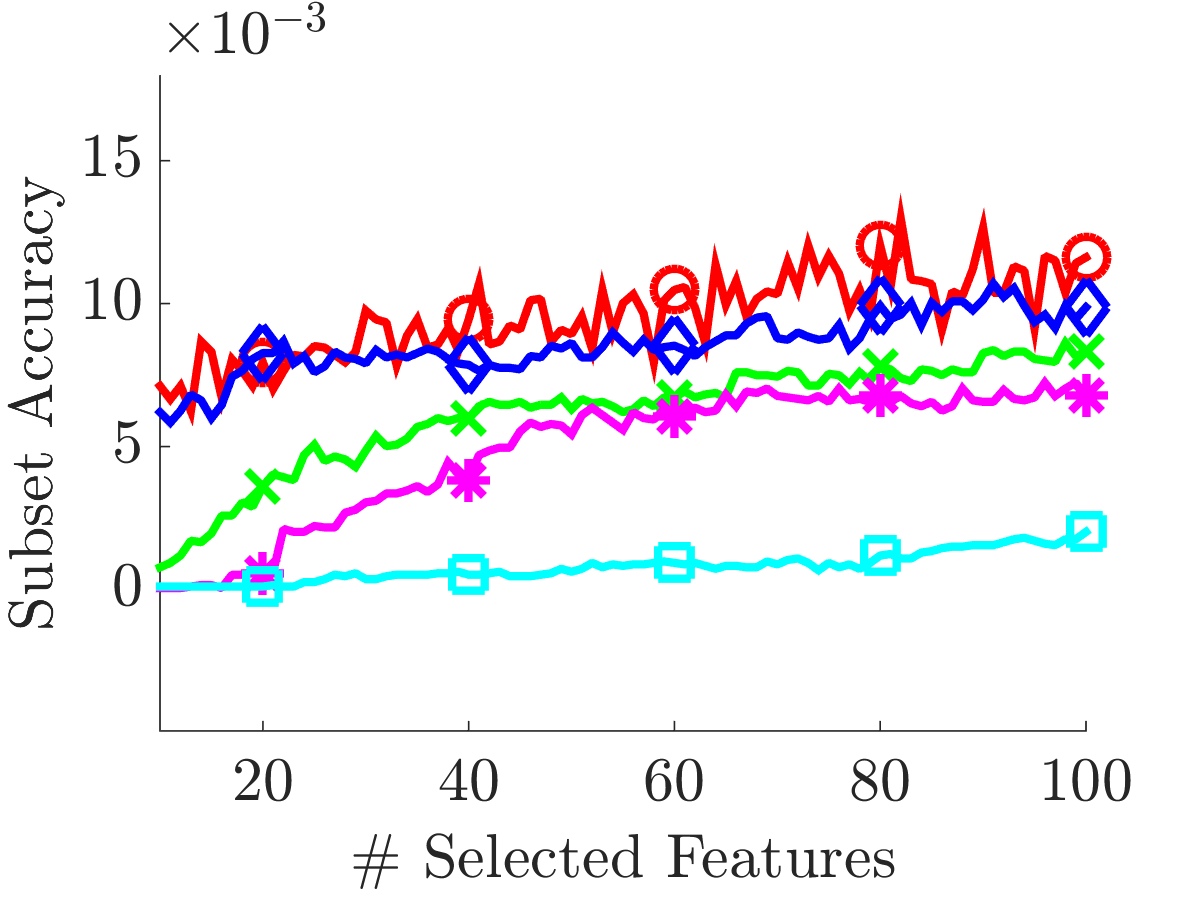}
\end{subfigure}
\begin{subfigure}[t]{0.27\textwidth}
\centering
\includegraphics[width=1.05\linewidth]{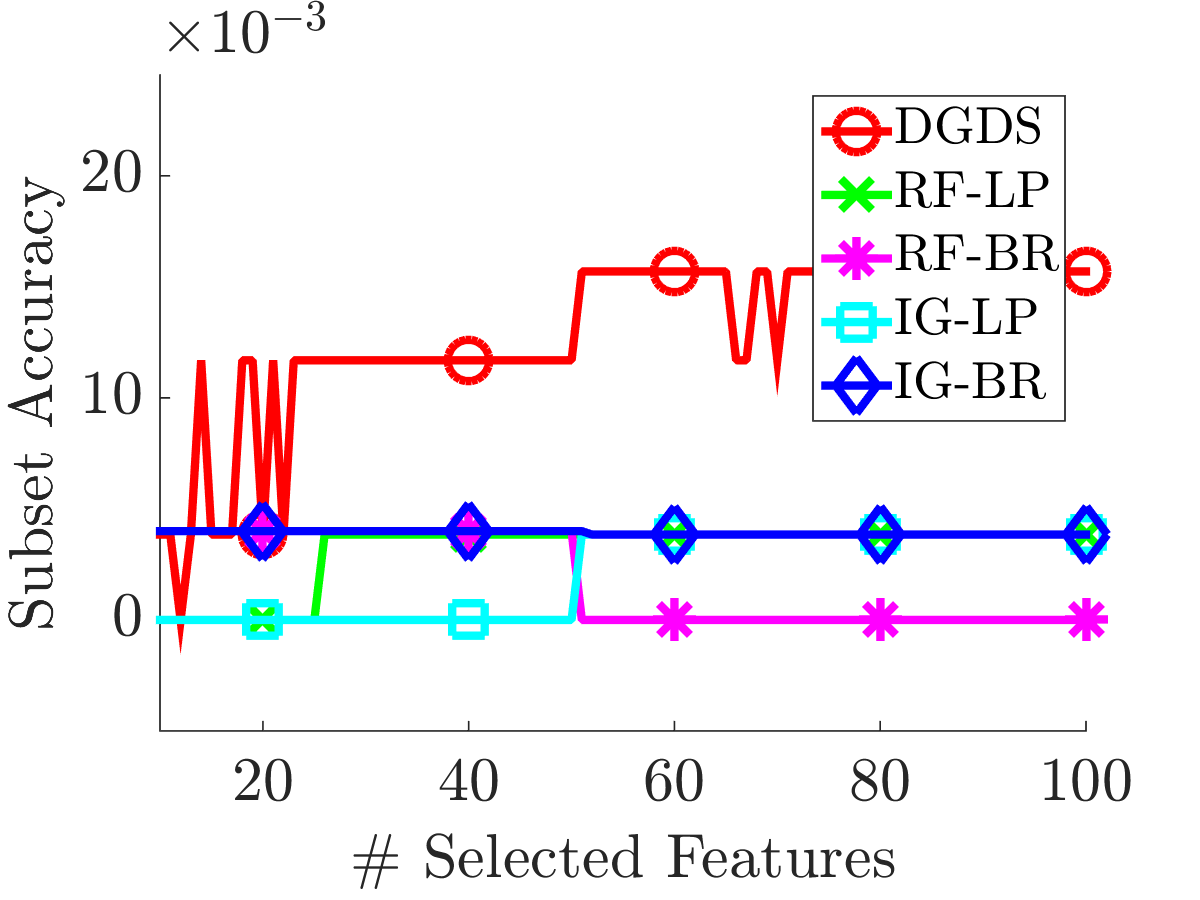}
\end{subfigure}
\begin{subfigure}[t]{0.27\textwidth}
\centering
\includegraphics[width=1.05\linewidth]{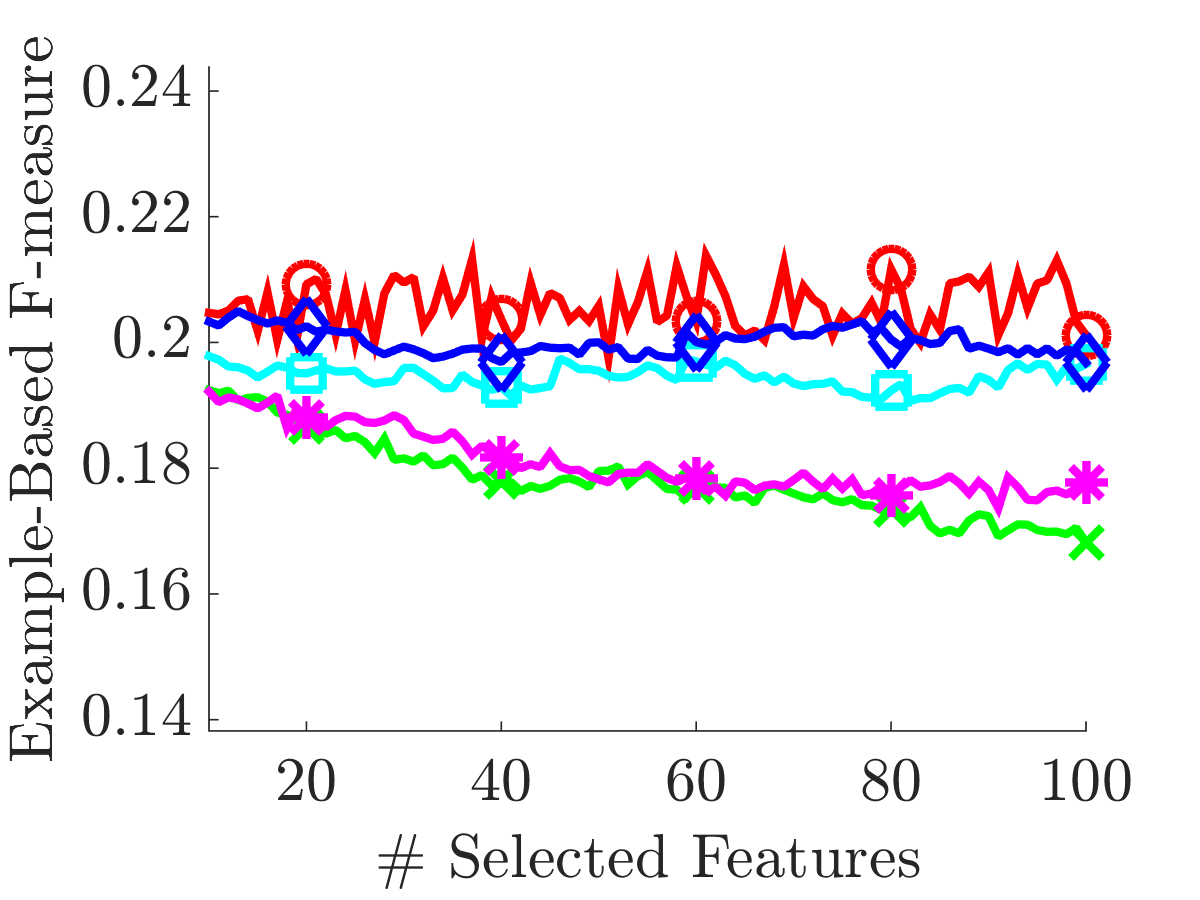}
\end{subfigure}
\begin{subfigure}[t]{0.27\textwidth}
\centering
\includegraphics[width=1.05\linewidth]{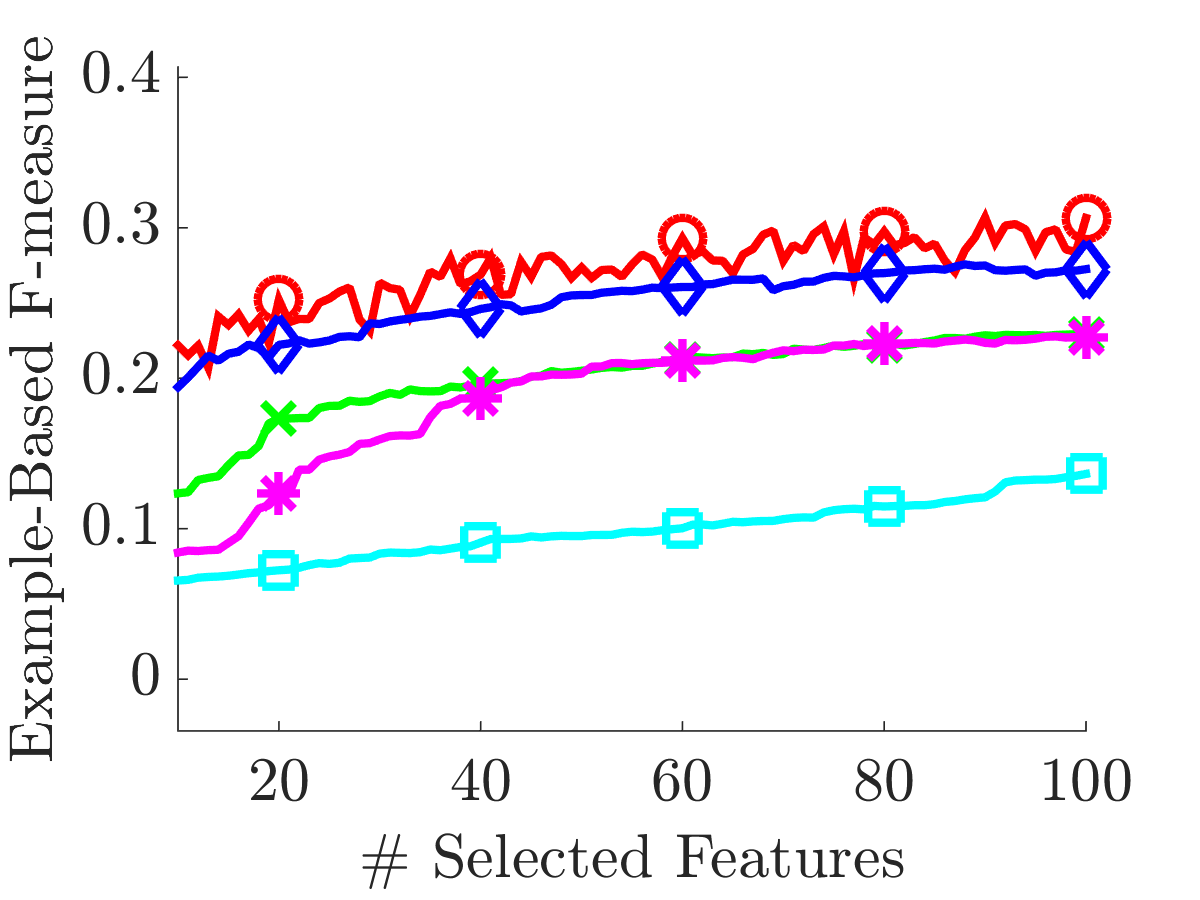}
\end{subfigure}
\begin{subfigure}[t]{0.27\textwidth}
\centering
\includegraphics[width=1.05\linewidth]{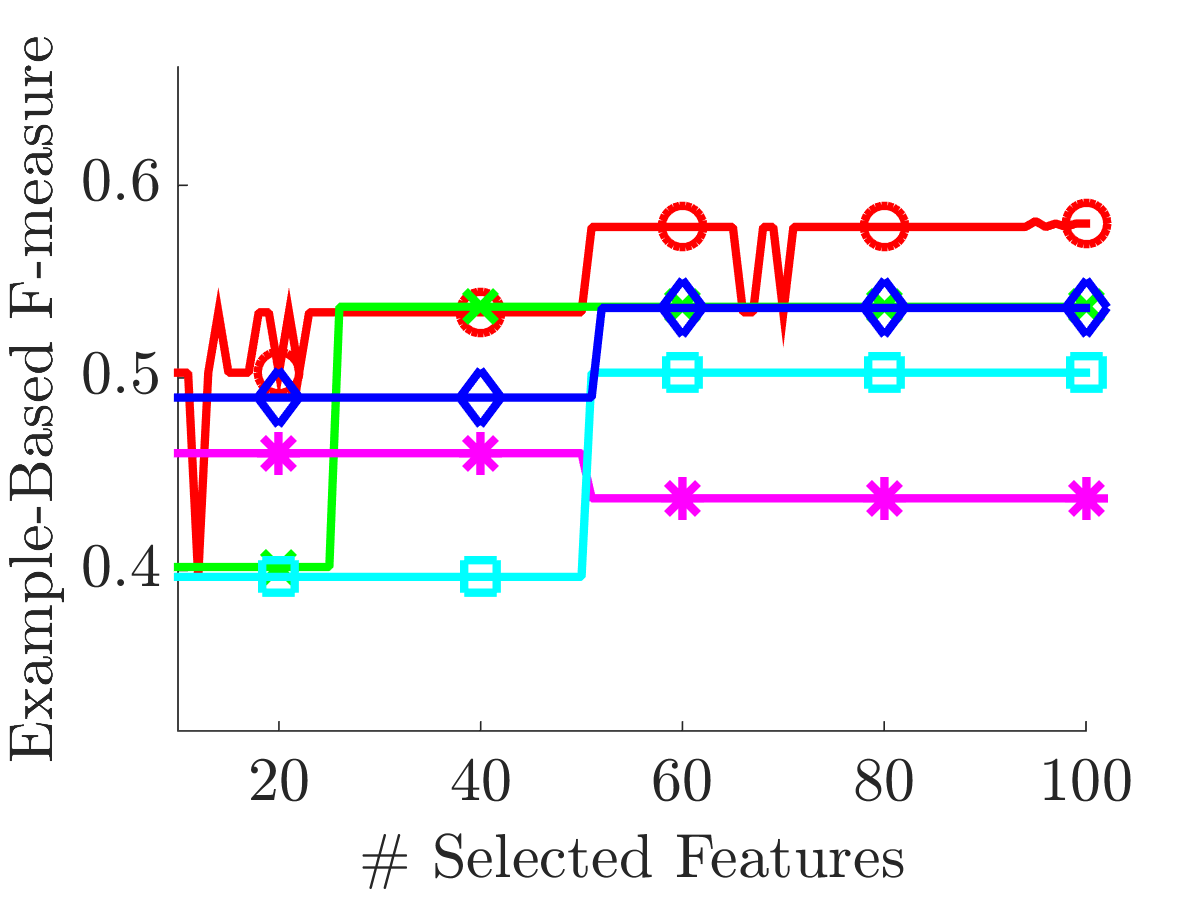}
\end{subfigure}
\begin{subfigure}[t]{0.27\textwidth}
\centering
\includegraphics[width=1.05\linewidth]{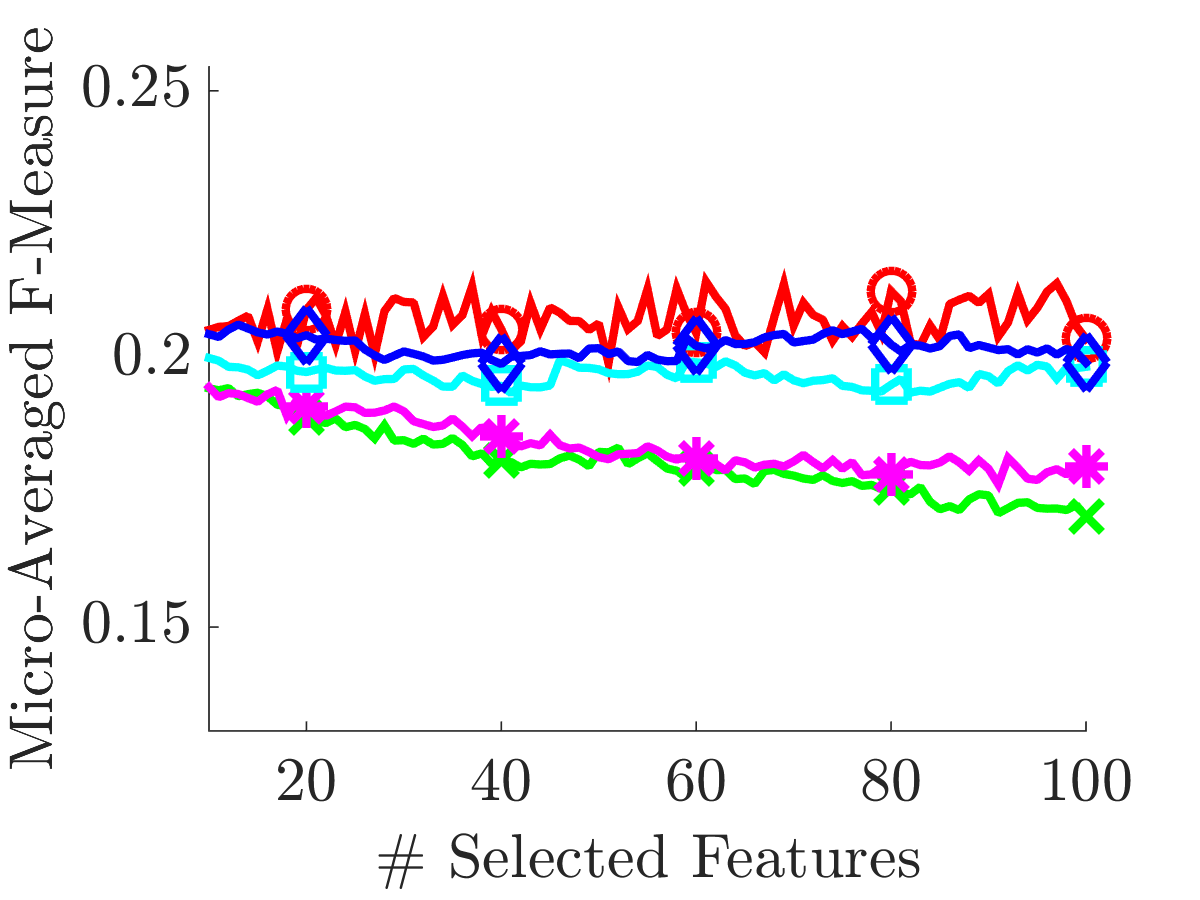}
\subcaption{Corel5k}
\end{subfigure}
\begin{subfigure}[t]{0.27\textwidth}
\centering
\includegraphics[width=1.05\linewidth]{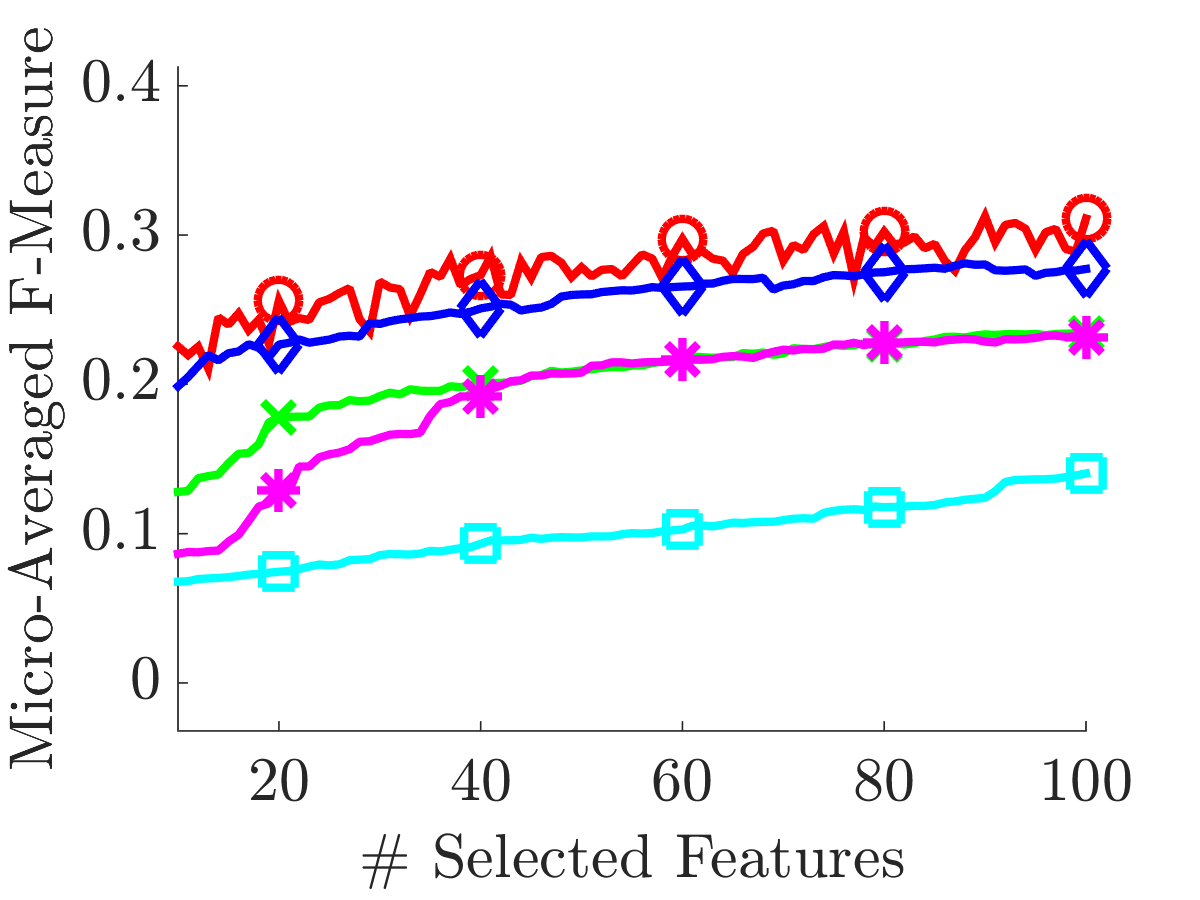}
\subcaption{Eurlex-ev}
\end{subfigure}
\begin{subfigure}[t]{0.27\textwidth}
\centering
\includegraphics[width=1.05\linewidth]{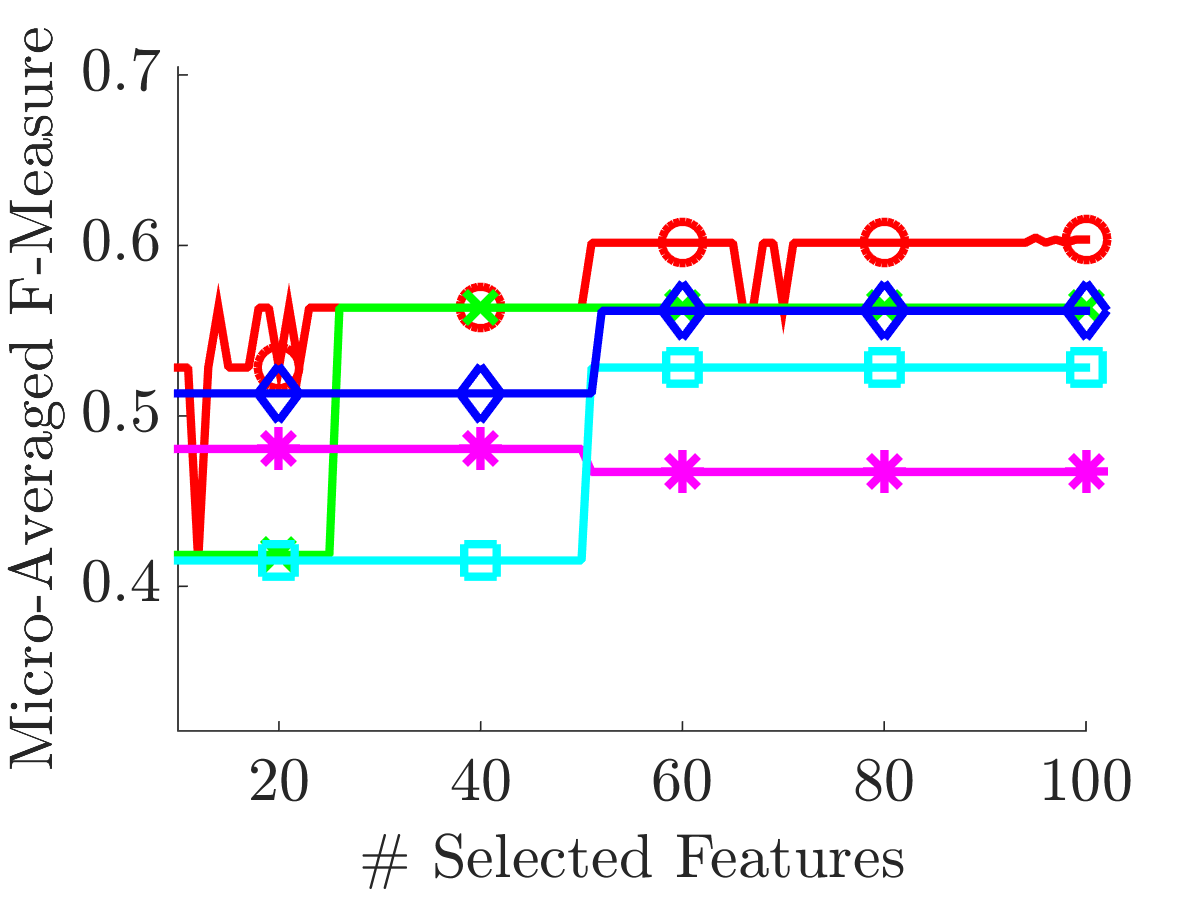}
\subcaption{Synthesized}
\end{subfigure}
\vskip\baselineskip
\vspace{-0.7cm}
\caption{Comparison of proposed distributed method with centralized methods in the literature.}
\label{Fig:comparison}
\vspace{-0.4cm}
\end{figure*}

For comparison, we selected 10 to 100 features with each method and did a multi-label classification using BRKNN-b proposed in~\citet{SpyromitrosTV08}. We did a 10-fold cross validation with five neighbors for BRKNN-b. We evaluated the classification outputs over five multi-label evaluation measures. They are subset accuracy, example-based accuracy, example-based F-measure, micro-averaged F-measure, and macro-averaged F-measure~\citep{SpolaorCML13,kashefmultilabel}. Evaluation measures are defined in Appendix C.

We used the Mulan library for the classification and computation of the evaluation measures~\citep{mulan}. We used a synthesized dataset and two real-world datasets-Corel5k~\citep{DuyguluBFF02} and Eurlex-ev~\citep{francesconi2010semantic}. Their specifications are shown in table~\ref{table:dataset}. The synthesized dataset made up of eight labels. Each label has two original features that repeated 50 times. One of the features has the same value as its label in half of the samples, and the other one has the same value as its label in a quarter of the samples. The results of this dataset show that our method outperforms other methods on a dataset with redundant features. The results of this experiments are shown in Figure~\ref{Fig:comparison}. Results of example-based accuracy and macro-average F-measure comparison for these datasets are included in Appendix D. We named our method distributed greedy diversity plus submodular (DGDS) in the plots. The other methods are named based on the conversion method they use (i.e., BR or LP) and the feature selection method they use (i.e., RF or IG). In the experiments, we used $\lambda=0.5$ and $\maxten$ for our method. Moreover, methods are compared on three other datasets in Appendix D. Results of the distributed method fluctuate more compared to other methods. The reason is that, for every number of features, we did the feature selection, including the random partitioning, from scratch. This caused more variation in its results but also showed that the method is relatively stable and does not produce poor quality results for different random partitionings.

As discussed, we compared our method to centralized feature selection methods because there is no distributed multi-label feature selection method prior to our work. We should note that this comparison is unfair to the distributed method because it uses much less of the data compared to centralized methods. For example, it does not use the relation (or the distance) between the features in different machines. The advantage of the distributed method is that it is much faster and scalable. This is supported by experiments on its speed-up (see Table~\ref{table:centVsDist}).

\subsection*{Comparison of Distributed and Centralized Algorithms}

Here, we compare the performance of our proposed algorithm (Algorithm~\ref{alg:DistAlgGMM}) with the centralized algorithm introduced in~\citet{BorodinJLY17} (Algorithm~\ref{alg:AlgAGMM}) on the optimization task. We compare the runtime and the value of the objective function the algorithms achieve. We select 10, 50, 100, and 200 features on two large datasets. If there are $d'$ features in a machine, and we want to select $k$ of them then the runtime of the machine is $\mathcal{O}(d'k)$. Therefore, if we have $\lceil\sqrt{d/k}\rceil$ slave machines then each of them has $\mathcal{O}(\sqrt{dk})$ features and its runtime is equal to $\mathcal{O}(k\sqrt{dk})$, where $d$ is the total number of features. Also, the master machine will have $\mathcal{O}(\sqrt{dk})$ features, and its runtime is $\mathcal{O}(k\sqrt{dk})$ which means the runtime complexity of the master machine and the slave machines are equal. If we increase or decrease the number of slave machines, then the running time of the master machine or the slave machines will increase which results in a lower speed-up. Hence, we set the number of slave machines equal to $\lceil\sqrt{d/k}\rceil$. The results show that in practice our proposed distributed algorithm achieves an approximate solution as good as the centralized algorithm in a much shorter time. The results are summarized in Table~\ref{table:centVsDist}. Moreover, we compared the distributed and the centralized algorithms on the classification task. Results of this experiment are included in Appendix E.

\subsection*{Effect of $\lambda$ hyper-parameter}

To show the importance of both terms of the objective function, redundancy (diversity function) and relevance (submodular function), we compared the performance of the method for different $\lambda$ value. We select 20, 30, 40, and 50 features on the scene dataset~\citep{BoutellLSB04}. As shown in Figure~\ref{Fig:lambda}, the best performance happens for some $\lambda$ between $0$ and $1$. This shows that both terms are necessary and it is possible to get better results by choosing $\lambda$ carefully.

\section{Conclusion}
In this paper, we presented a greedy algorithm for maximizing the sum of a sum-sum diversity function and a non-negative, monotone, submodular function subject to a cardinality constraint in distributed and streaming settings. We showed that this algorithm guarantees a provable theoretical approximation. Moreover, we formulated the multi-label feature selection problem as such an optimization problem and developed a multi-label feature selection method for distributed and streaming settings that can handle the redundancy of the features. Improving the theoretical approximation guarantee is appealing for future work. From the empirical standpoint, it would be nice to try other metric distances and other submodular functions for the multi-label feature selection problem.

\small
\bibliographystyle{unsrtnat}
\bibliography{ref}

\appendix
\section{Appendix A}
\label{appendix:A}
\begin{proofof}{Lemma~\ref{thm:main}.}
Clearly $g$ is non-negative and monotone.
Since the sum of submodular functions is a submodular function, 
We only need to show that $\maxp_{x\in S} \{MI(x, \ell)\}$ is submodular.  We assume that $\maxz_{x\in S} \{MI(x, \ell)\} = 0$. Let $S\subseteq T\subset U$ and $a\in U\setminus T$. We show that 
\begin{align*}
\maxp_{x\in S\cup \{a\}} & \{MI(x, \ell)\} - \maxp_{x\in S} \{MI(x, \ell)\} \\ & \geq \maxp_{x\in T\cup \{a\}} \{MI(x, \ell)\} - \maxp_{x\in T} \{MI(x, \ell)\}.
\end{align*}
We have two cases. If $MI(a,\ell)$ is not among the $p$ largest numbers of $\{I(x, \ell)|x\in S\cup \{a\}\}$ then both sides of the above inequality are zero. If $MI(a,\ell)$ is among the $p$ largest numbers of $\{I(x, \ell)|x\in S\cup \{a\}\}$ then the left hand side of the inequality is equal to $MI(a, \ell) - MI(b,\ell)$ where $b$ is the $p$'th largest number in $\{I(x, \ell)|x\in S\}$. The right hand side is equal to $\max \{0, MI(a, \ell) - MI(c,\ell)\}$ where $c$ is the $p$'th largest number in $\{I(x, \ell)|x\in T\}$. The $p$'th largest number in $\{I(x, \ell)|x\in T\}$ is greater than or equal to the $p$'th largest number in $\{I(x, \ell)|x\in S\}$ because $S\subseteq T$. Therefore, in this case $MI(a, \ell) - MI(b,\ell) \geq \max \{0, MI(a, \ell) - MI(c,\ell)\}$ and the inequality holds.
\end{proofof}

\section{Appendix B}
\label{appendix:B}
For $S\subseteq U$ and $x\in U \setminus S$, let $\Delta(x, S) = g(S\cup \{x\}) - g(S)$.
We now show that Algorithm~\ref{alg:AlgGMM} is a $\beta$-nice algorithm for $f$.
This is ultimately needed for the proof of both key lemmas.

\begin{proofof}{Theorem~\ref{bnice}.}
Let $\texttt{ALG}$ be the Algorithm~\ref{alg:AlgGMM}, $T\subseteq U$, $t\in T\setminus \texttt{ALG}(T)$, and $x_1,\ldots,x_k$ be the elements that $\texttt{ALG}$ selected in the order of selection. Also, let $S_i=\{x_1,\ldots,x_i\}$ and $S_0=\emptyset$. 

For the first property of $\beta$-nice algorithms it is enough to have a consistent tiebreaking rule for $\texttt{ALG}$. It is sufficient to fix an 
ordering on all elements of $U$ up front. If some iteration finds multiple elements with the same  maximum marginal gain, then it should select earliest one in the a priori  ordering. 

Now we prove the second property of the $\beta$-nice algorithms for $\texttt{ALG}$. Because of the greedy selection of $\texttt{ALG}$, we have the following inequalities.
\begin{align*}
\Delta(x_1,S_0) & \geq \Delta(t,S_0) \\
\Delta(x_2,S_1) + d(x_2,x_1) & \geq d(t,x_1) + \Delta(t,S_1) \\
\Delta(x_3,S_2) + \sum_{i=1}^2 d(x_3,x_i) & \geq \sum_{i=1}^2 d(t,x_i) + \Delta(t,S_2) \\
& \cdots \\
\Delta(x_k,S_{k-1}) + \sum_{i=1}^{k-1} d(x_k,x_i) & \geq \sum_{i=1}^{k-1} d(t,x_i) + \Delta(t,S_{k-1}) \\
\end{align*}

Adding these inequalities together gives the following inequality.
\begin{align}
\label{eqn:agg}
\nonumber g(S_k) + D(S_k) & \geq \sum_{i=1}^{k-1} (k-i) d(t,x_i) + \sum_{i=0}^{k-1} \Delta(t,S_{i}) \\ & \geq \sum_{i=1}^{k-1} (k-i) d(t,x_i) + k\Delta(t,S_{k}),
\end{align}
where the second inequality holds because of the submodularity of $g$. Note that
\begin{equation}
\label{eqn:increm}
f(\texttt{ALG}(T) \cup \{x \}) - f(\texttt{ALG}(T)) = \Delta(t,\texttt{ALG}(T)) + \sum_{x \in \texttt{ALG}(T)} d(x,t).
\end{equation}

\noindent
One may thus note that  if the right-hand side coefficients in (\ref{eqn:agg})
were all $k/2$ (instead of $k-i$) we would have $2$-niceness of the algorithm. Our strategy is to achieve this  by shifting some of the ``weight'' from coefficients 
where $k-i > k/2$ to coefficients $<k/2$. This uses the metric inequality since $d(x_{k-i},x_i) + d(x_i,t) \geq d(x_{k-i},t)$. Hence if we added $d(x_{k-i},x_i)$ to both sides of (\ref{eqn:agg}), then we may increase the coefficient of $d(t,x_{k-i})$ by $1$
at the expense of  reducing the coefficient of $d(t,x_i)$ by $1$. 

We use this idea to  fix all of the ``small'' components in bulk by adding
a batch of {\em distinct} distances to both sides of (\ref{eqn:agg}). Since these distances are distinct, we increase the left-hand side by at most $D(S_k)$. In particular, the new left-hand side will be at most $2 (g(S_k)+D(S_k))$.

The batch of distances we add to both sides of the inequality is
$\sum_{i = \lceil \frac{k}{2} \rceil + 1}^k \sum_{j=1}^{i - \lfloor \frac{k}{2} \rfloor - 1} d(x_i,x_j)$. 
Clearly these distances are  distinct so we now need to make sure that the strategy produces the desired coefficients of terms $d(t,x_i)$. More formally, we claim that the following inequality holds.
\begin{claim}
\label{claim:weightShifitng}
\begin{align*}
\sum_{i=1}^{k-1}(k-i)d(t,x_i) & +\sum_{i=\lceil\frac{k}{2}\rceil+1}^k\sum_{j=1}^{i-\lfloor\frac{k}{2}\rfloor-1}d(x_i,x_j) \\ & \geq \sum_{i=1}^{k}(\lceil\frac{k}{2}\rceil-1)d(t,x_i)
\end{align*}
\end{claim}

We prove this claim later. Using this we have the following.

\begin{align*}
2(g(S_k) + & D(S_k)) \\ & \geq g(S_k) + D(S_k) + \sum_{i = \lceil \frac{k}{2} \rceil + 1}^k \sum_{j=1}^{i - \lfloor \frac{k}{2} \rfloor - 1} d(x_i,x_j) \\ & \geq \sum_{i=1}^{k-1} (k-i) d(t,x_i) + k\Delta(t,S_{k})  \\ & + \sum_{i = \lceil \frac{k}{2} \rceil + 1}^k \sum_{j=1}^{i - \lfloor \frac{k}{2} \rfloor - 1} d(x_i,x_j) \\ & \geq \sum_{i=1}^{k} (\lceil \frac{k}{2} \rceil - 1) d(t,x_i) + (\lceil \frac{k}{2} \rceil - 1)\Delta(t,S_{k})
\end{align*}

where the second inequalities holds because of the metric property, i.e. triangle inequality, and monotonicity of $g$. By using the above inequality, non-negativity of $g$, and (\ref{eqn:increm}) we  have

\begin{align*}
\frac{2}{\lceil \frac{k}{2} \rceil - 1}f(\texttt{ALG}(T)) & = \frac{2}{\lceil \frac{k}{2} \rceil - 1}(g(S_k) + D(S_k)) \\ & \geq \sum_{i=1}^{k} d(t,x_i) + \Delta(t,S_{k}) \\ & = f(\texttt{ALG}(T) \cup \{t \}) - f(\texttt{ALG}(T)).
\end{align*}

We can easily see that for $k \geq 10$, $\frac{5}{k} \geq \frac{2}{\lceil \frac{k}{2} \rceil - 1}$ and $\frac{4.5}{k-1} \geq \frac{2}{\lceil \frac{k}{2} \rceil - 1}$. Therefore, $\texttt{ALG}$ is a $5$-nice algorithm for $f$ and because of monotonicity of $g$, $\frac{4.5}{k-1} f(\texttt{ALG}(T)) \geq \sum_{i=1}^k d(t, x_i)$.
\end{proofof}

Now we prove Claim~\ref{claim:weightShifitng} to conclude Theorem~\ref{bnice}.

\begin{proofof}{Claim~\ref{claim:weightShifitng}.}
Note that $k=\lceil\frac{k}{2}\rceil+\lfloor\frac{k}{2}\rfloor$ and $\lfloor\frac{k}{2}\rfloor+1\geq\lceil\frac{k}{2}\rceil$. First, we show that 

\begin{equation}
\label{eq:ws}
\sum_{j=1}^{k-\lfloor\frac{k}{2}\rfloor-1}(\lceil\frac{k}{2}\rceil-j)d(t,x_j)=\sum_{i=\lceil \frac{k}{2}\rceil+1}^k\sum_{j=1}^{i-\lfloor\frac{k}{2}\rfloor-1}d(t,x_j).
\end{equation}

In the right hand side of (\ref{eq:ws}), $d(t,x_j)$ appears in the inner summation when $i-\lfloor\frac{k}{2}\rfloor-1\geq j$ or equivalently, when $i\geq j+\lfloor\frac{k}{2}\rfloor+1$. We know that $k\geq i\geq\lceil\frac{k}{2}\rceil+1$. We also know that $j\geq 1$. Hence, $j+\lfloor\frac{k}{2}\rfloor+1\geq\lceil\frac{k}{2}\rceil+1$. Therefore, $d(t,x_j)$ definitely appears in the inner summation when $k\geq i\geq j+\lfloor\frac{k}{2}\rfloor+1$. This means that $d(t,x_j)$ appears $k-j-\lfloor\frac{k}{2}\rfloor=\lceil\frac{k}{2}\rceil-j$ many times in the right hand side of (\ref{eq:ws}). Moreover, note that the index $j$ in the right hand side of (1) ranges between 1 and  $k-\lfloor\frac{k}{2}\rfloor-1$. Hence (\ref{eq:ws}) holds. Let 
\[
A=\sum_{i=k-\lfloor\frac{k}{2}\rfloor}^k(k-i)d(t,x_i)+\sum_{i=1}^{k-\lfloor\frac{k}{2}\rfloor-1}(\lceil\frac{k}{2}\rceil-1)d(t,x_i).
\]
By decomposing $\sum_{i=1}^{k-1}(k-i)d(t,x_i)$ to three summations, noting that $(k-k)d(t,x_k)=0$, and using (\ref{eq:ws}), we have 

\begin{align*}
\sum_{i=1}^{k-1}(k-i)d(t,x_i)&=\sum_{i=k-\lfloor\frac{k}{2}\rfloor}^k(k-i)d(t,x_i) \\ & +\sum_{i=1}^{k-\lfloor\frac{k}{2}\rfloor-1}(\lceil\frac{k}{2}\rceil-1)d(t,x_i) \\ & +\sum_{j=1}^{k-\lfloor\frac{k}{2}\rfloor-1}(k-j-\lceil\frac{k}{2}\rceil+1)d(t,x_j)\\&
=A+\sum_{j=1}^{k-\lfloor\frac{k}{2}\rfloor-1}(\lfloor\frac{k}{2}\rfloor-j+1)d(t,x_j)\\&
\geq A+\sum_{j=1}^{k-\lfloor\frac{k}{2}\rfloor-1}(\lceil\frac{k}{2}\rceil-j)d(t,x_j)\\&
=A+\sum_{i=\lceil\frac{k}{2}\rceil+1}^k\sum_{j=1}^{i-\lfloor\frac{k}{2}\rfloor-1}d(t,x_j).
\end{align*}

Therefore, by the triangle inequality and the above statements, we have 

\begin{align*}
\sum_{i=1}^{k-1}&(k-i)d(t,x_i)+\sum_{i=\lceil\frac{k}{2}\rceil+1}^k\sum_{j=1}^{i-\lfloor\frac{k}{2}\rfloor-1}d(x_i,x_j)\\&\geq A+\sum_{i=\lceil\frac{k}{2}\rceil+1}^k\sum_{j=1}^{i-\lfloor\frac{k}{2}\rfloor-1}d(t,x_j) \\ & +\sum_{i=\lceil\frac{k}{2}\rceil+1}^k\sum_{j=1}^{i-\lfloor\frac{k}{2}\rfloor-1}d(x_i,x_j)\\&=A+\sum_{i=\lceil\frac{k}{2}\rceil+1}^k\sum_{j=1}^{i-\lfloor\frac{k}{2}\rfloor-1}(d(t,x_j)+d(x_i,x_j))\\&\geq A+\sum_{i=\lceil\frac{k}{2}\rceil+1}^k\sum_{j=1}^{i-\lfloor\frac{k}{2}\rfloor-1}d(t,x_i)\\&
=A+\sum_{i=\lceil\frac{k}{2}\rceil+1}^k(i-\lfloor\frac{k}{2}\rfloor-1)d(t,x_i)\\&
\geq A+\sum_{i=\lceil\frac{k}{2}\rceil+1}^k(i-\lfloor\frac{k}{2}\rfloor-1)d(t,x_i) \\ & +(\lceil\frac{k}{2}\rceil-\lfloor\frac{k}{2}\rfloor-1)d(t,x_{\lceil\frac{k}{2}\rceil})
\\&=A+\sum_{i=\lceil\frac{k}{2}\rceil}^k(i-\lfloor\frac{k}{2}\rfloor-1)d(t,x_i)\\&
=\sum_{i=k-\lfloor\frac{k}{2}\rfloor}^k(k-i)d(t,x_i)+\sum_{i=1}^{k-\lfloor\frac{k}{2}\rfloor-1}(\lceil\frac{k}{2}\rceil-1)d(t,x_i) \\ & +\sum_{i=k-\lfloor\frac{k}{2}\rfloor}^k(i-\lfloor\frac{k}{2}\rfloor-1)d(t,x_i)\\&=\sum_{i=k-\lfloor\frac{k}{2}\rfloor}^k(k-i+i-\lfloor\frac{k}{2}\rfloor-1)d(t,x_i) \\ & +\sum_{i=1}^{k-\lfloor\frac{k}{2}\rfloor-1}(\lceil\frac{k}{2}\rceil-1)d(t,x_i)\\&=\sum_{i=k-\lfloor\frac{k}{2}\rfloor}^k(\lceil\frac{k}{2}\rceil-1)d(t,x_i)+\sum_{i=1}^{k-\lfloor\frac{k}{2}\rfloor-1}(\lceil\frac{k}{2} \rceil-1)d(t,x_i) \\ & =\sum_{i=1}^{k}(\lceil \frac{k}{2}\rceil-1)d(t,x_i).
\end{align*}
This yields the result.
\end{proofof}

We now proceed to bound the diversity part of the optimal solution (Lemma~\ref{boundD}). We re-use the key ideas from \citet{AghamolaeiFZ15} to achieve this. 
Let $O$ be an optimal solution for maximizing $f(S)$ subject to $S \subseteq U$ and $|S| = k$. Let $O^i = T^i \cap O$, $Q^i = O^i \setminus S^i$. So $Q^i$ are the elements of $O$ on machine $I$ that were ``missed'' by $S^i$. Intuitively, we
 bound the damage to optimality by missing these elements by 
 finding a low-weight matching between $Q^i$ and $S^i$. The following normalization parameters are used in the next two lemmas:
$r_i = \frac{f(S^i)}{{k \choose 2}}$ and $r = \max_{i=1,\ldots,m} r_i$.  Let $G^i(O^i\cup S^i,E)$ be a complete weighted graph. For $u, v \in O^i\cup S^i$, we use $d(u,v)$ as the edge weight  in our matching problem.

\begin{lemma}
There exists a  bipartite matching between $Q^i$ and $S^i$ in $G^i$ with a weight of at most $\frac{4.5}{2} |Q^i| r$ that covers all the $Q^i$.
\label{matching}
\end{lemma}
\begin{proof}
The number of all maximal bipartite matchings between $Q^i$ and $S^i$ is $\frac{k!}{(k-|Q^i|)!}$. Any of these matchings covers $Q^i$ because $|Q^i|\leq |S^i|$. Each edge $\{q,x\}$ with $q\in Q^i$ and $x\in S^i$ is in $\frac{(k-1)!}{(k-|Q^i|)!}$ of these matchings. 
Hence the total weight of all matchings can be expressed as
\begin{align*}
\frac{(k-1)!}{(k-|Q^i|)!}\sum_{q\in Q^i}\sum_{x\in S^i}d(q,x) & \leq \frac{(k-1)!}{(k-|Q^i|)!}\sum_{q\in Q^i}\frac{4.5}{k-1}f(S^i) \\ & \leq \frac{(k-1)!}{(k-|Q^i|)!}\sum_{q\in Q^i}\frac{4.5}{k-1}{k \choose 2}r \\ & = \frac{(k-1)!}{(k-|Q^i|)!}|Q^i|\frac{4.5k}{2}r \\ & = \frac{k!}{(k-|Q^i|)!}\frac{4.5}{2}|Q^i|r
\end{align*}
The first inequality is from Lemma~\ref{bnice} and the second by the definition of $r$.
It follows that there exists a matching with a weight of at most $\frac{4.5}{2}|Q^i|r$.
\end{proof}

We are now in position to upper bound the diversity portion of an optimal solution in terms of $f(\texttt{OPT}(\cup_i^m S^i))$.

\begin{proofof}{Lemma~\ref{boundD}.}
Let $M^i$ be the maximal bipartite matching between $Q^i$ and $S^i$ with a weight of less than or equal to $\frac{4.5}{2}|Q^i|r$. It exists because of Lemma~\ref{matching}. Let $M = \cup_{i=1}^m M^i$. Note that $S_i$'s are disjoint and $Q^i$'s are disjoint. This implies that $M^i$'s are disjoint. Therefore, $M$ is a matching between $\cup_{i=1}^m Q^i$ and $\cup_{i=1}^m S^i$ that covers all of $\cup_{i=1}^m Q^i$ with a weight of less than or equal to $\frac{4.5}{2}\sum_{i=1}^m |Q^i|r \leq \frac{4.5}{2}|O|r = \frac{4.5}{2}kr$.

Let $e:O\rightarrow \cup_{i=1}^m S^i$ be a mapping which maps any $o\in O\cap (\cup_{i=1}^m S^i)$ to itself and any $o\in (\cup_{i=1}^m Q^i)$ to its matched vertex in $M$. The weight of this mapping is less than or equal to the weight of $M$ since $d(o,o)=0$. Note that each vertex in the $range(e)$ is mapped from at most two vertices in $O$. We use this fact in the second inequality below
and  use the triangle inequality in the first inequality. We have
\begin{align*}
D & (O) = \sum_{\{u,v\}\in O} d(u,v) &  \\ & \leq \sum_{\{u,v\}\in O} (d(u, e(u)) + d(e(u),e(v)) + d(e(v), v)) \\ & = (|O| - 1)\sum_{u\in O}d(o,e(o))+\sum_{\{u,v\}\in O} d(e(u),e(v)) \\ & \leq (k-1)\frac{4.5}{2}kr + 4D(range(e)) \\ & \leq 4.5{k \choose 2}r + 4f(\texttt{OPT}(\cup_{i=1}^m S^i)) \\ & \leq 8.5f(\texttt{OPT}(\cup_{i=1}^m S^i))
\end{align*}
\end{proofof}

Now, we proceed to bound $g(O)$ and the proofs of the next two lemmas follow those found in \citet{MirrokniZ15}. Let $o_1,\ldots,o_k$ be an ordering of elements of $O$. For $x = o_i\in O$ define $O_x = \{o_1,\ldots,o_{i-1} \}$ and $O_{o_1} = \emptyset$.

\begin{lemma}
$g(O) \leq 6f(\texttt{OPT}(\cup_{i=1}^m S^i)) + \sum_{i=1}^m \sum_{x \in O\cap T^i \setminus S^i} (\Delta(x,O_x)-\Delta(x,O_x\cup S^i))$.
\label{boundG}
\end{lemma}
\begin{proof}
Note that $g(O) = g(O\cap (\cup_{i=1}^m S^i)) + \sum_{x\in O \setminus (\cup_{i=1}^m S^i)} \Delta(x,O_x\cup (O \cap (\cup_{i=1}^m S^i)))$. Therefore, using submodularity and monotonicity of $g$ and 5-niceness of Algorithm~\ref{alg:AlgGMM}, we have
{\begin{scriptsize}
\begin{align*}
& g(O) \leq f(\texttt{OPT}(\cup_{i=1}^m S^i)) + \sum_{x\in O \setminus (\cup_{i=1}^m S^i)}\Delta(x,O_x) \\ & = f(\texttt{OPT}(\cup_{i=1}^m S^i)) \\ & + \sum_{i=1}^m\sum_{x\in O \cap T^i \setminus S^i}(\Delta(x,O_x\cup S^i)+\Delta(x,O_x)-\Delta(x,O_x\cup S^i)) \\ & \leq f(\texttt{OPT}(\cup_{i=1}^m S^i)) \\ & + \sum_{i=1}^m\sum_{x\in O \cap T^i \setminus S^i}(\Delta(x,S^i)+\Delta(x,O_x)-\Delta(x,O_x\cup S^i)) \\ & \leq f(\texttt{OPT}(\cup_{i=1}^m S^i)) \\ & + \sum_{i=1}^m\sum_{x\in O \cap T^i \setminus S^i}(\frac{5}{k}f(S^i)+\Delta(x,O_x)-\Delta(x,O_x\cup S^i)) \\ & \leq f(\texttt{OPT}(\cup_{i=1}^m S^i)) \\ & + \sum_{i=1}^m\sum_{x\in O \cap T^i \setminus S^i}(\frac{5}{k}f(\texttt{OPT}(\cup_{i=1}^m S^i))+\Delta(x,O_x)-\Delta(x,O_x\cup S^i)) \\ & \leq f(\texttt{OPT}(\cup_{i=1}^m S^i)) + 5f(\texttt{OPT}(\cup_{i=1}^m S^i)) \\ & + \sum_{i=1}^m\sum_{x\in O \cap T^i \setminus S^i}(\Delta(x,O_x)-\Delta(x,O_x\cup S^i)) \\ & \leq 6f(\texttt{OPT}(\cup_{i=1}^m S^i)) + \sum_{i=1}^m\sum_{x\in O \cap T^i \setminus S^i}(\Delta(x,O_x)-\Delta(x,O_x\cup S^i))
\end{align*}
\end{scriptsize}}
\end{proof}

In the next Lemma, we use the randomness of the partitioning of the data over machines and the first property of $\beta$-niceness.

\begin{lemma}
$\mathbb{E}[\sum_{i=1}^m \sum_{x \in O\cap T^i \setminus S^i} (\Delta(x,O_x)-\Delta(x,O_x\cup S^i))] \leq \mathbb{E}[f(\texttt{OPT}(\cup_{i=1}^m S^i))]$.
\label{bounddDif}
\end{lemma}
\begin{proof}
We show that $\mathbb{E}[\sum_{i=1}^m \sum_{x \in O\cap T^i \setminus S^i} (\Delta(x,O_x)-\Delta(x,O_x\cup S^i))] \leq \frac{\mathbb{E}[\sum_{i=1}^m g(S^i)]}{m}$ and the statement of the lemma follows from the fact that $\frac{\sum_{i=1}^m g(S^i)}{m} \leq f(\texttt{OPT}(\cup_{i=1}^m S^i))$. We first establish an inequality
\[
A:=\mathbb{E}[\sum_{i=1}^m \sum_{x \in O\cap T^i \setminus S^i} (\Delta(x,O_x)-\Delta(x,O_x\cup S^i))] \leq \frac{1}{m} B
\]

\noindent 
where 
\[
B:= \mathbb{E}[\sum_{i=1}^m \sum_{x\in O} (\Delta(x,O_x)-\Delta(x,O_x\cup S^i))].
\]

Let $\texttt{ALG}$ be Algorithm~\ref{alg:AlgGMM}. For $T\subseteq U$ and $x\in U$, let $q(x,T) = \Delta(x,O_x) - \Delta(x,O_x\cup \texttt{ALG}(T))$. Let $P[.]$ be the probability mass function for the uniform distribution  over $m$-partitions $\mathbbm{P}=(T^1, \ldots ,T^m)$ of $U$, and let $\mathbbm{1}[x\notin \texttt{ALG}(T\cup\{x\})]$ be a $0,1$ indicator function. Note that
\begin{align*}
P[T^i=T] & = (\frac{1}{m})^{|T|} (1-\frac{1}{m})^{|U|-|T|} \\
P[T^i=T \cup \{x\}] & = (\frac{1}{m})^{|T|+1} (1-\frac{1}{m})^{|U|-|T|-1}
\end{align*}
Therefore
\begin{equation}
\label{eqn:AB}
P[T^i=T \cup \{x\}] = \frac{P[T^i=T]+P[T^i=T \cup \{x\}]}{m}.
\end{equation}

We have that
{\begin{tiny}
\begin{align*}
& A = \sum_{i=1}^m \sum_{x\in O} \sum_{T \subseteq U\setminus \{x\}} P[T^i=T\cup\{x\}]\mathbbm{1}[x\notin \texttt{ALG}(T\cup\{x\})]q(x,T\cup\{x\}) \\
& B = \sum_{i=1}^m \sum_{x\in O} \sum_{T \subseteq U\setminus \{x\}} (P[T^i=T\cup\{x\}]q(x,T\cup\{x\})+P[T^i=T]q(x,T)) \\
 & \geq \sum_{i=1}^m \sum_{x\in O} \sum_{T \subseteq U\setminus \{x\}} \mathbbm{1}[x\notin \texttt{ALG}(T\cup\{x\})]q(x,T\cup\{x\})(P[T^i=T\cup\{x\}] \\ & \qquad\qquad\qquad\qquad\qquad\qquad\qquad\qquad\qquad\qquad\qquad\qquad +P[T^i=T]).
\end{align*}
\end{tiny}}
The last inequality holds because $q(.,.)$ is a non-negative function and multiplying it by $\mathbbm{1}[x\notin \texttt{ALG}(T\cup\{x\})]$ can only decrease the sum value. Also, $q(x,T)$ is replaced by $q(x,T\cup\{x\})$. It does not change the sum value because when $\mathbbm{1}[x\notin \texttt{ALG}(T\cup\{x\})]=1$, $q(x,T)=q(x,T\cup\{x\})$. 

\noindent
We now deduce $A\leq B/m$  from (\ref{eqn:AB}). 

Now note that $\sum_{x\in O} \Delta(x,O_x\cup S^i) = g(O\cup S^i) - g(S^i)$, and $\sum_{x\in O} \Delta(x,O_x) = g(O)$. Therefore, because of the monotonicity of $g$, we have for any $i$
\begin{align*}
\sum_{x\in O} \Delta(x,O_x) - & \Delta(x,O_x\cup S^i) \\ & = g(O) - g(O\cup S^i) + g(S^i) \leq g(S^i).
\end{align*}
Hence $B\leq \frac{\mathbb{E}[\sum_{i=1}^m g(S^i)]}{m}$ and the lemma follows.
\end{proof}

We now have that Lemma~\ref{boundg} follows directly from
Lemmas~\ref{boundG}, and \ref{bounddDif} as they imply 
\[
g(O) \leq  6f(\texttt{OPT}(\cup_{i=1}^m S^i)) + \mathbb{E}[f(\texttt{OPT}(\cup_{i=1}^m S^i))].
\]

\noindent
Therefore this completes the proof of Theorem~\ref{thm:main}.

\section{Appendix C}
\label{evalMeasure}

\begin{figure*}[t]
\centering
\begin{subfigure}[t]{0.28\textwidth}
\centering
\includegraphics[width=1.05\linewidth]{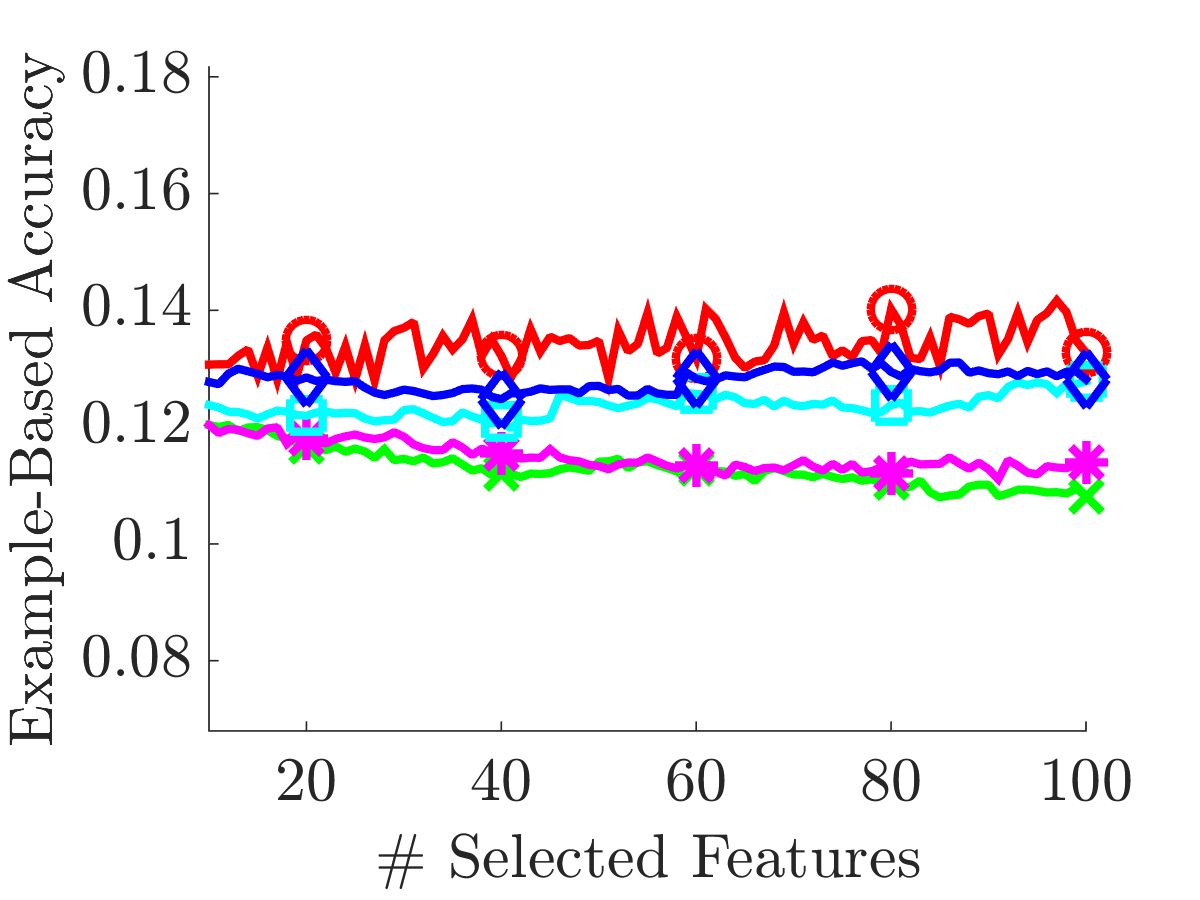}
\end{subfigure}
\begin{subfigure}[t]{0.28\textwidth}
\centering
\includegraphics[width=1.05\linewidth]{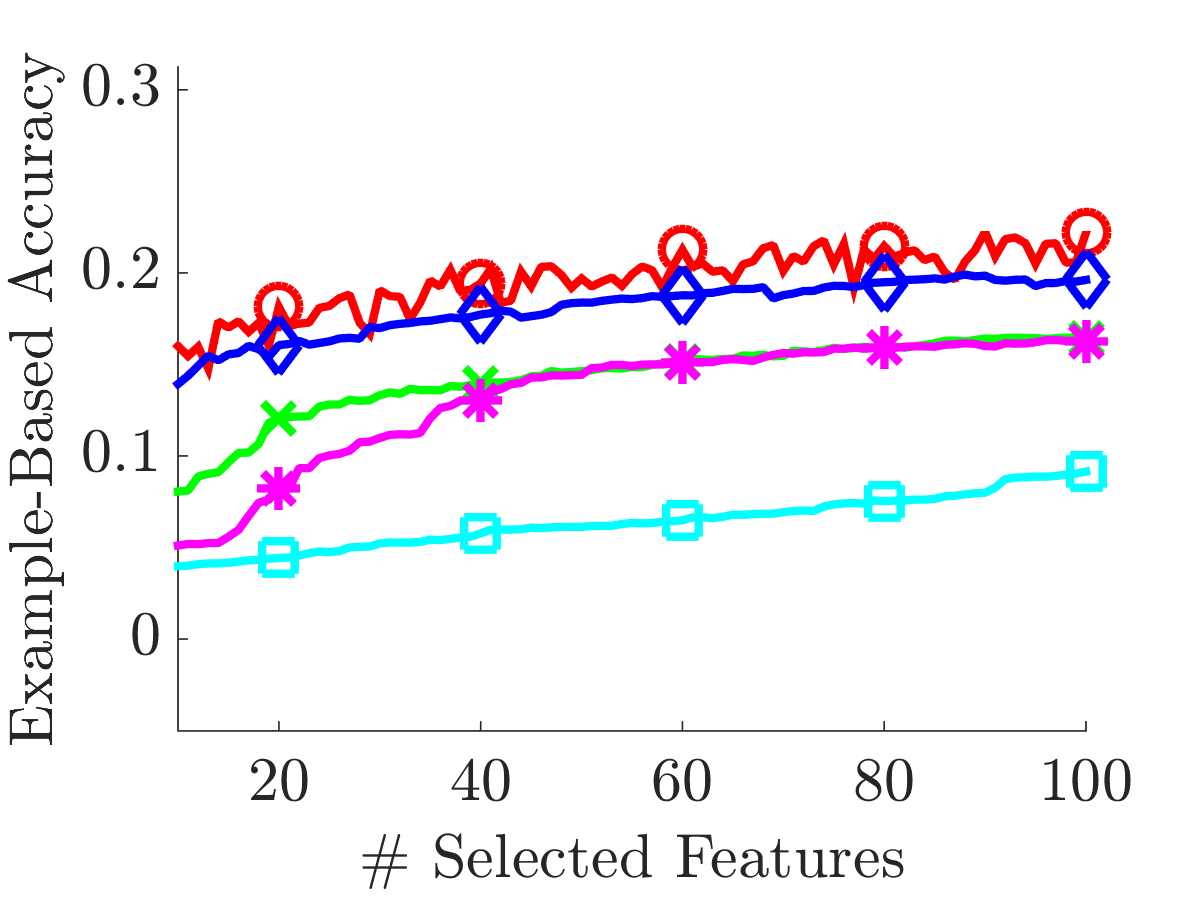}
\end{subfigure}
\begin{subfigure}[t]{0.28\textwidth}
\centering
\includegraphics[width=1.05\linewidth]{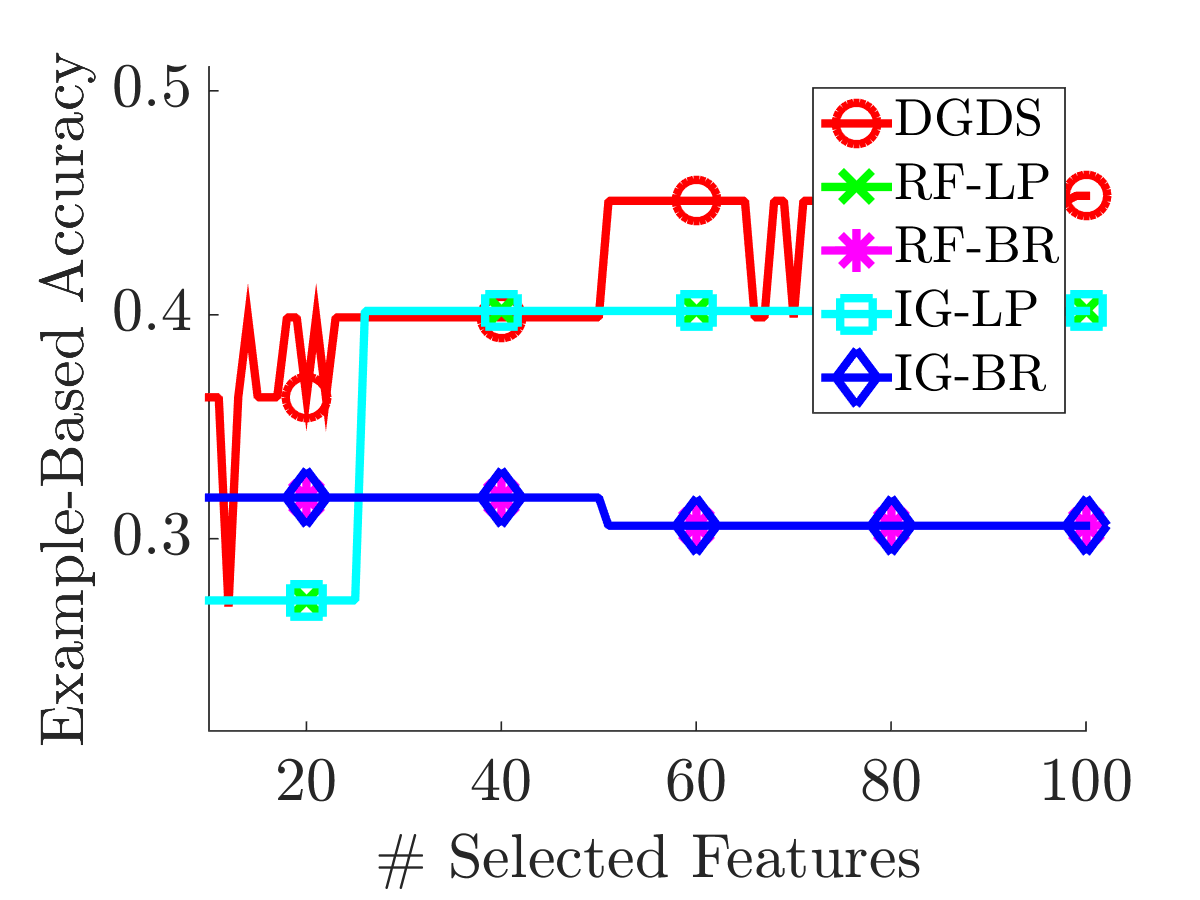}
\end{subfigure}
\vskip\baselineskip
\begin{subfigure}[t]{0.28\textwidth}
\centering
\includegraphics[width=1.05\linewidth]{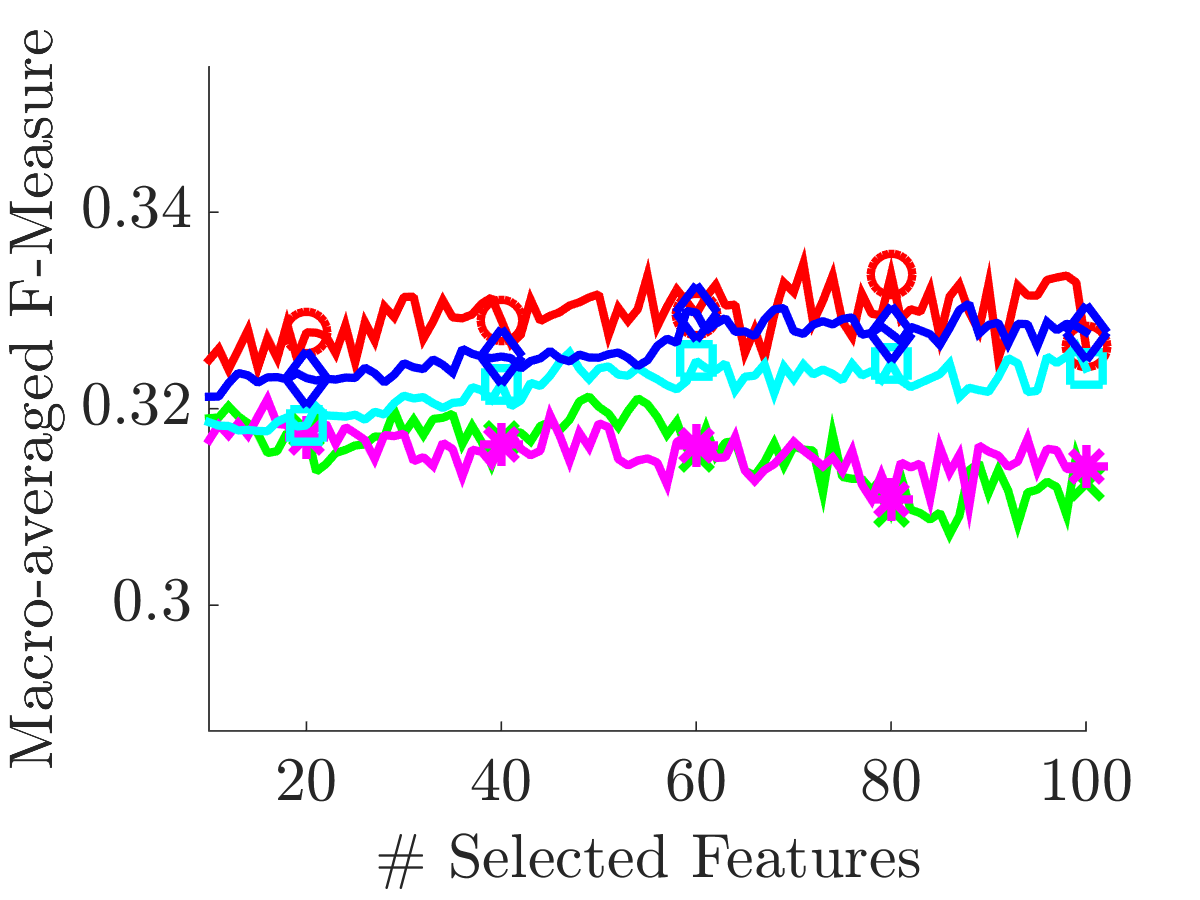}
\subcaption{Corel5k}
\end{subfigure}
\begin{subfigure}[t]{0.28\textwidth}
\centering
\includegraphics[width=1.05\linewidth]{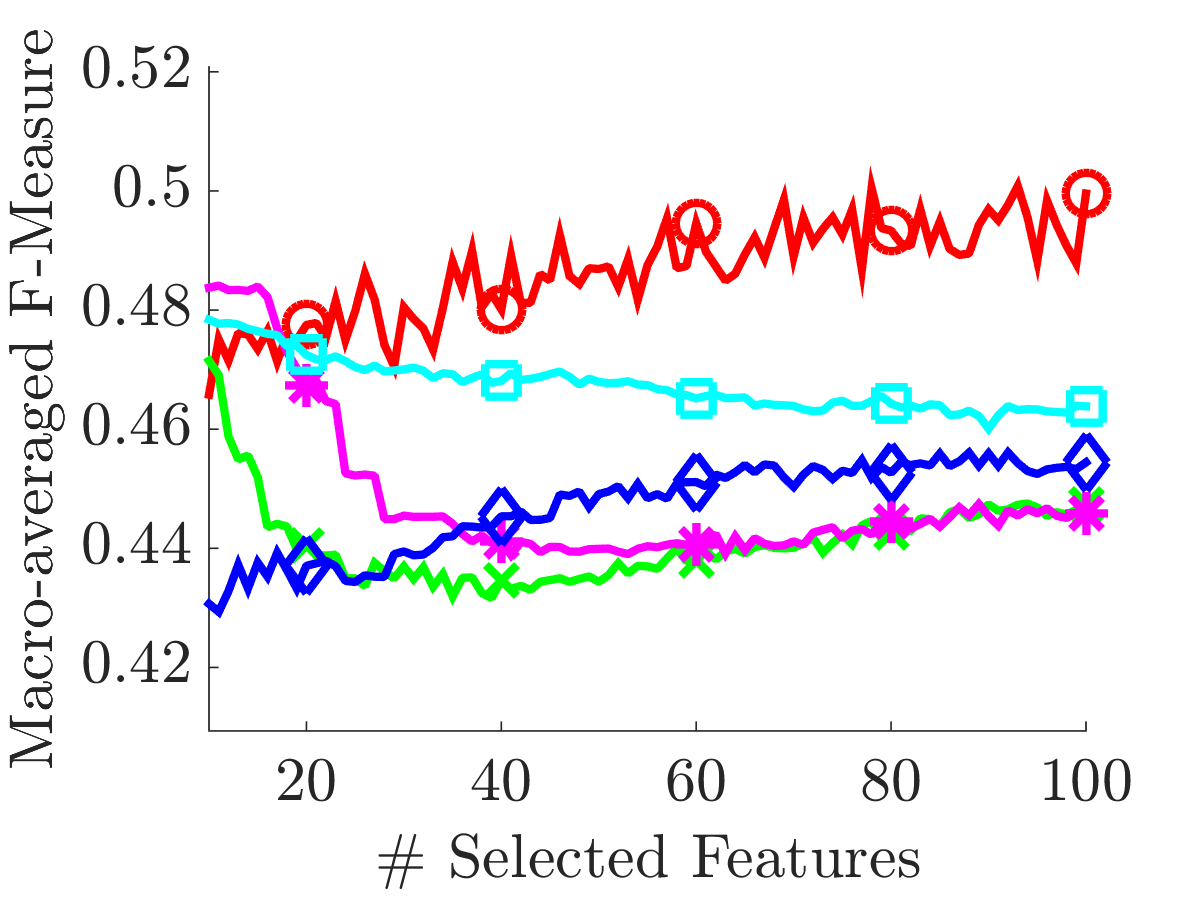}
\subcaption{Eurlex-ev}
\end{subfigure}
\begin{subfigure}[t]{0.28\textwidth}
\centering
\includegraphics[width=1.05\linewidth]{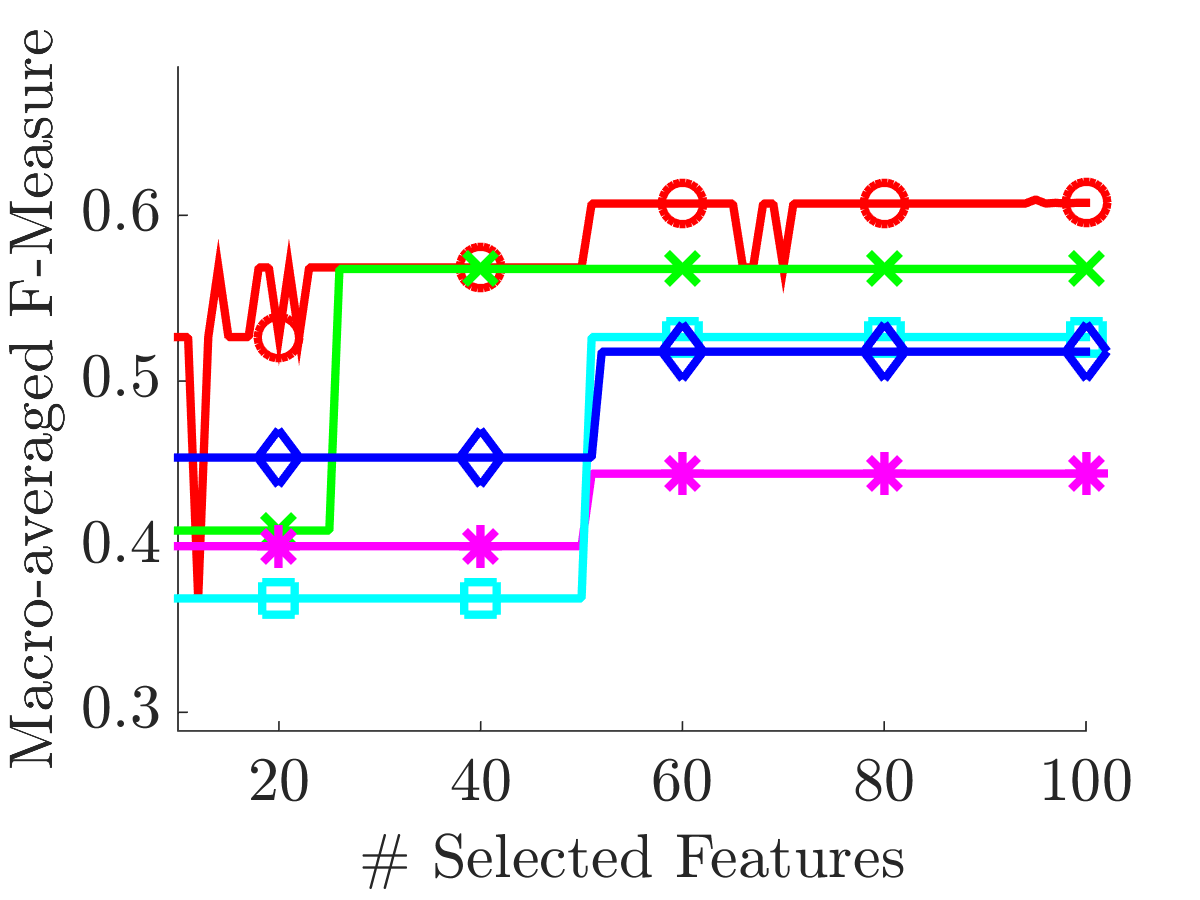}
\subcaption{Synthesized}
\end{subfigure}
\vskip\baselineskip
\caption{Comparison of proposed distributed method with centralized methods in the literature.}
\label{Fig:comparison2}
\end{figure*}

Let $n$ be the number of samples in the dataset, $L_i$ be the set of labels for sample $i$ that are $1$ in the dataset, and $L'_i$ be the set of labels for sample $i$ that we predicted to be $1$. Then the subset accuracy of our learning method is equal to
\[
\frac{1}{n}\sum_{i=1}^n \mathbb{I}(L_i, L'_i)
\]
where $\mathbb{I}(,.,)$ is a 0, 1 indicator function and is equal to 1 when set $L_i$ is equal to the set $L'_i$, and it is 0 otherwise. Example-based accuracy is equal to
\[
\frac{1}{n}\sum_{i=1}^n \frac{|L_i\cap L'_i|}{|L_i\cup L'_i|}.
\]
Example-based F-measure is equal to
\[
\frac{1}{n}\sum_{i=1}^n \frac{2|L_i\cap L'_i|}{|L_i|+|L'_i|}.
\]
These evaluation measures are example-based. Micro-averaged F-measure and Macro-averaged F-measure are two label-based measures for multi-label classification. Let $t$ be the number of labels in the dataset, $E_i$ be the set of examples that their $i$'th label is equal to 1, and $E'_i$ be the set of example that we predicted their $i$'th labels to be 1. Then Micro-averaged F-measure is equal to
\[
\frac{1}{t}\sum_{i=1}^{t} \frac{2|E_i\cap E'_i|}{|E_i|+|E'_i|}.
\]
Macro-averaged F-measure is equal to
\[
\frac{2\sum_{i=1}^{t}|E_i\cap E'_i|}{\sum_{i=1}^{t}|E_i|+\sum_{i=1}^{t}|E'_i|}.
\]
\section{Appendix D}
\label{otherDatasets}

Results of example-based accuracy
and macro-average F-measure comparison for Corel5k, Eurlex-ev, and Synthesized datasets are included in are shown in Figure~\ref{Fig:comparison2}. Specifications of three other datasets are shown in Table~\ref{table:dataset2} and the performance of our method on these datasets is compared to centralized methods in Figure~\ref{Fig:comparison3}.

\begin{table}[h]
    \centering
    \caption{Specifications of other datasets.}    
    \label{table:dataset2}
    \resizebox{\columnwidth}{!}{%
    \begin{tabular}{ccccc}
        \toprule Dataset Name & \# Features & \# Instances & \# Labels & Reference \\
        \midrule CAL500 & 68 & 502 & 174 & \citep{TurnbullBTL08} \\
        \midrule Delicious & 500 & 16,105 & 983 & \citep{tsoumakas2008effective} \\
        \midrule Scene & 294 & 2407 & 6 & \citep{BoutellLSB04} \\
        \bottomrule
        \end{tabular}}
\end{table}

\begin{figure*}[t]
\centering
\begin{subfigure}[t]{0.28\textwidth}
\centering
\includegraphics[width=1.05\linewidth]{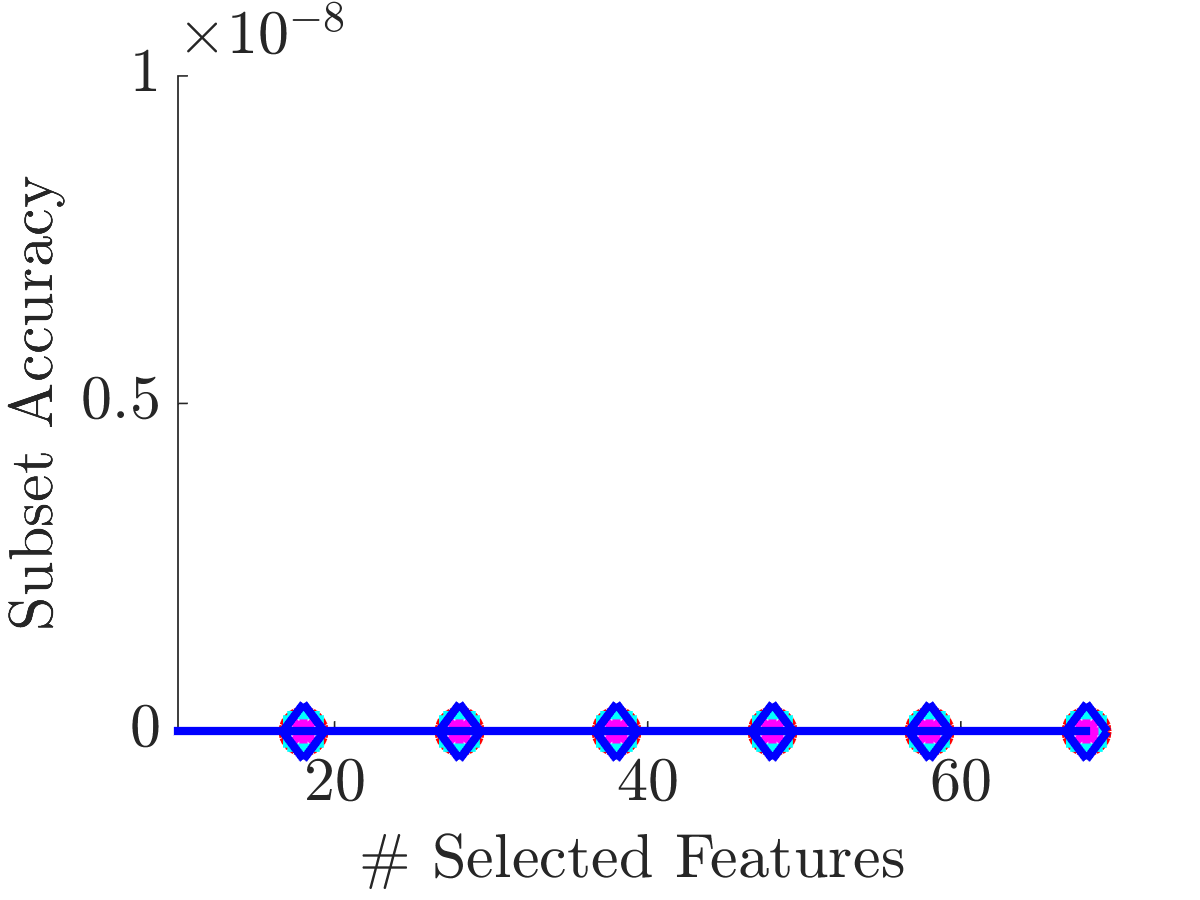}
\end{subfigure}
\begin{subfigure}[t]{0.28\textwidth}
\centering
\includegraphics[width=1.05\linewidth]{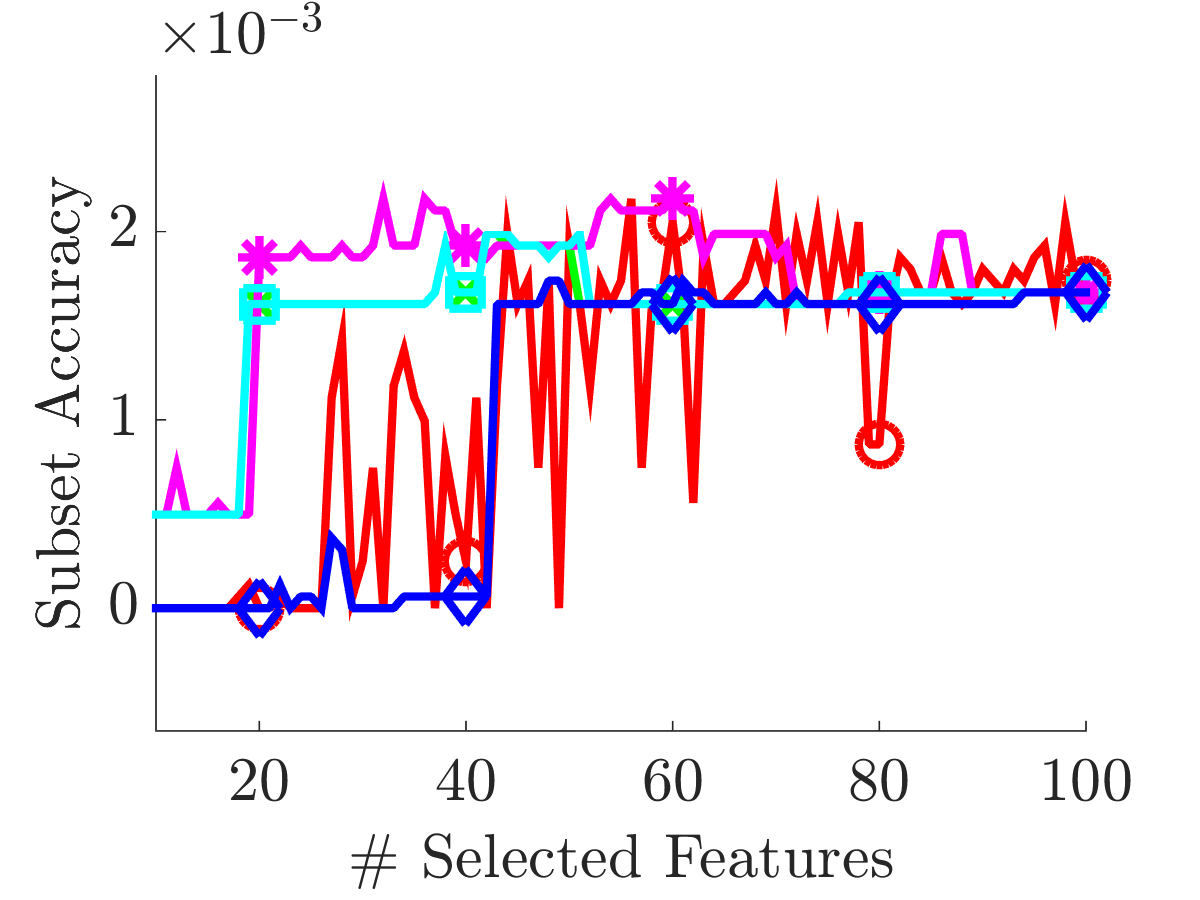}
\end{subfigure}
\begin{subfigure}[t]{0.28\textwidth}
\centering
\includegraphics[width=1.05\linewidth]{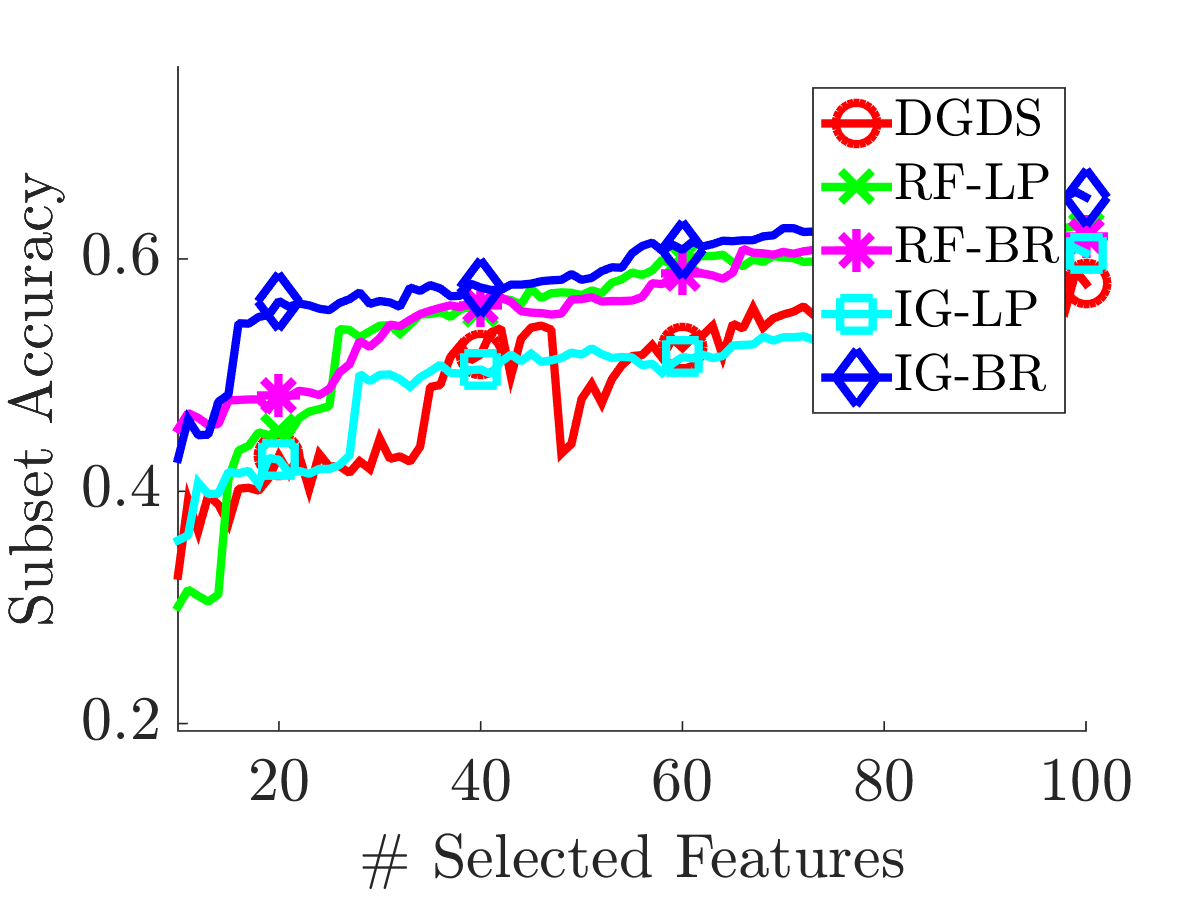}
\end{subfigure}
\vskip\baselineskip
\begin{subfigure}[t]{0.28\textwidth}
\centering
\includegraphics[width=1.05\linewidth]{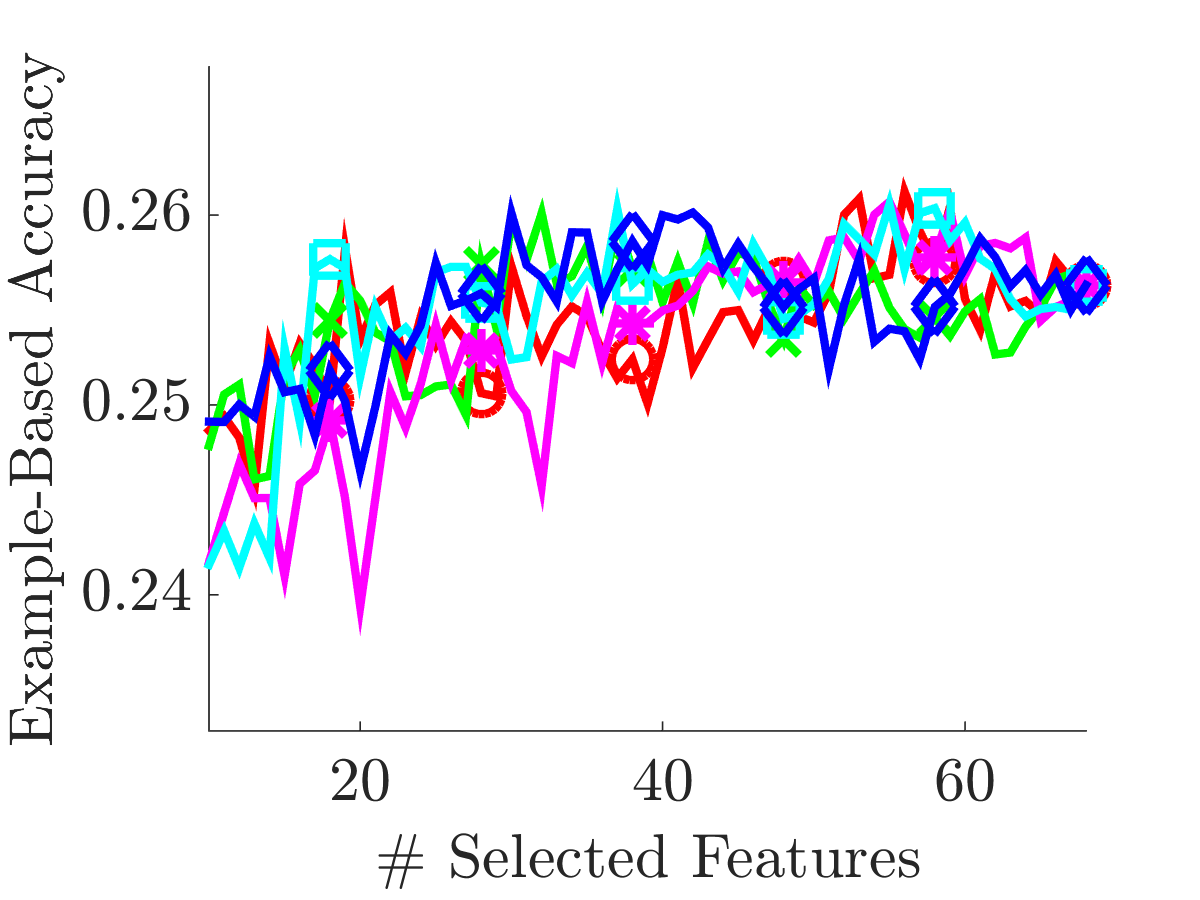}
\end{subfigure}
\begin{subfigure}[t]{0.28\textwidth}
\centering
\includegraphics[width=1.05\linewidth]{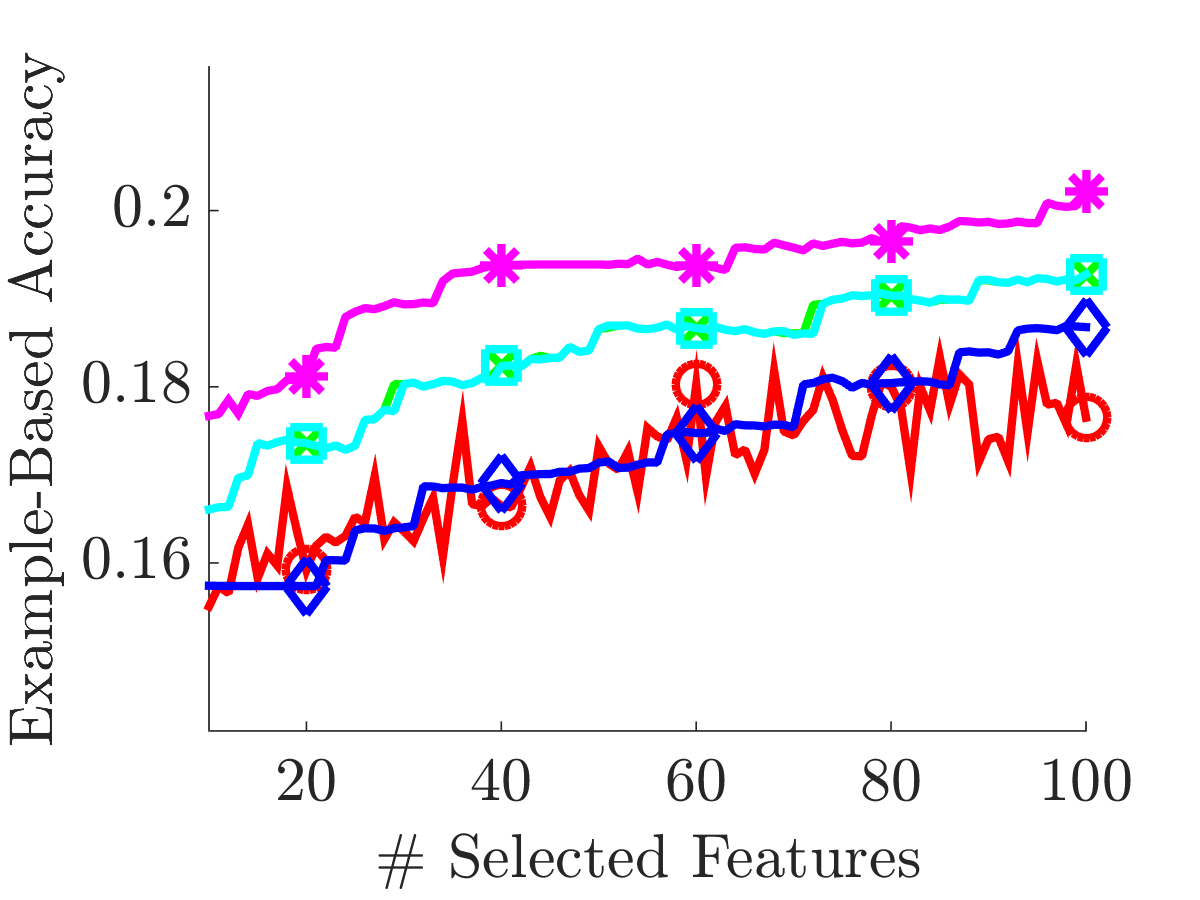}
\end{subfigure}
\begin{subfigure}[t]{0.28\textwidth}
\centering
\includegraphics[width=1.05\linewidth]{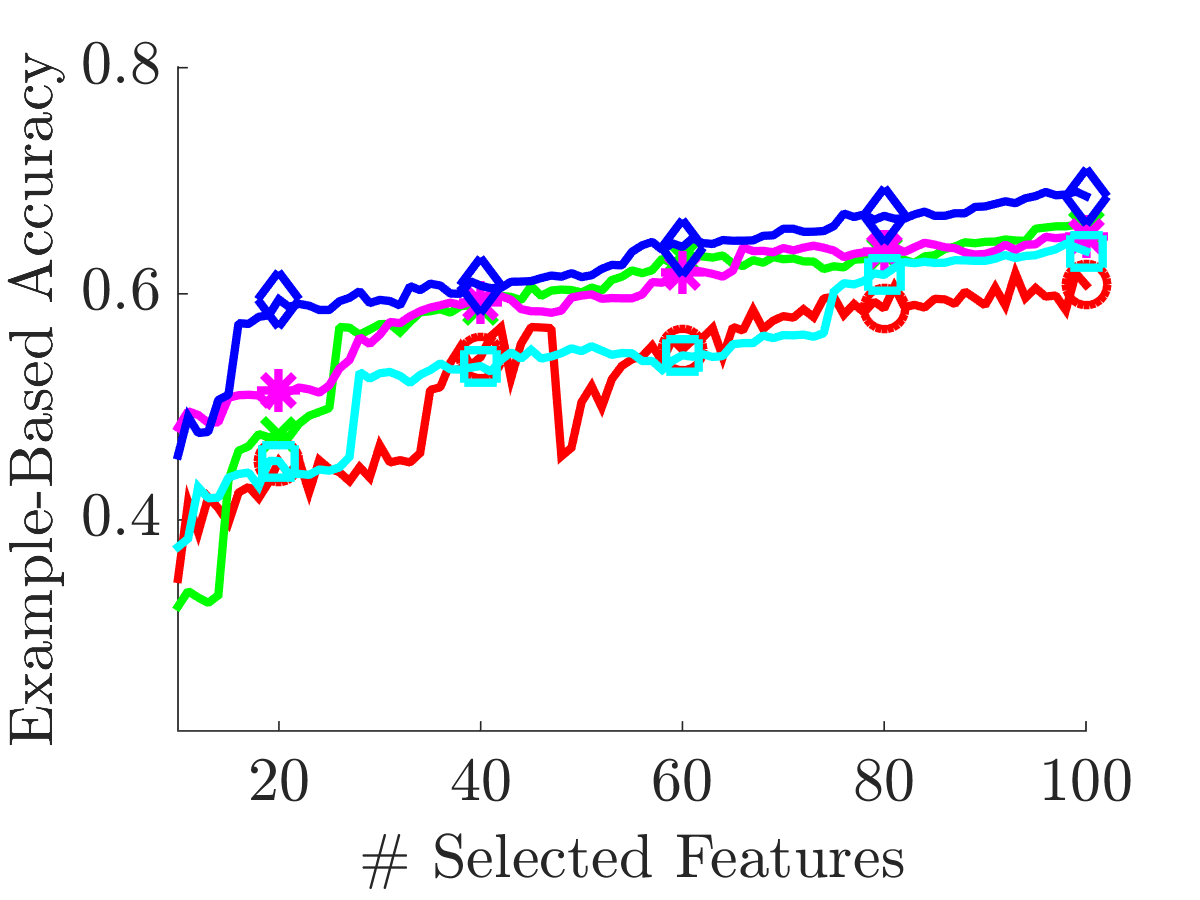}
\end{subfigure}
\vskip\baselineskip
\begin{subfigure}[t]{0.28\textwidth}
\centering
\includegraphics[width=1.05\linewidth]{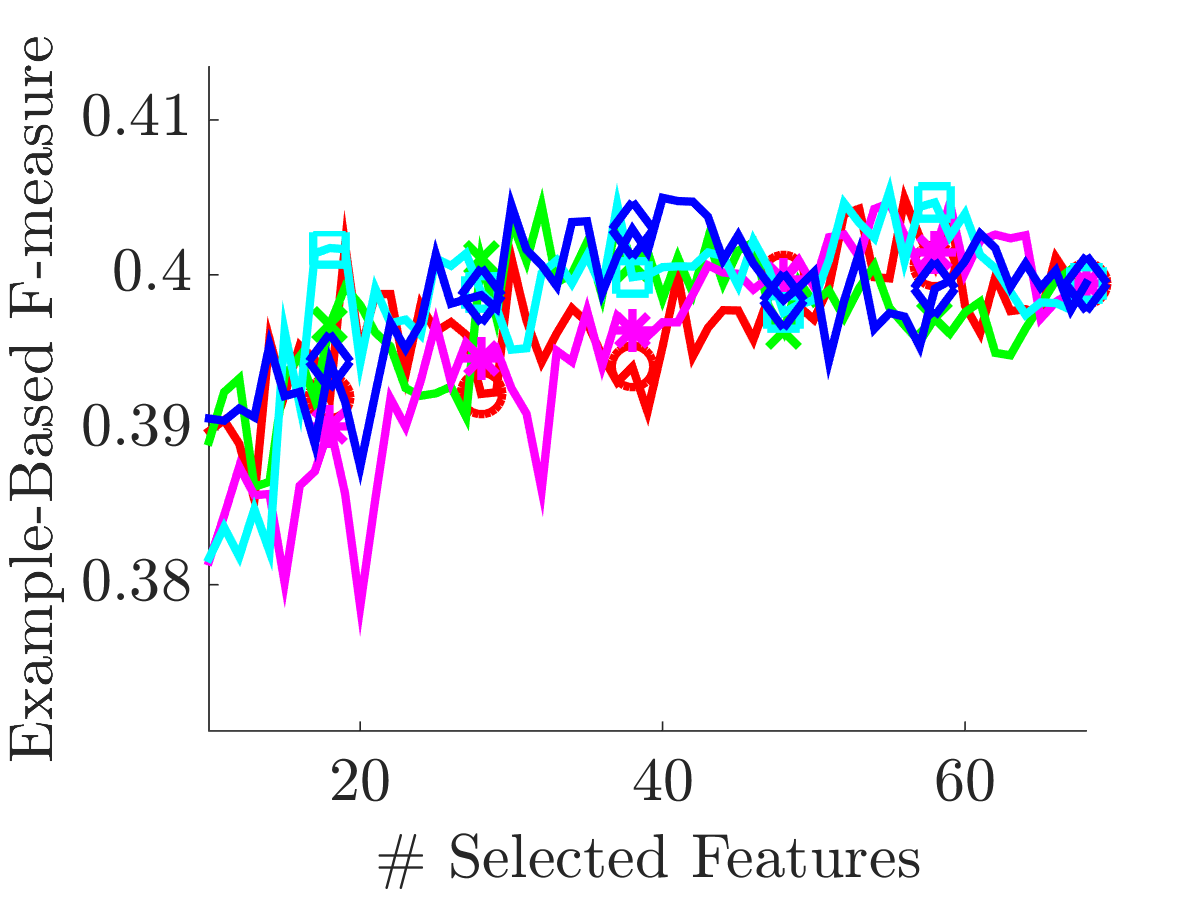}
\end{subfigure}
\begin{subfigure}[t]{0.28\textwidth}
\centering
\includegraphics[width=1.05\linewidth]{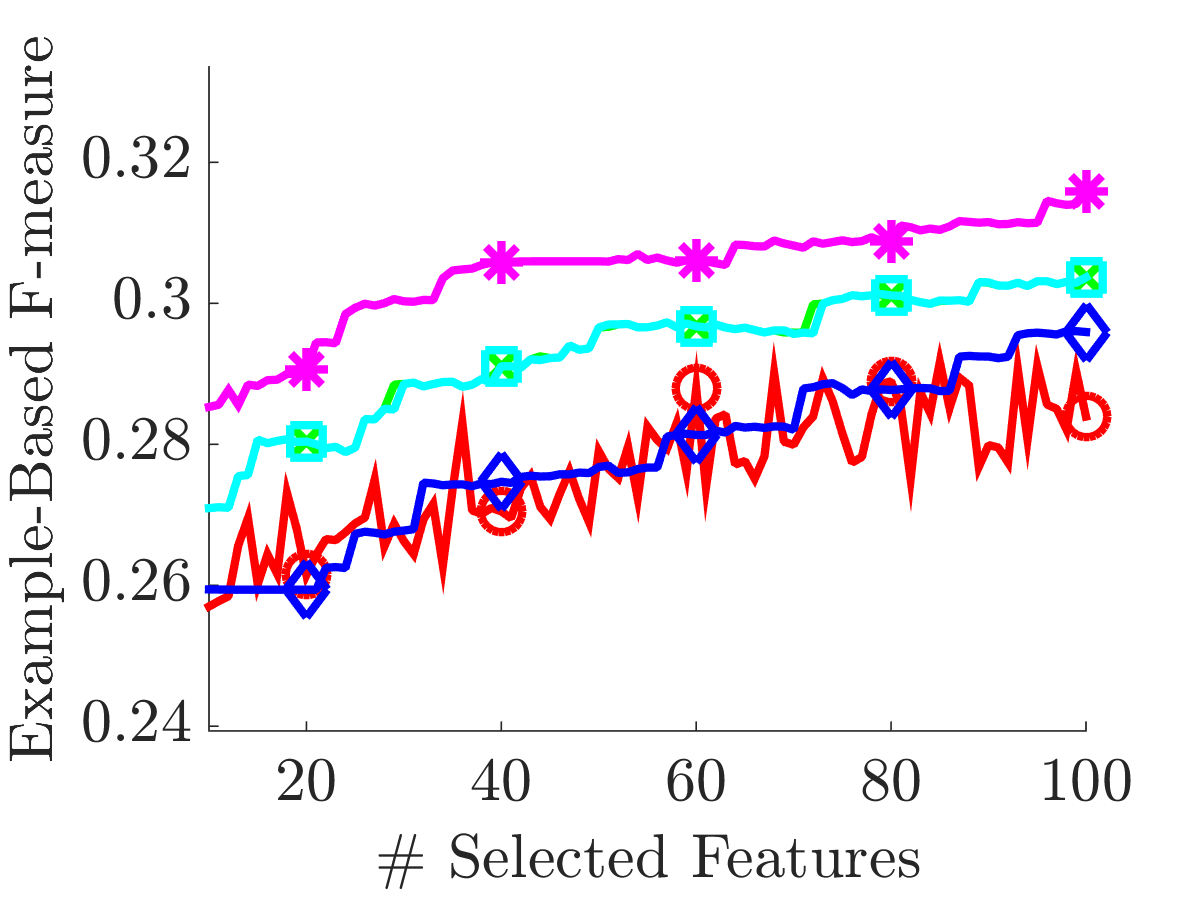}
\end{subfigure}
\begin{subfigure}[t]{0.28\textwidth}
\centering
\includegraphics[width=1.05\linewidth]{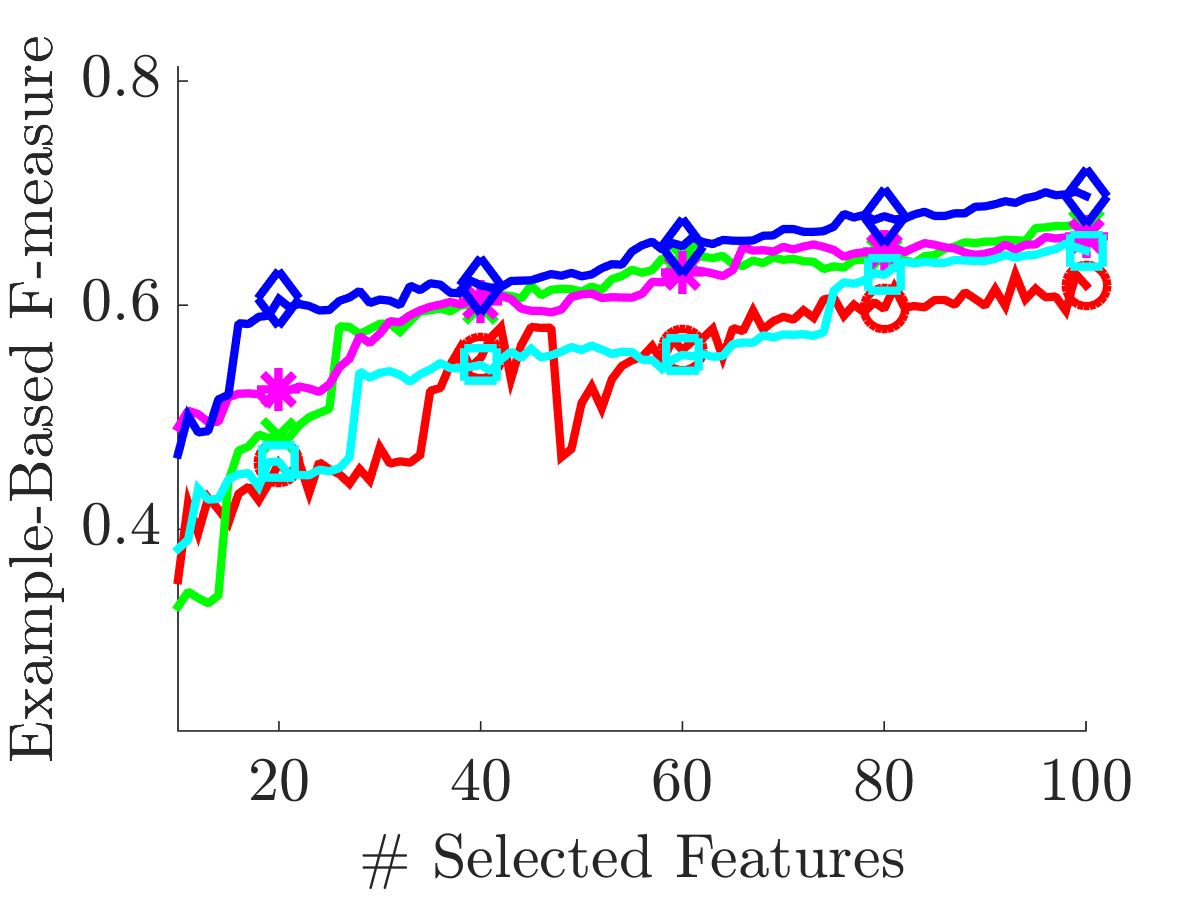}
\end{subfigure}
\vskip\baselineskip
\begin{subfigure}[t]{0.28\textwidth}
\centering
\includegraphics[width=1.05\linewidth]{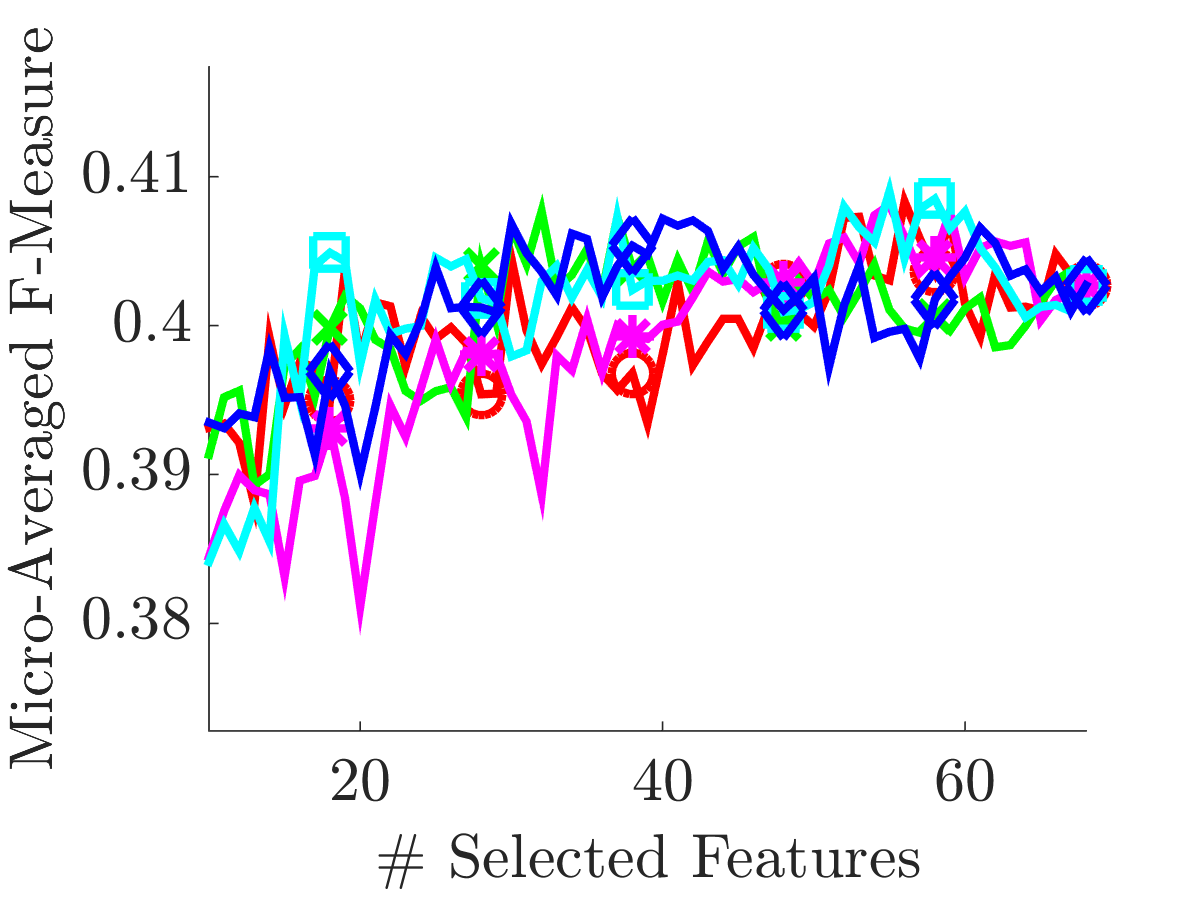}
\end{subfigure}
\begin{subfigure}[t]{0.28\textwidth}
\centering
\includegraphics[width=1.05\linewidth]{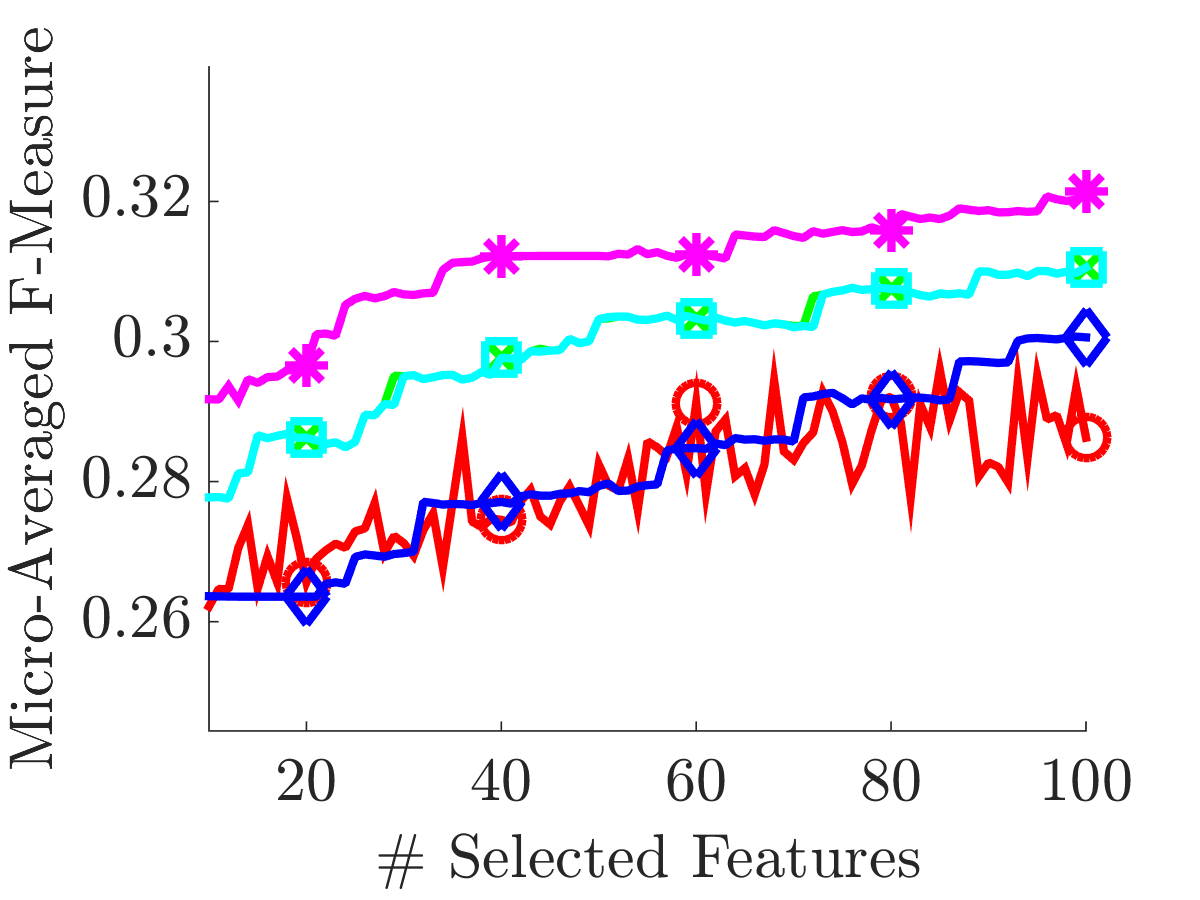}
\end{subfigure}
\begin{subfigure}[t]{0.28\textwidth}
\centering
\includegraphics[width=1.05\linewidth]{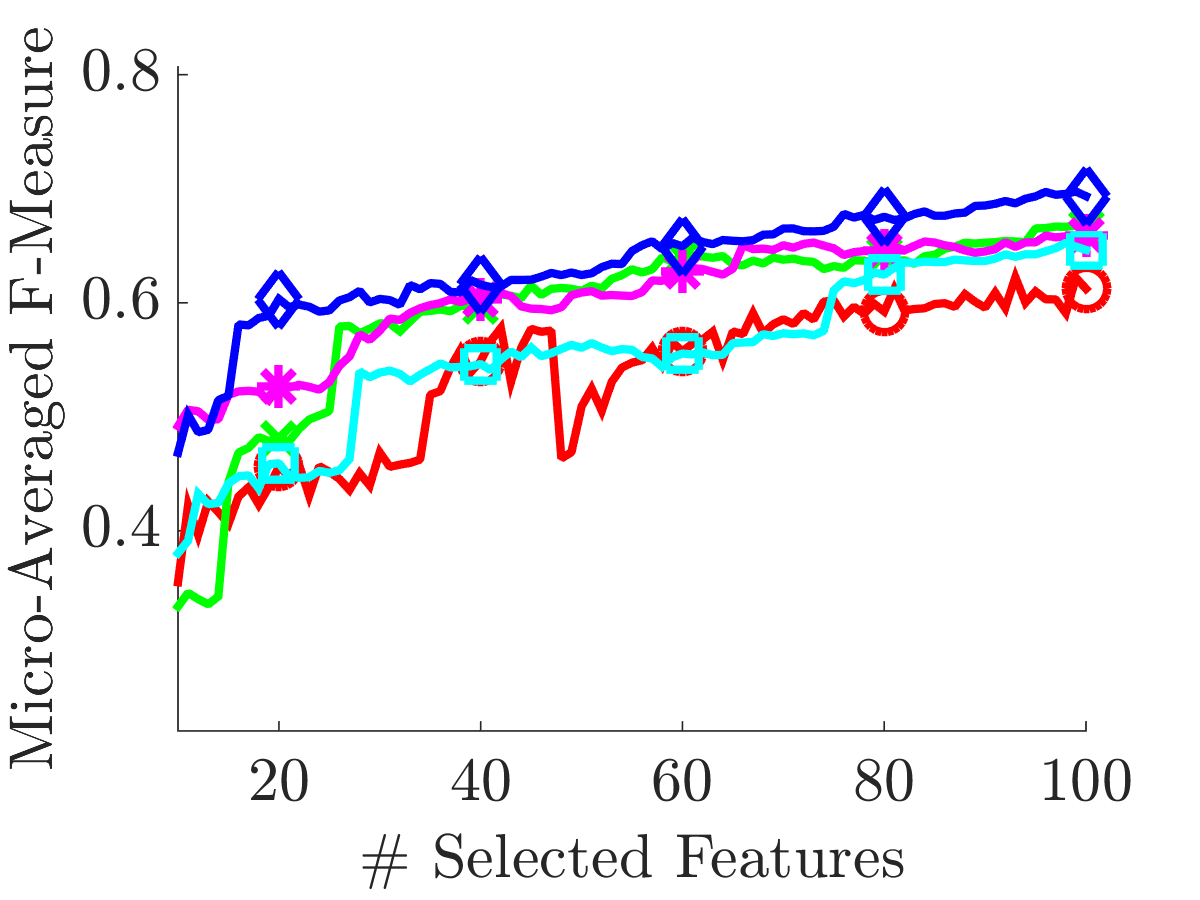}
\end{subfigure}
\vskip\baselineskip
\begin{subfigure}[t]{0.28\textwidth}
\centering
\includegraphics[width=1.05\linewidth]{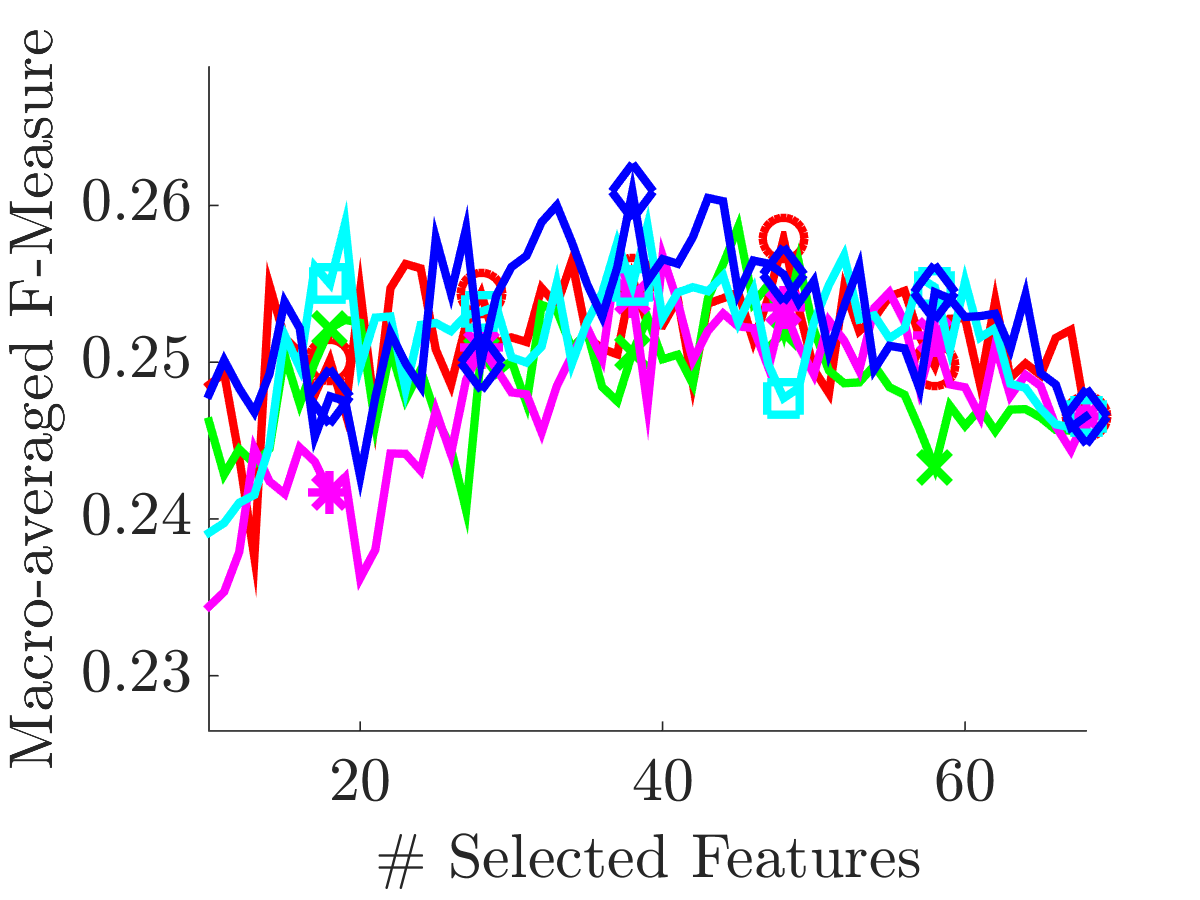}
\subcaption{CAL500}
\end{subfigure}
\begin{subfigure}[t]{0.28\textwidth}
\centering
\includegraphics[width=1.05\linewidth]{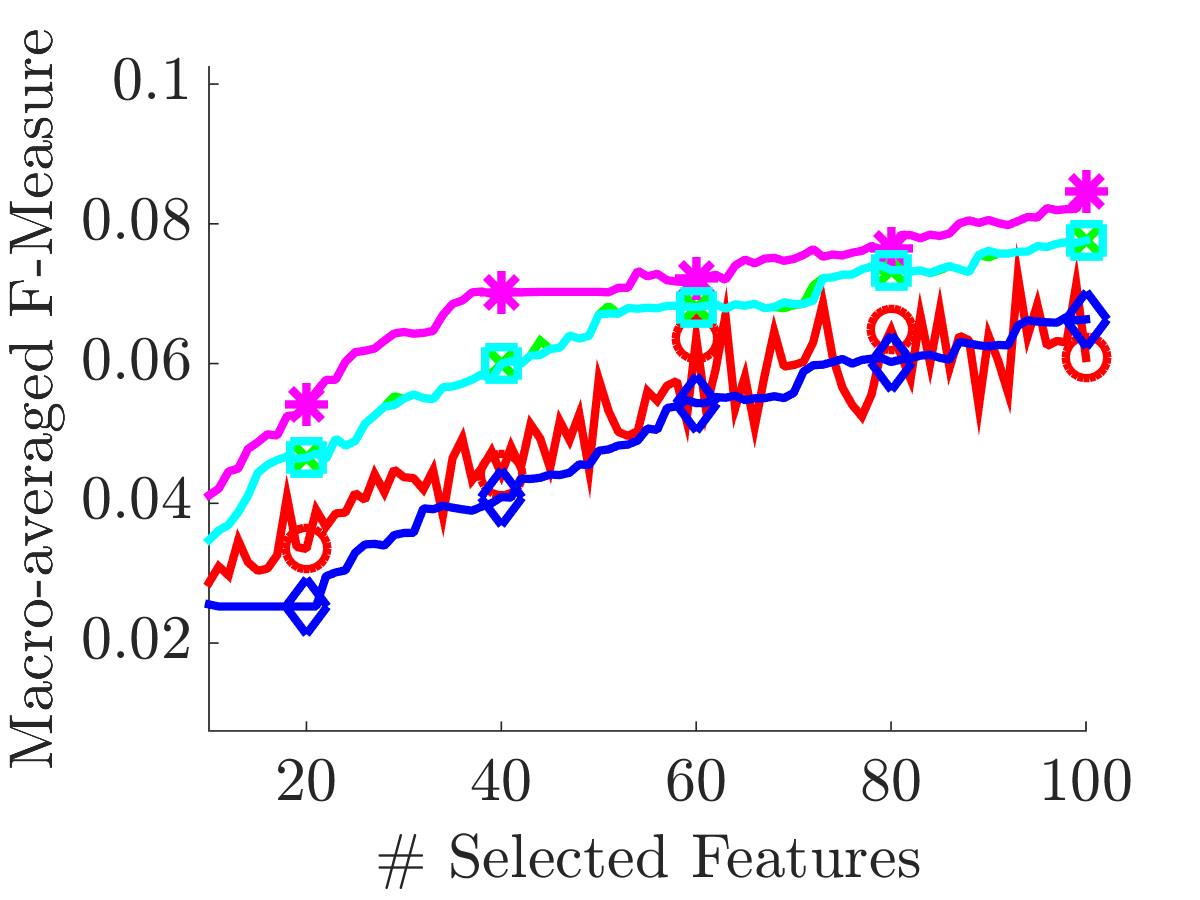}
\subcaption{Delicious}
\end{subfigure}
\begin{subfigure}[t]{0.28\textwidth}
\centering
\includegraphics[width=1.05\linewidth]{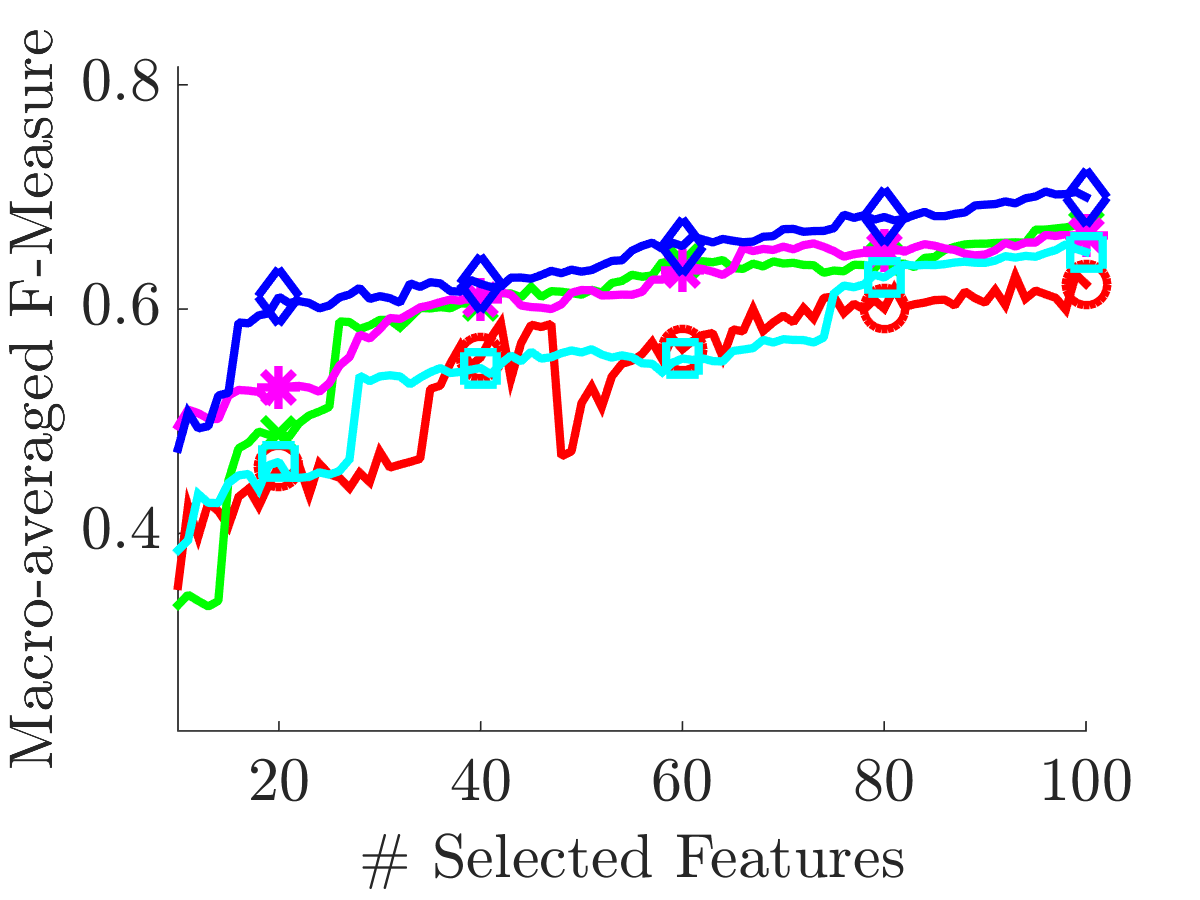}
\subcaption{Scene}
\end{subfigure}
\vskip\baselineskip
\caption{Comparison of proposed distributed method with centralized methods in the literature.}
\label{Fig:comparison3}
\end{figure*}
\section{Appendix E}
\label{centVSdist}
The performance of our distributed and centralized methods are compared in Figure~\ref{Fig:comparison4}.

\begin{figure*}[t]
\centering
\begin{subfigure}[t]{0.28\textwidth}
\centering
\includegraphics[width=1.05\linewidth]{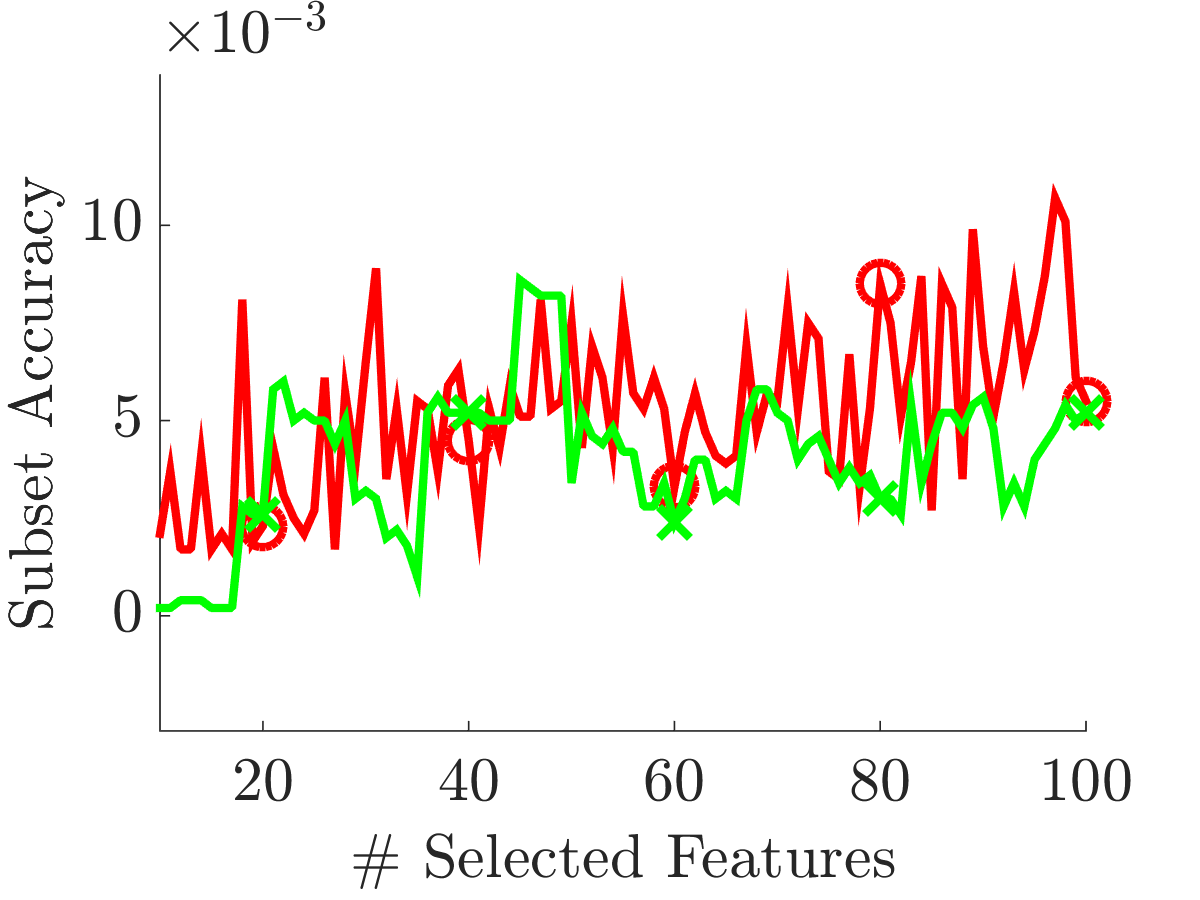}
\end{subfigure}
\begin{subfigure}[t]{0.28\textwidth}
\centering
\includegraphics[width=1.05\linewidth]{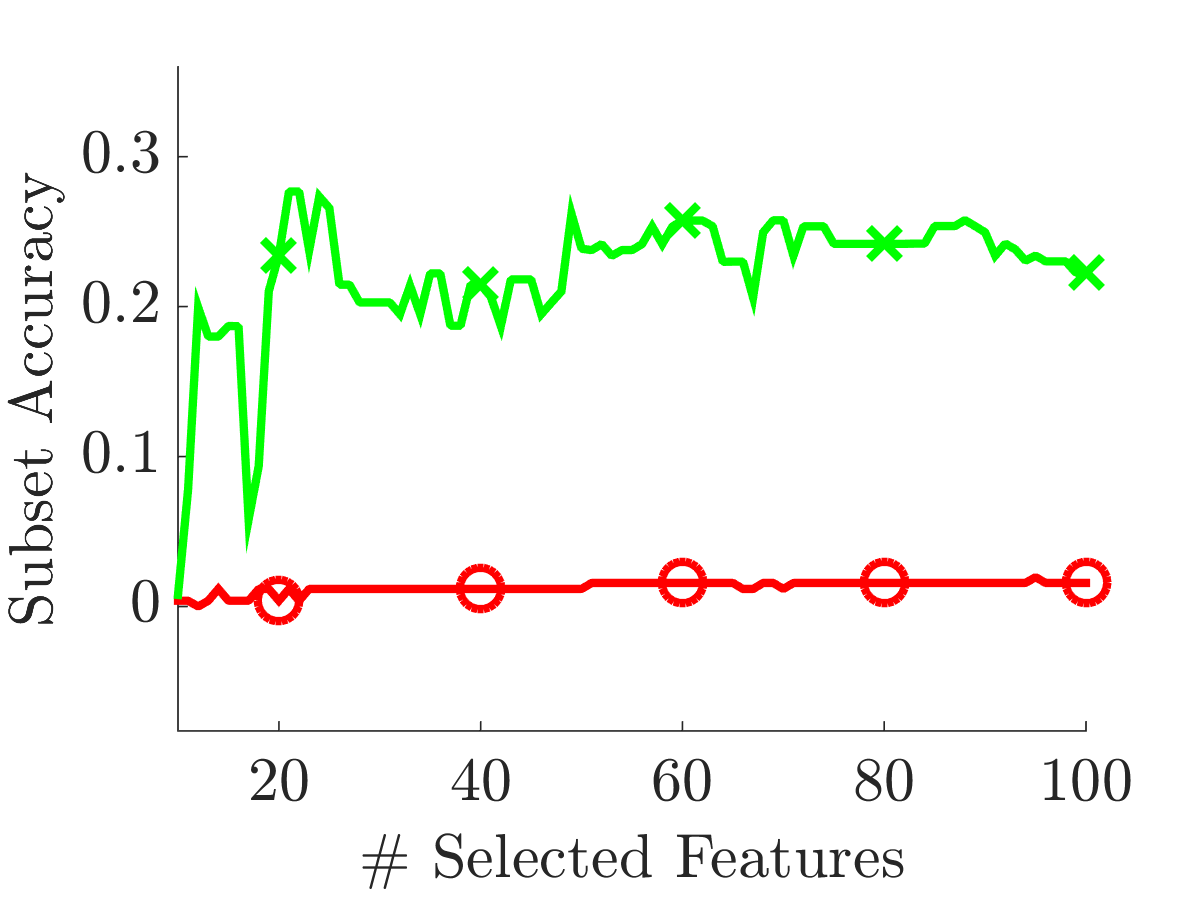}
\end{subfigure}
\begin{subfigure}[t]{0.28\textwidth}
\centering
\includegraphics[width=1.05\linewidth]{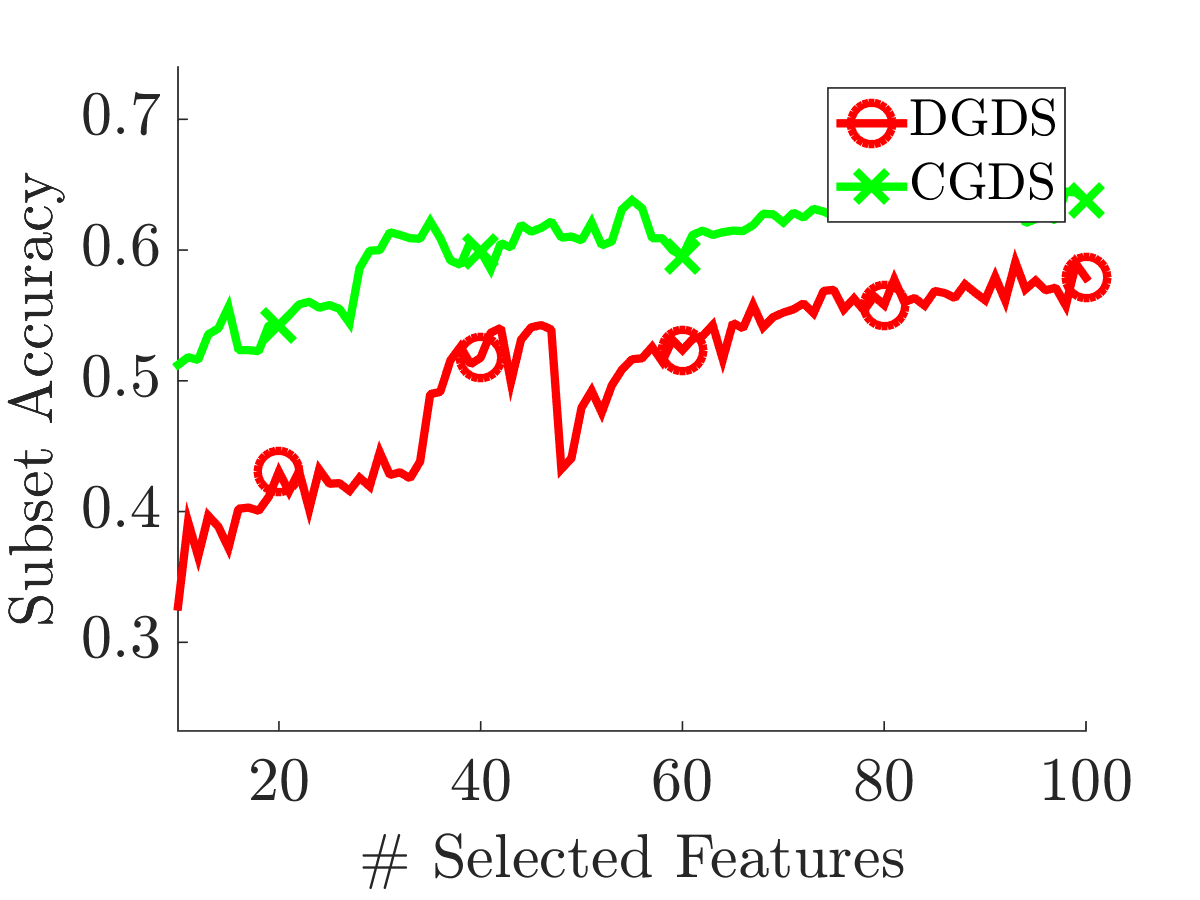}
\end{subfigure}
\vskip\baselineskip
\begin{subfigure}[t]{0.28\textwidth}
\centering
\includegraphics[width=1.05\linewidth]{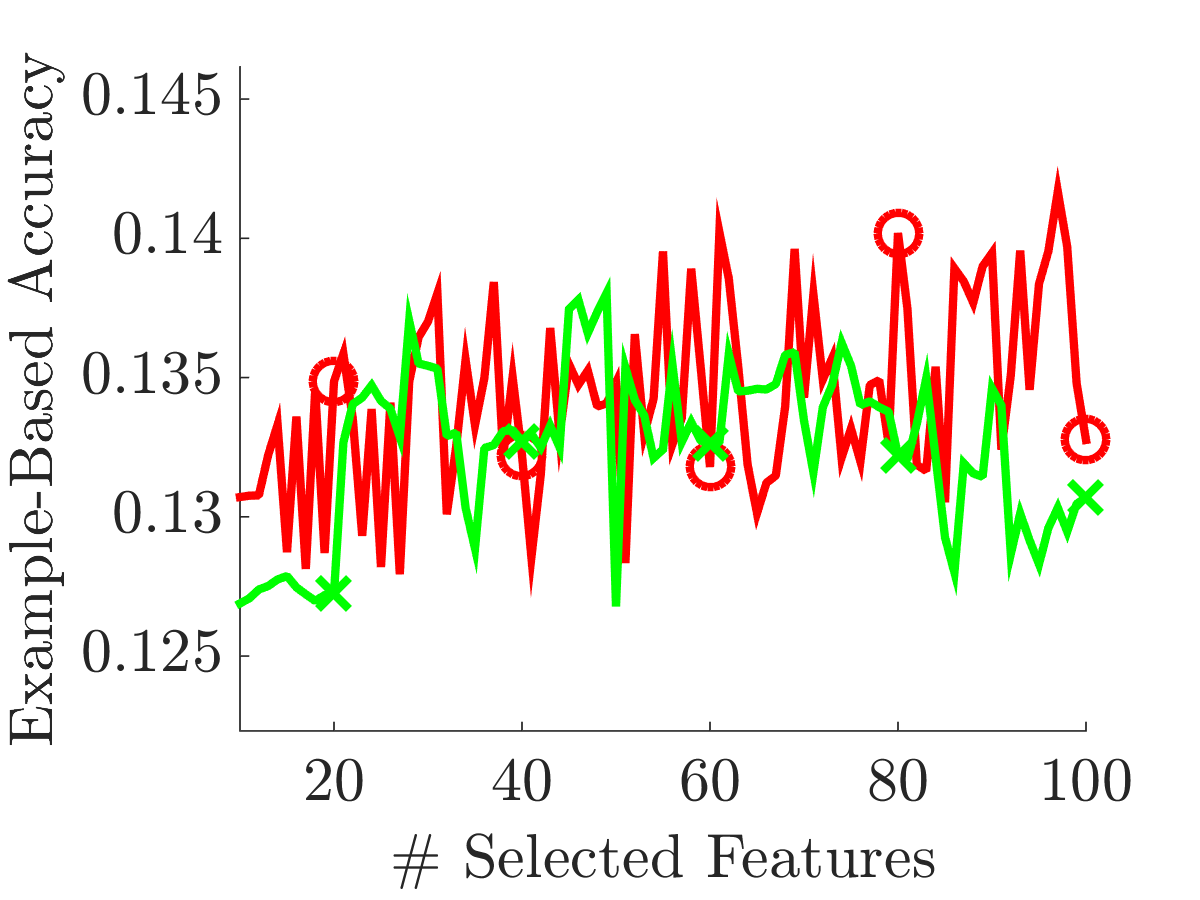}
\end{subfigure}
\begin{subfigure}[t]{0.28\textwidth}
\centering
\includegraphics[width=1.05\linewidth]{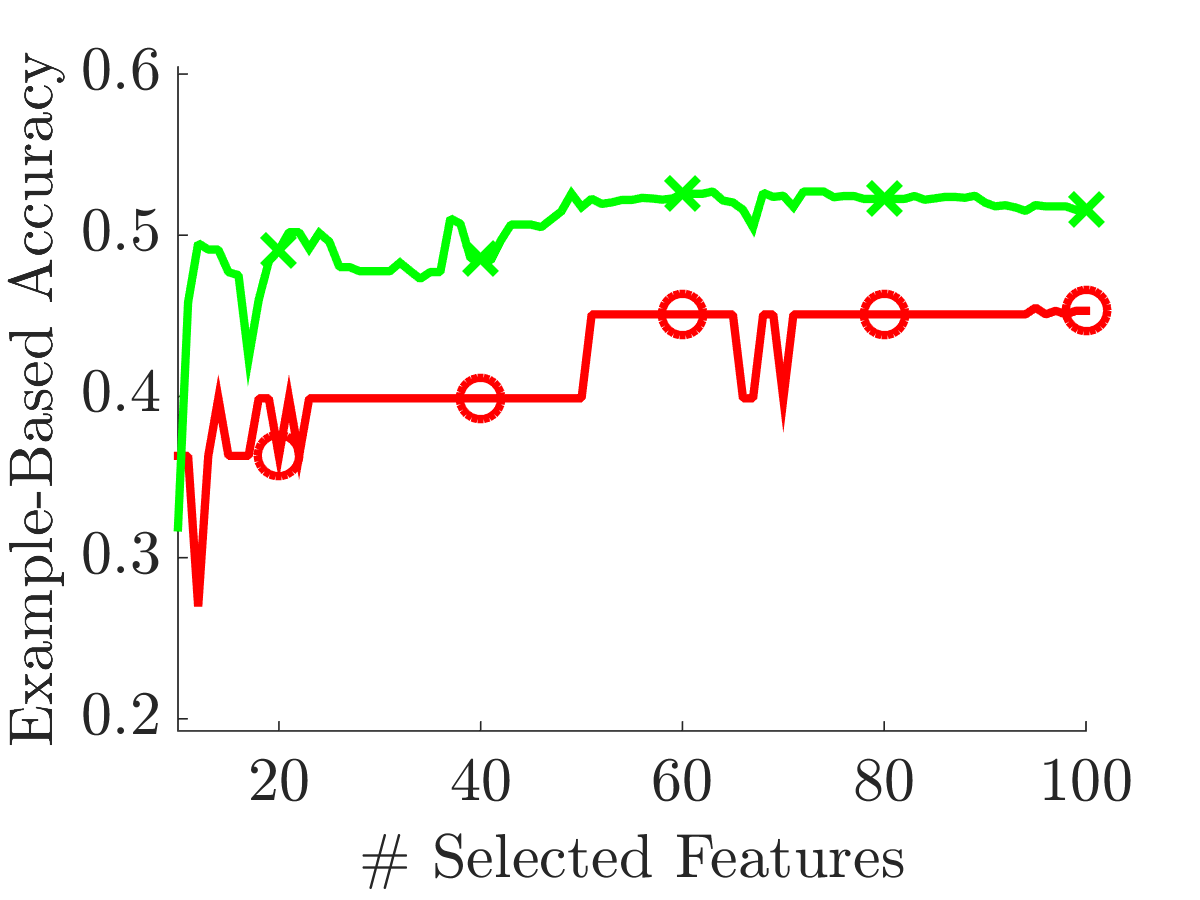}
\end{subfigure}
\begin{subfigure}[t]{0.28\textwidth}
\centering
\includegraphics[width=1.05\linewidth]{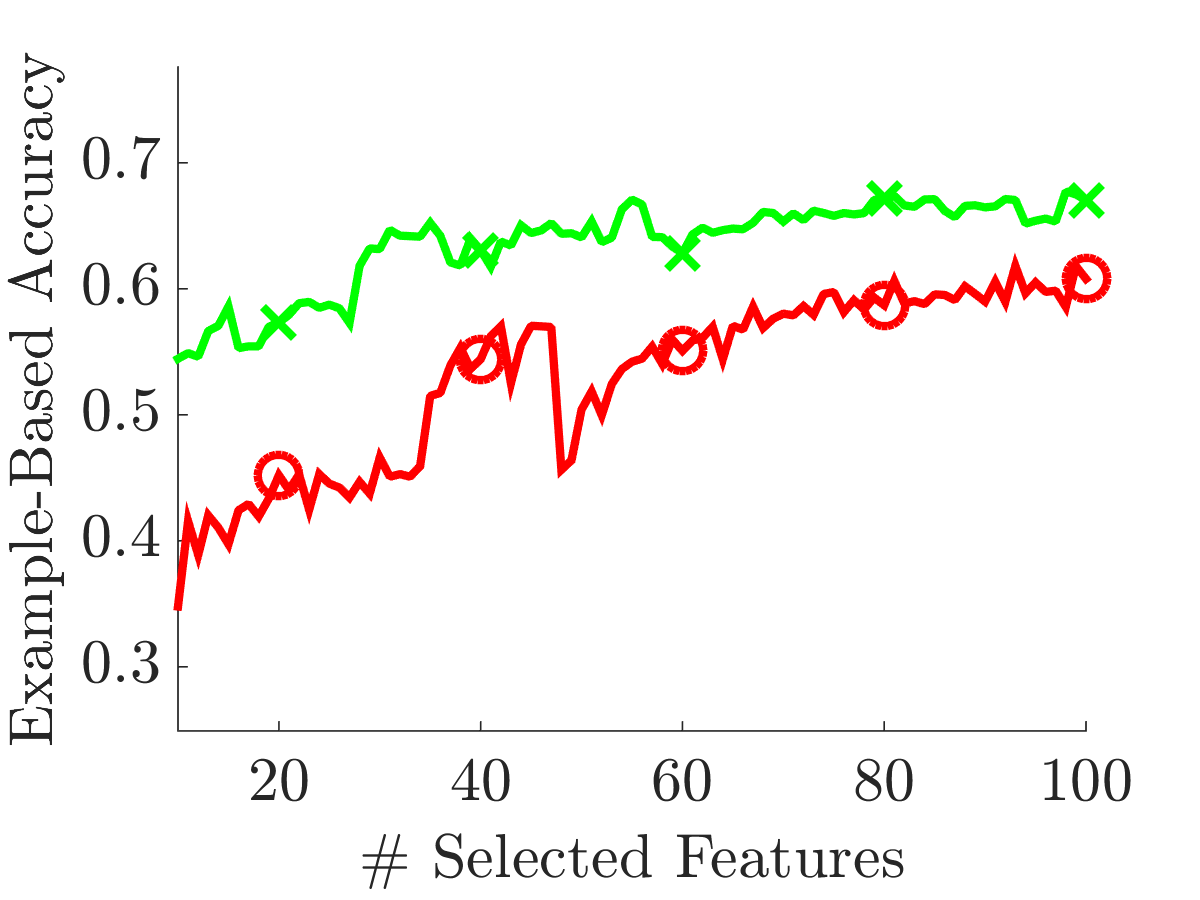}
\end{subfigure}
\vskip\baselineskip
\begin{subfigure}[t]{0.28\textwidth}
\centering
\includegraphics[width=1.05\linewidth]{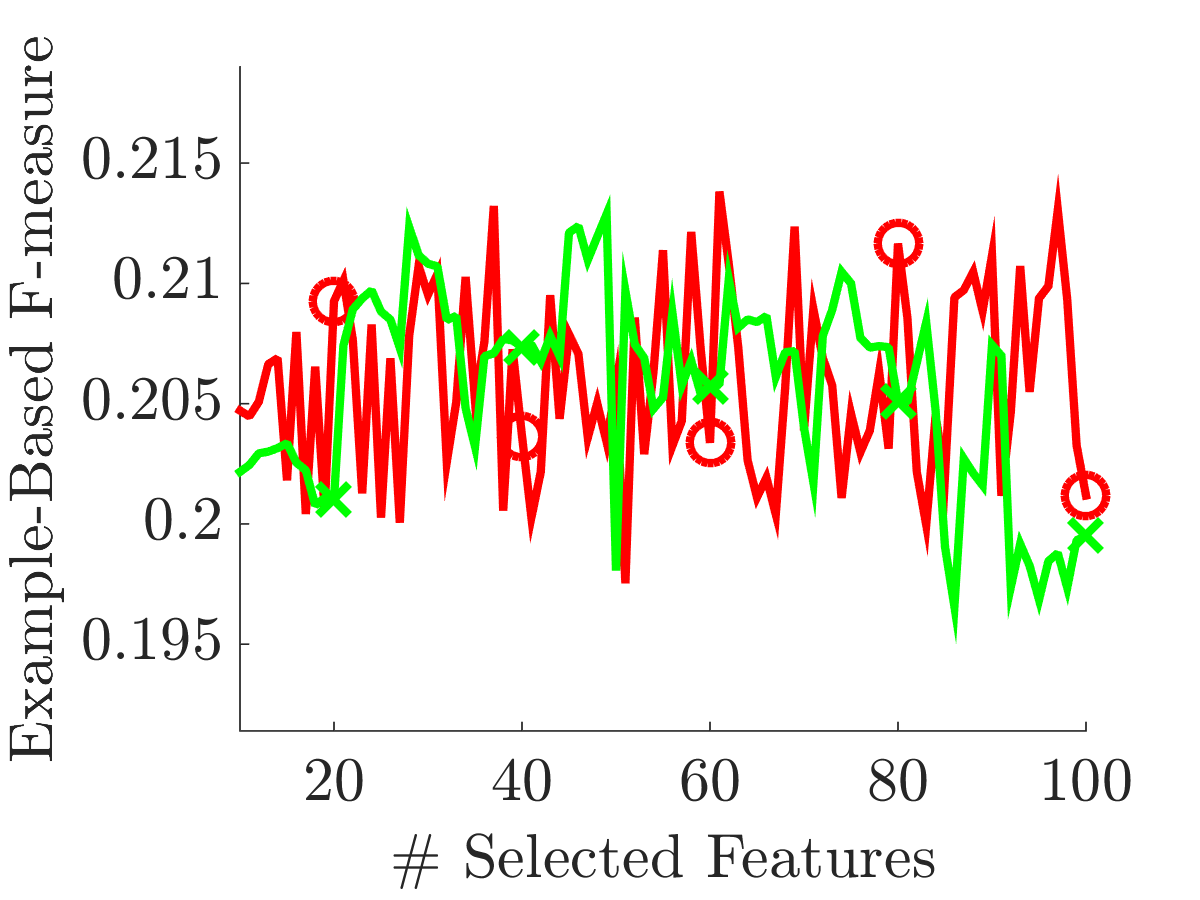}
\end{subfigure}
\begin{subfigure}[t]{0.28\textwidth}
\centering
\includegraphics[width=1.05\linewidth]{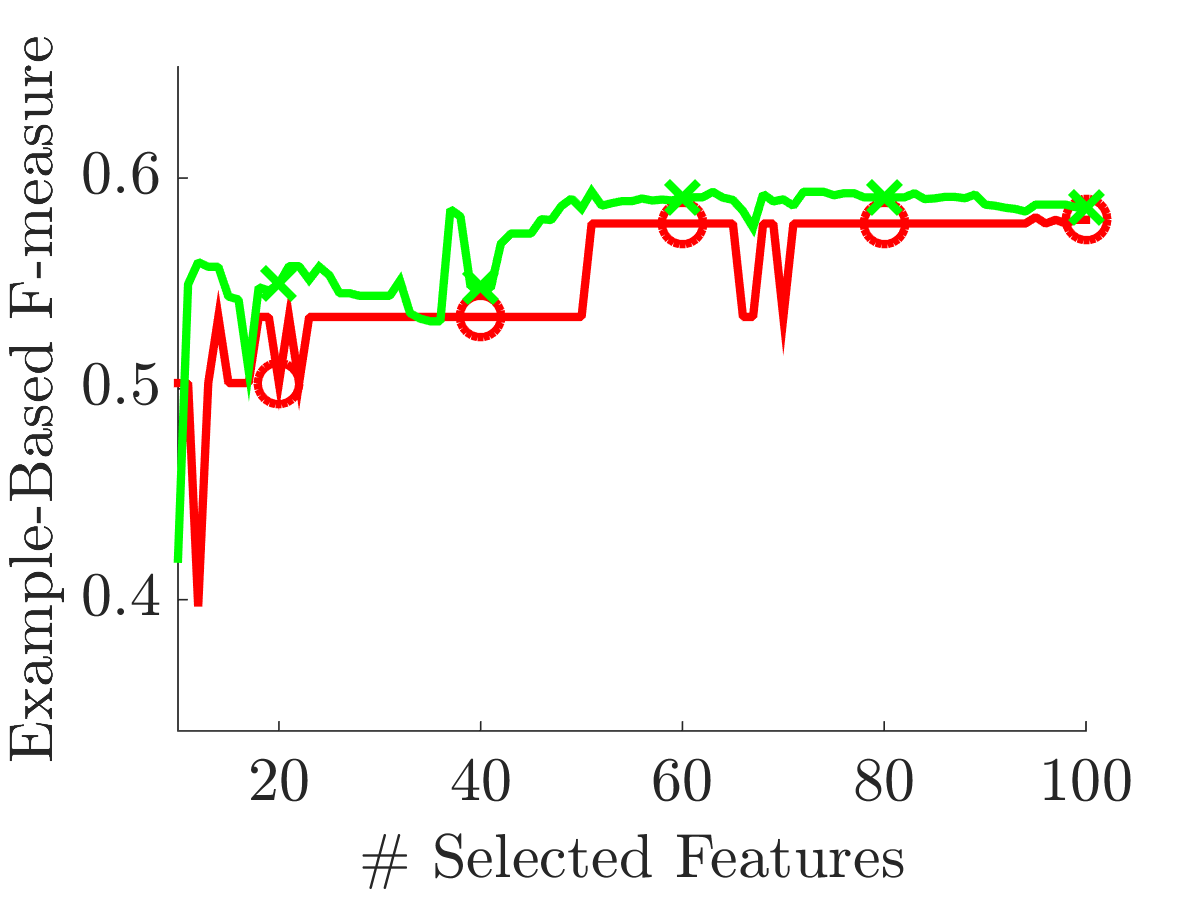}
\end{subfigure}
\begin{subfigure}[t]{0.28\textwidth}
\centering
\includegraphics[width=1.05\linewidth]{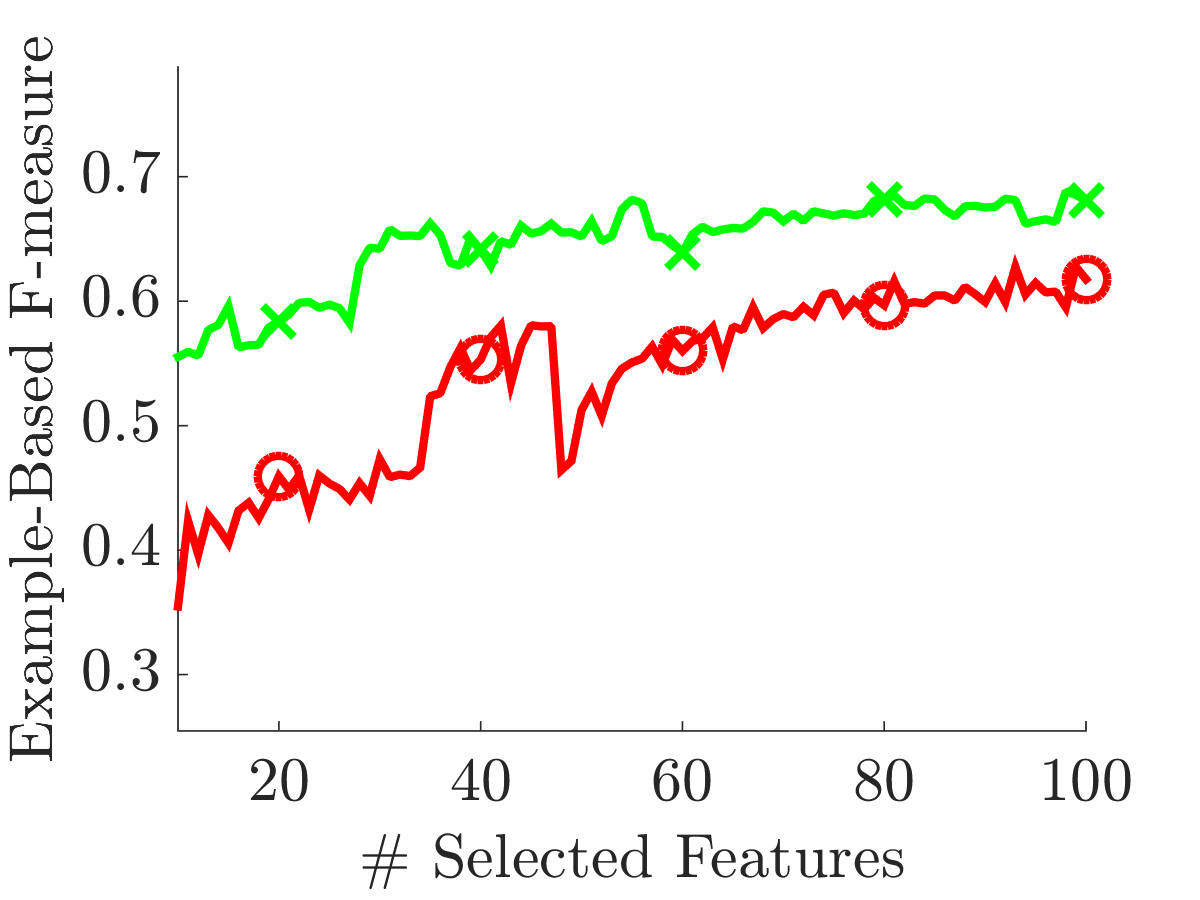}
\end{subfigure}
\vskip\baselineskip
\begin{subfigure}[t]{0.28\textwidth}
\centering
\includegraphics[width=1.05\linewidth]{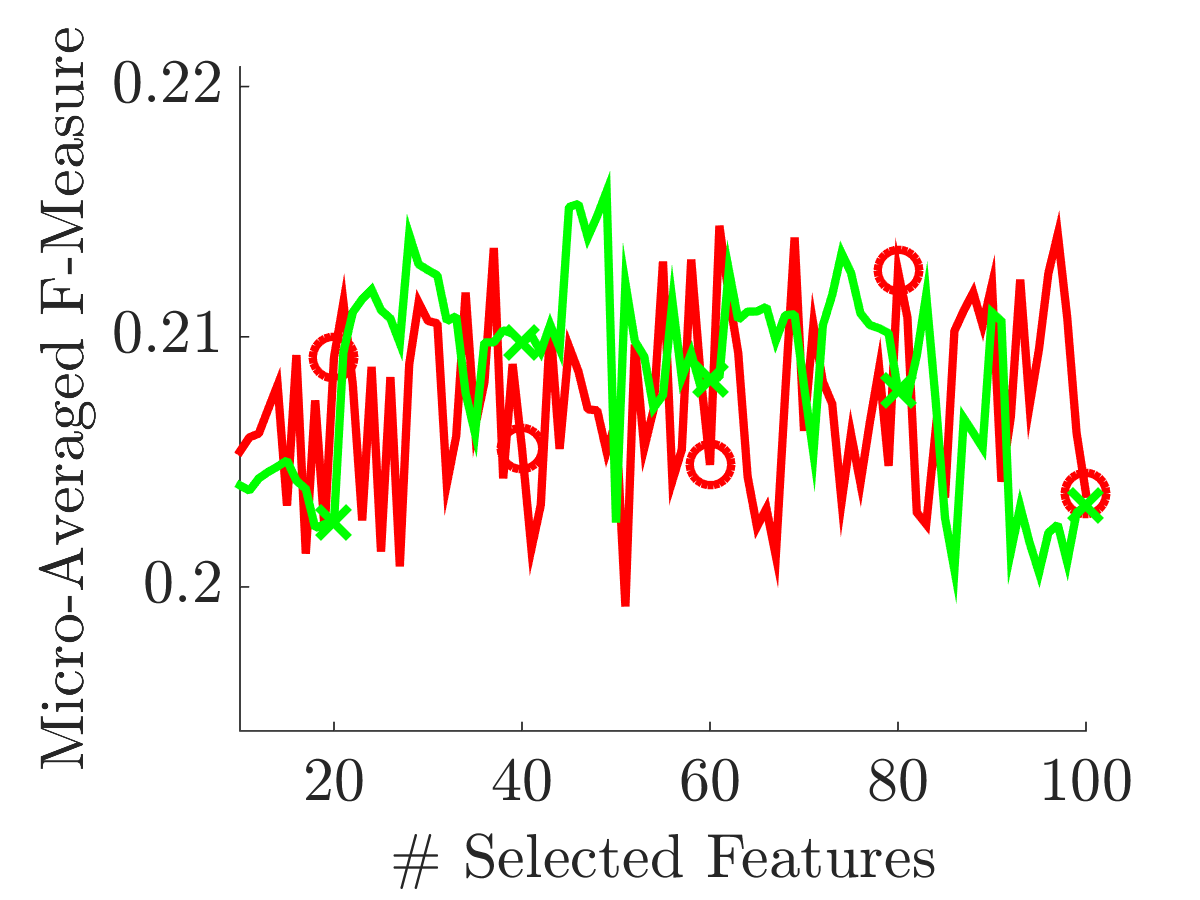}
\end{subfigure}
\begin{subfigure}[t]{0.28\textwidth}
\centering
\includegraphics[width=1.05\linewidth]{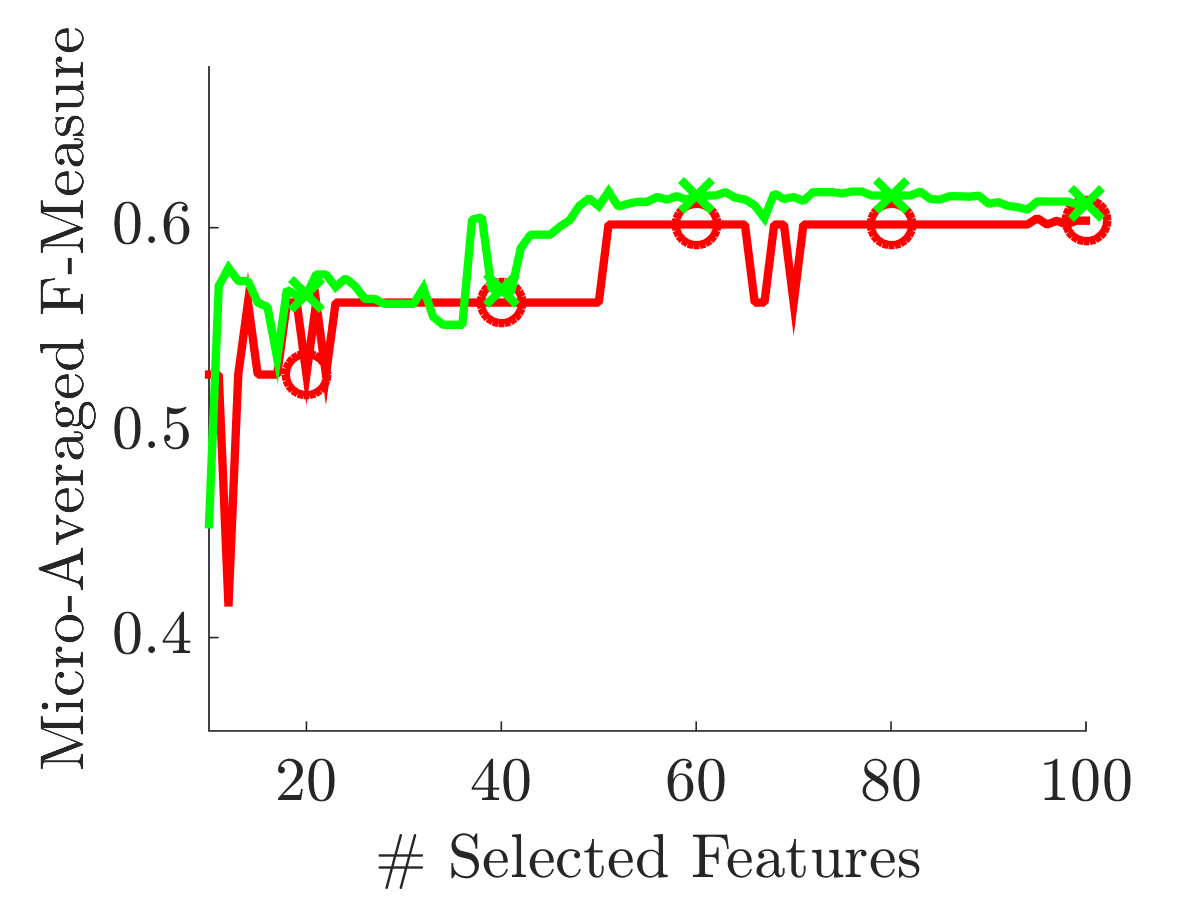}
\end{subfigure}
\begin{subfigure}[t]{0.28\textwidth}
\centering
\includegraphics[width=1.05\linewidth]{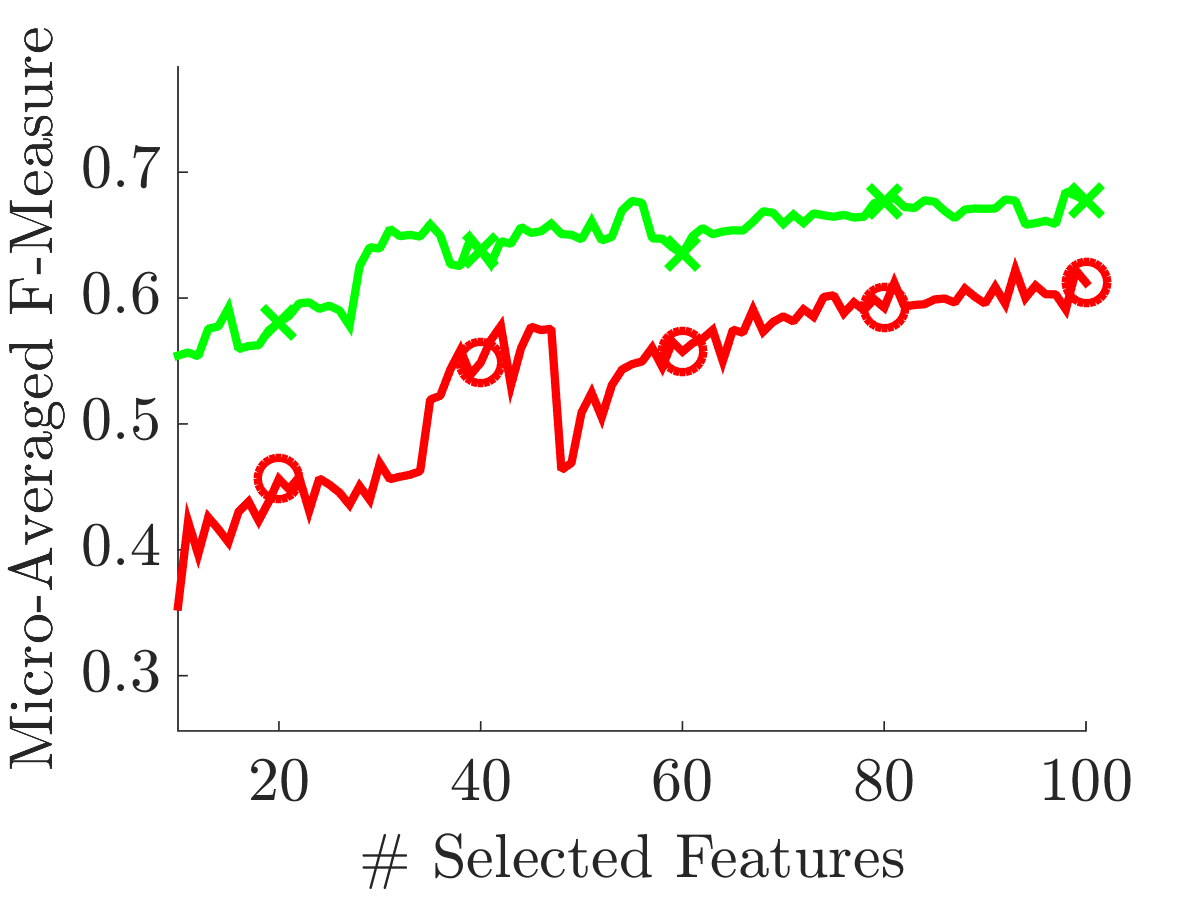}
\end{subfigure}
\vskip\baselineskip
\begin{subfigure}[t]{0.28\textwidth}
\centering
\includegraphics[width=1.05\linewidth]{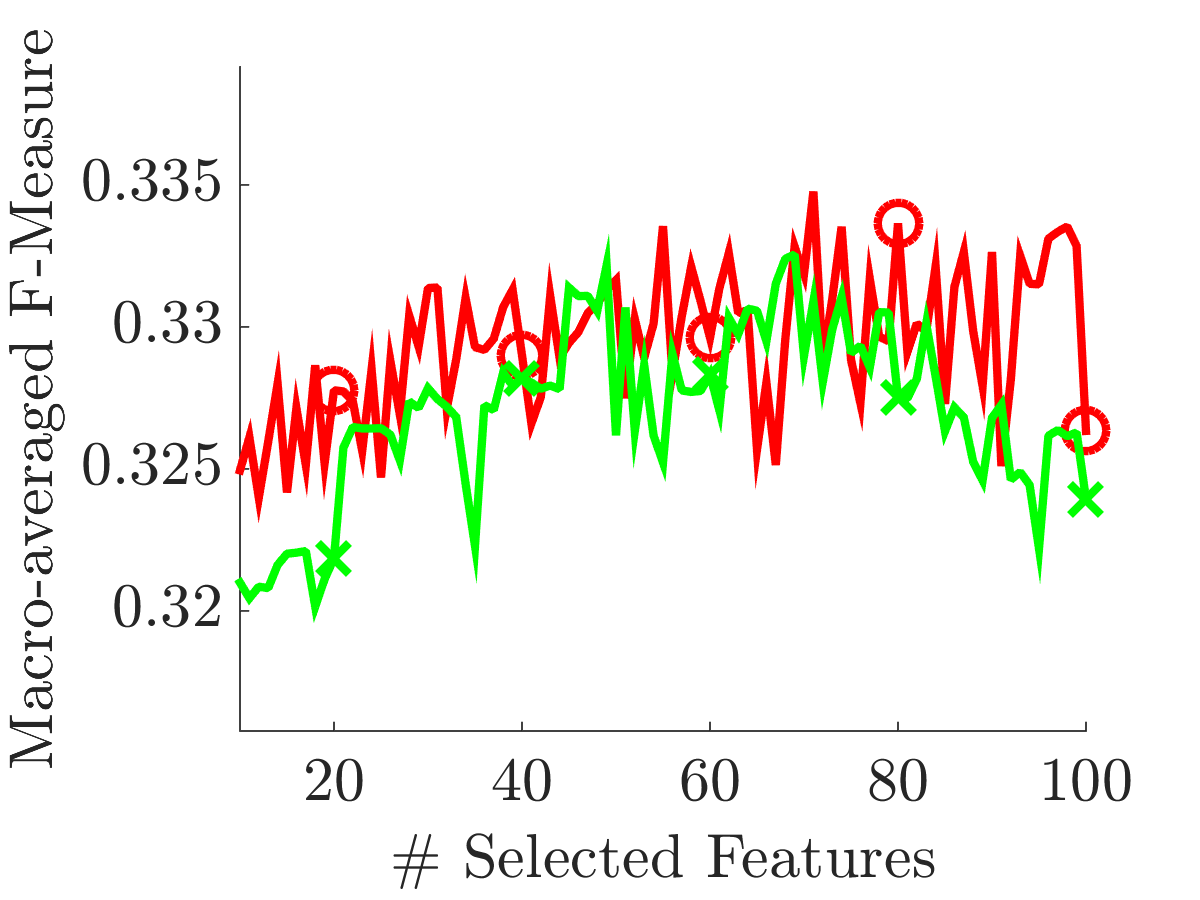}
\subcaption{Corel5k}
\end{subfigure}
\begin{subfigure}[t]{0.28\textwidth}
\centering
\includegraphics[width=1.05\linewidth]{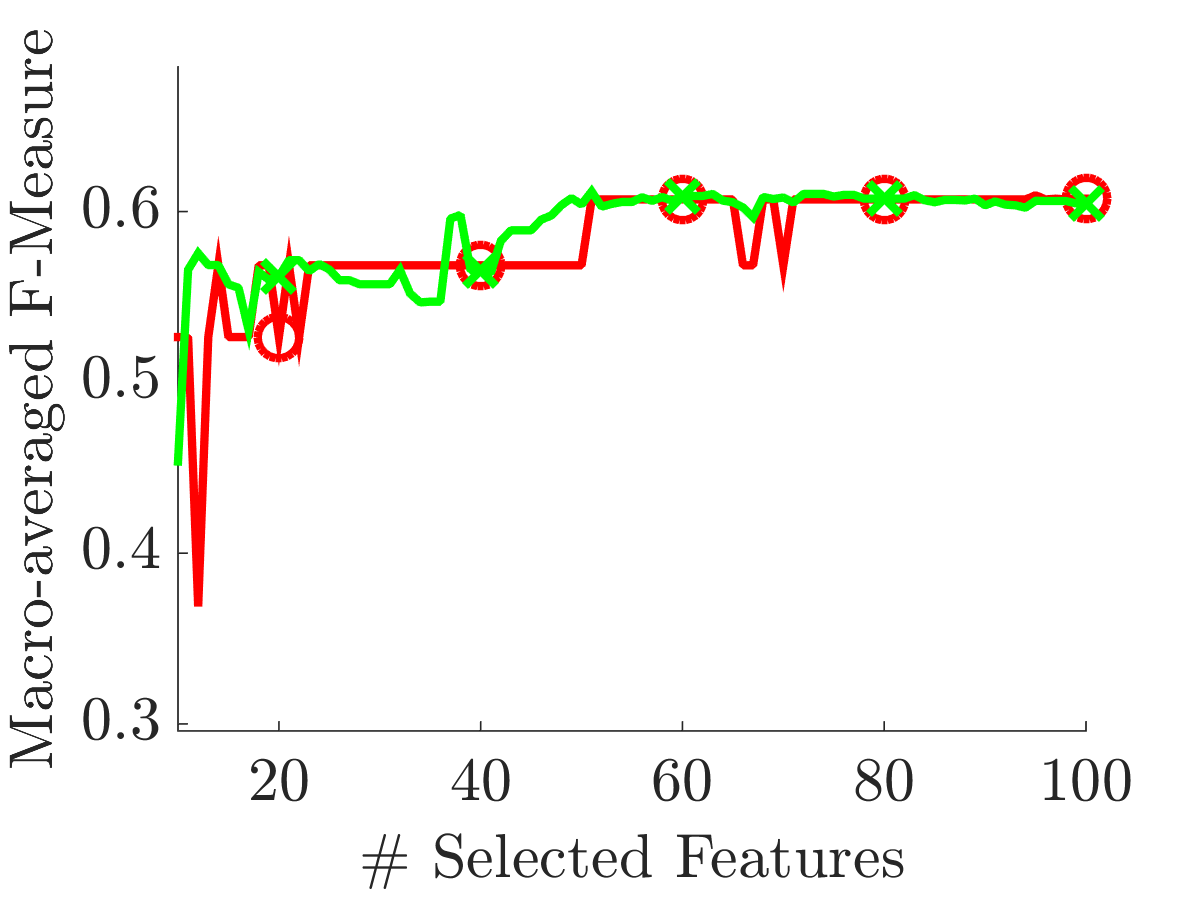}
\subcaption{Synthesized}
\end{subfigure}
\begin{subfigure}[t]{0.28\textwidth}
\centering
\includegraphics[width=1.05\linewidth]{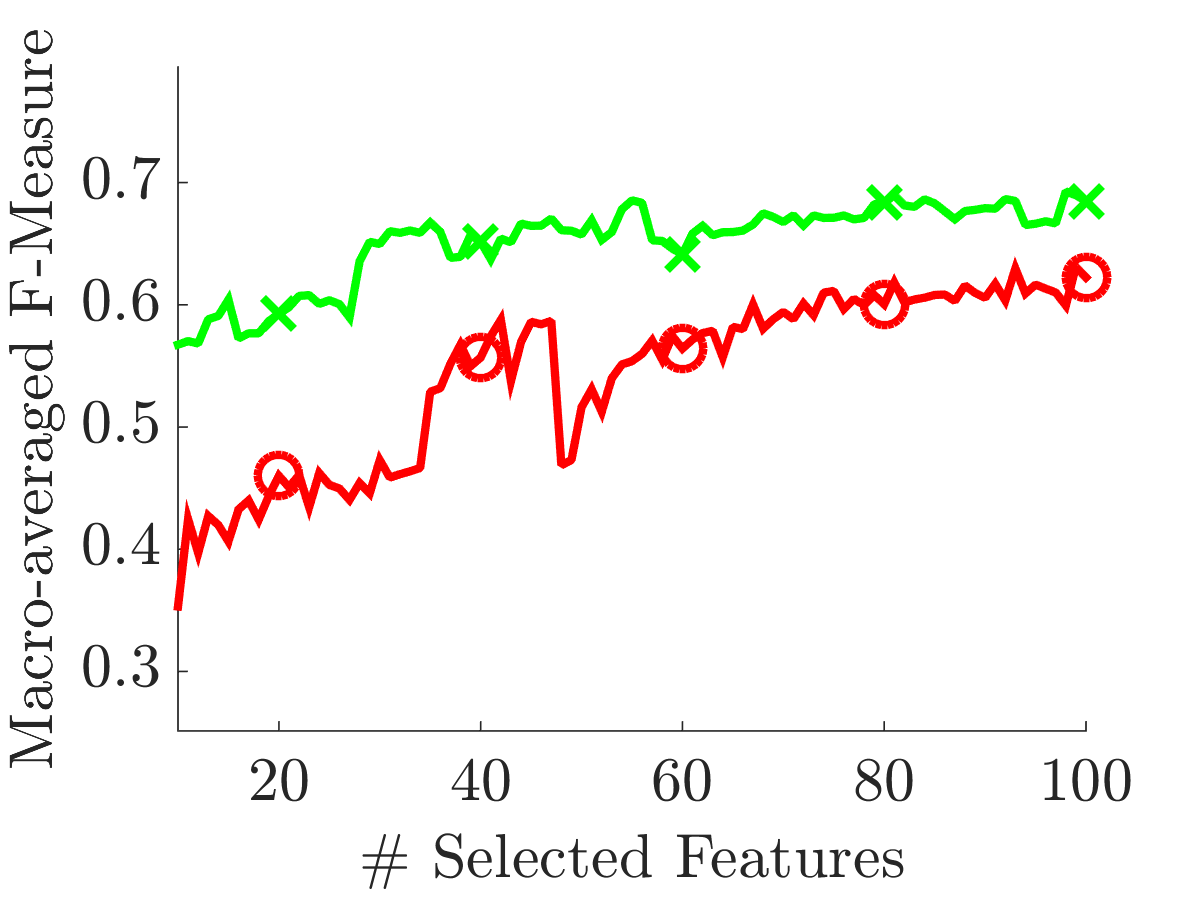}
\subcaption{Scene}
\end{subfigure}
\vskip\baselineskip
\caption{Comparison of proposed distributed method (DGDS) with proposed centralized method (CGDS) on the classification task.}
\label{Fig:comparison4}
\end{figure*}

\end{document}